\documentclass{article}

\usepackage[final]{neurips_2025}

\usepackage[utf8]{inputenc} 
\usepackage[T1]{fontenc}    
\usepackage{hyperref}       
\usepackage{url}            
\usepackage{booktabs}       
\usepackage{amsfonts}       
\usepackage{nicefrac}       
\usepackage{microtype}      
\usepackage[table, dvipsnames, svgnames]{xcolor}

\usepackage{amsmath}
\usepackage{hyperref}
\usepackage[capitalize,noabbrev]{cleveref}
\usepackage{wrapfig}
\usepackage{subcaption}
\usepackage{graphicx}
\usepackage{booktabs}
\usepackage{multirow}
\usepackage{makecell}
\usepackage{tikz}
\usetikzlibrary{calc}
\usepackage{algorithm}
\usepackage{algpseudocode}
\usepackage{svg}
\usepackage{bbm}
\usepackage{wrapfig}
\usepackage{amsthm}
\usepackage{thmtools}
\usepackage{subcaption}
\usepackage{arydshln}
\declaretheorem[name=Lemma]{lemma}
\declaretheorem[name=Lemma,numbered=no]{lemmarestate}

\usepackage{xspace}
\makeatletter
\DeclareRobustCommand\onedot{\futurelet\@let@token\@onedot}
\def\@onedot{\ifx\@let@token.\else.\null\fi\xspace}

\makeatother

\definecolor{Red}{rgb}{0.768, 0.054, 0.054}
\definecolor{Blue}{rgb}{0.43, 0.65, 0.86}
\definecolor{Green}{rgb}{0,0.4,0.7}
\definecolor{hotpink}{rgb}{1.0, 0.41, 0.71}
\definecolor{brown}{rgb}{0.59, 0.29, 0.0}
\definecolor{darkpastelgreen}{rgb}{0.01, 0.75, 0.24}
\definecolor{celestialblue}{rgb}{0.29, 0.59, 0.82}
\definecolor{ceruleanblue}{rgb}{0.16, 0.32, 0.75}
\definecolor{goldenrod}{rgb}{0.85, 0.65, 0.13}
\definecolor{navyblue}{rgb}{0.0, 0.0, 0.5}
\definecolor{coolgrey}{rgb}{0.55, 0.57, 0.67}
\definecolor{darkseagreen}{rgb}{0.56, 0.74, 0.56}
\definecolor{darkturquoise}{rgb}{0.0, 0.81, 0.82}
\definecolor{berryred}{rgb}{0.79, 0.25, 0.14}
\definecolor{teagreen}{rgb}{0.81,0.94,0.75}
\definecolor{lightgrey}{rgb}{0.5,0.5,0.5}
\definecolor{purple}{rgb}{0.35,0.25,0.55}
\hypersetup{
    colorlinks=true,
    citecolor=teal,
    linkcolor=Red,
    urlcolor=Green,
}

\usepackage{pifont}

\title{$\mathbf{\Delta}$ Attention: Fast and Accurate Sparse Attention Inference by Delta Correction}

\author{Jeffrey Willette$^{1}$,\quad Heejun Lee$^{1}$, Sung Ju Hwang$^{1,2}$\\
KAIST$^{1}$, DeepAuto.ai$^2$ \\
  \texttt{\{jwillette, ainl, sjhwang82\}@kaist.ac.kr}
}

\begin{document}

\maketitle

\begin{abstract}
The attention mechanism of a transformer has a quadratic complexity, leading to high inference costs and latency for long sequences. However, attention matrices are mostly sparse, which implies that many entries may be omitted from computation for efficient inference. Sparse attention inference methods aim to reduce this computational burden; however, they also come with a troublesome performance degradation. We discover that one reason for this degradation is that the sparse calculation induces a distributional shift in the attention outputs. The distributional shift causes decoding-time queries to fail to align well with the appropriate keys from the prefill stage, leading to a drop in performance. We propose a simple, novel, and effective procedure for correcting this distributional shift, bringing the distribution of sparse attention outputs closer to that of quadratic attention. Our method can be applied on top of any sparse attention method, and results in an average 36\%pt performance increase, recovering 88\% of quadratic attention accuracy on the 131K RULER benchmark when applied on top of sliding window attention with sink tokens while only adding a small overhead. Our method can maintain approximately 98.5\% sparsity over full quadratic attention, making our model 32 times faster than Flash Attention 2 when processing 1M token prefills.
\end{abstract}

\section{Introduction}
\label{sec:intro}

The main operation that powers modern transformers, self-attention~\citep{attention}, creates causal pairwise comparisons for every item in a sequence. While powerful and expressive, this operation comes with a quadratic complexity, leading to the need for large amounts of computation during inference on long sequences. This increases direct costs for hardware and electricity as well as negative externalities such as CO$_2$ emissions. Training-free sparse attention modifications aim to lower the quadratic complexity at inference time, but come with unwanted side effects such as accuracy degradation due to the sparsification of the attention matrix.

\begin{figure}[H]
    \centering
    \includegraphics[width=0.9\linewidth]{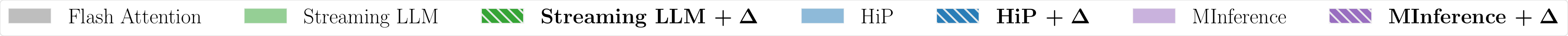}
    \includegraphics[width=\linewidth]{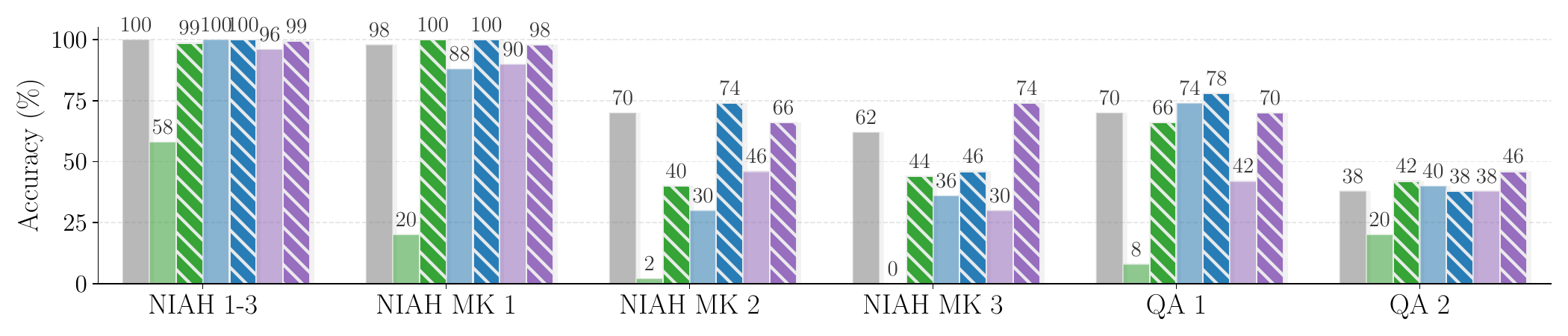}
    \vspace{-1.5em}
    \caption{
    \textbf{RULER 131K Subsets.} At long context lengths, sparse attention can degrade performance by a large margin. Our simple $\mathbf{\Delta}$ correction improves performance and only requires an additional 1.5\% of the full quadratic attention computation.}
    \label{fig:ruler-subset-motivation}
    \vspace{-0.5em}
\end{figure} 

Recent works on sparse attention have found that a sparse sliding window can be added at inference time without a total loss of model stability. This is accomplished by saving a small number of initial tokens, and applying a sliding window on all subsequent tokens (Streaming LLM~\citep{sink}). Subsequent works such as Star Attention~\citep{star-attn} have proposed a similar sparse prefill strategy with a fully dense decoding procedure to generate new tokens. This strategy has the positive attribute of a sparse prefill while still performing attention with all tokens during generation. This should allow the model to accurately recall context buried deep within the prompt. However, we find that this is not the case in practice. For example, there is a challenging subset of the RULER~\citep{ruler} benchmark titled MultiKey-3, which consists entirely of unique UUID keys and values, and the large language model (LLM) must be able to recall the proper value for a particular key in order to get a correct answer. In this setting, a sliding window of 2048 tokens provides more than adequate room for encoding individual key and value pairs together within the window. One would then expect that a dense decode procedure would be able to retrieve the proper UUID given a user query. However, we find that this is not the case and the dense decode achieves a surprisingly low accuracy of 0\% as opposed to 62\% when using quadratic attention.  

\begin{wrapfigure}[13]{r}{0.35\textwidth}
  \vspace{-1.8em}
  \begin{center}
    \includegraphics[width=\linewidth]{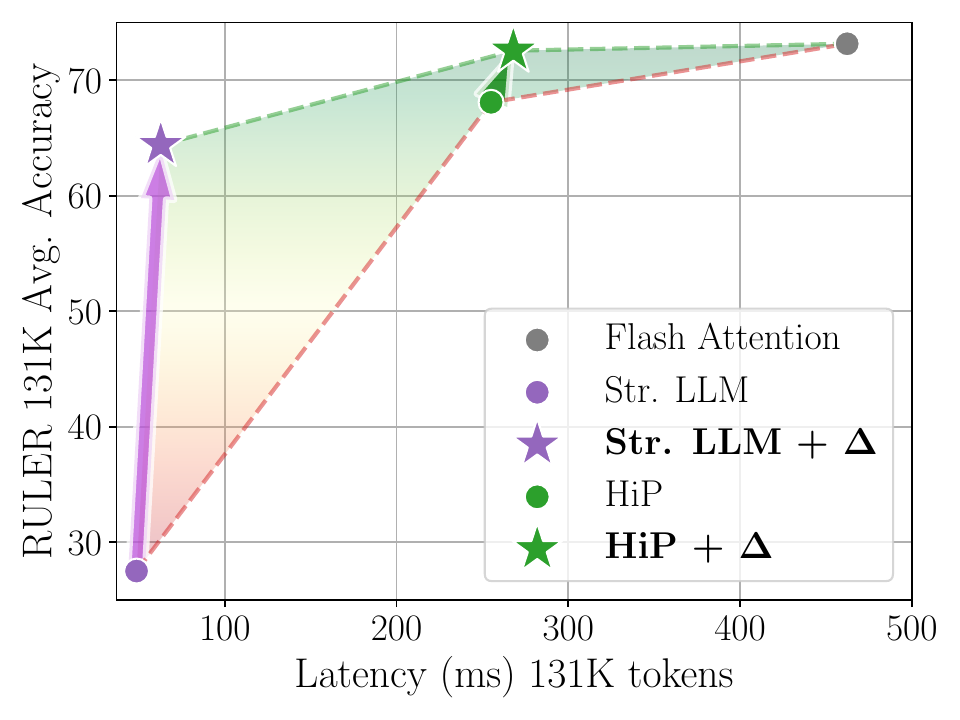}
  \end{center}
  \vspace{-1.1em}
  \caption{
  Comparing RULER 131K prefill attention latency and accuracy for sparse attention methods.
  }
  \label{fig:ruler-latency-acc-motivation}
\end{wrapfigure}

We find this drop in accuracy arises from a distributional shift in the output tokens of each layer due to the sparse prefill. This distributional shift causes problems with the query-key dot products in long contexts and therefore results in an extreme drop in performance as the queries no longer align with the expected keys. We study this problem and found a surprisingly simple fix which we dub $\mathbf{\Delta}$ Attention that improves the accuracy of sliding window attention from \textbf{0\% to 44\%~(\cref{fig:ruler-subset-motivation}, NIAH MK3) on this challenging subset} while maintaining more than 11-fold speedup over plain Flash Attention 2~\citep{flashattention2} for processing 131K context lengths~(\cref{fig:ruler-latency-acc-motivation}). Through evaluations on perplexity, natural language understanding, and synthetic tasks, we demonstrate that our method consistently results in better performance while maintaining the low latency of the sparse prefill. 

Our contributions are as follows:

\begin{itemize}
    \item We identify a distributional shift in tokens when applying an inference-time sparse attention method to pretrained transformers, which interferes with query-key alignment on long contexts and leads to a drop in performance.
    \item We introduce Delta ($\mathbf{\Delta}$) Attention, a sparse post-processing correction that realigns sparse outputs with full quadratic attention.
    \item Our method adds negligible latency overhead compared to plain sparse attention, while drastically increasing performance over purely sparse methods.
    \item Our method is designed to work in the attention output space, so it can be seamlessly integrated with existing sparse attention kernels and inference pipelines without major modification.   
\end{itemize}

\section{Background \& Related Work}
\label{sec:background}

The self attention mechanism of a transformer takes an input sequence $\mathbf{X} \in \mathbb{R}^{N \times d}$ of individual tokens $\mathbf{x}_i \in \mathbb{R}^{d}$ for $i \in \{1..N\}$. After applying linear projections $\mathbf{W_Q}, \mathbf{W_K}, \mathbf{W_V} \in \mathbb{R}^{d \times d}$ to the input $\mathbf{X}$ to achieve the respective $\mathbf{Q}, \mathbf{K}, \mathbf{V}$ matrices, positional encodings such as~\citep{rope} are applied to $\mathbf{Q}$ and $\mathbf{K}$. With $\sigma$ representing the softmax operation over the last dimension, the self-attention operation for an arbitrary layer in a transformer is the following,

\vspace{-1em}
\begin{equation}
    \label{eq:attention}
     \mathbf{AV} = \sigma\left(\frac{\mathbf{Q}\mathbf{K}^\top}{\sqrt{d}}\right)\mathbf{V} = \sigma\left(\frac{\mathbf{XW_Q}(\mathbf{XW_K})^\top}{\sqrt{d}}\right)\mathbf{XW_V}
\end{equation}
\vspace{-0.75em}

We omit the output projections, attention heads, and post-attention multilayer perceptrons (MLPs). For a deeper discussion of these topics in transformers, please see~\citep{attention}. The most expensive operation in~\cref{eq:attention} that arises from the multiplication inside $\sigma()$ which results in the implicit construction of an attention matrix $\mathbf{A} \in \mathbb{R}^{N \times N}$ which is computationally expensive for large $N$. Due to the causality condition of language, a token $\mathbf{x}_i$ may only influence another token $\mathbf{x}_j$ where the index $i \le j$. In practice, this means that only the lower triangle of $\mathbf{A}$ is computed.  

After traversing through the layers of the network, the next token in the sequence $\mathbf{x}_{N+1}$ is generated (predicted) and added to the input sequence to generate the next token and so on until the sequence terminates. In this generation phase, each iteration may use the previously computed tokens, which are stored within a cache at each layer, so that we may avoid re-calculating the entire attention matrix in~\cref{eq:attention}. With a union operator $\cup$ which concatenates matrices by adding new rows, and considering that $\mathbf{K, V}$ contain tokens with indices $\{1..N\}$, and the newly generated token has index $i = N + 1$, the generative process for the next token proceeds through the attention layers as,

\vspace{-1.25em}
\begin{equation}
    \label{eq:generation}
    (\mathbf{a}^\top \mathbf{v})_i = \sigma\left(\frac{\mathbf{q}_i^\top\left[\mathbf{K} \cup \mathbf{k}_i^\top\right]^\top}{\sqrt{d}}\right)\left(\mathbf{V} \cup \mathbf{v}_i^\top \right)
\end{equation} 
\vspace{-1.em}

Sparse attention prefill methods aim to reduce the quadratic computation in~\cref{eq:attention} by computing a subset of entries within $\mathbf{A}$, forming a sparse matrix $\mathbf{A^*}$ where the number of computed entries $\sum_{i,j} \mathbbm{1}\{\mathbf{A^*}_{i,j} > 0\} \ll \frac{N^2}{2}$ with minimal information loss. However, in practice, large portions of the attention matrix are ignored, which may cause unintended differences in the output tokens and lead to unexpected behavior of future query-key dot products, which could degrade performance on downstream tasks. Previous works have studied in-context learning (ICL) processes such as induction heads~\citep{induction-head}, which are responsible for copying relevant content from earlier tokens into later tokens in the sequence~\citep{stacked-induction}. Induction heads are known to be more prevalent in the lower layers of the network~\citep{induction-fv}, which implies that a distributional mismatch between queries and keys at the lower layers of the network will inhibit ICL processes. Additionally, \cite{retrieval-head} showed that these induction or retrieval heads are universal for all transformer model types and further highlighted that interfering with these special attention heads causes a catastrophic drop in performance on downstream tasks during inference. 

Recent works on sparse attention, such as Streaming LLM~\citep{sink}, have shown that a pretrained quadratic transformers can be modified on-the-fly at test time into a stable sparse attention model by utilizing sink tokens and sliding windows. This has inspired a multitude of recent works that utilize this knowledge for inference time adaptations that selectively prune the less important `middle' tokens from the KV-cache during inference. Two approaches, H2O~\citep{h2o} and SnapKV~\citep{snapkv} accomplish this by looking at historical attention scores to decide which tokens to prune. However, these works still leave the quadratic prompt in place, which requires a computation overhead of $\mathcal{O}(n^2)$. 

\begin{wrapfigure}[20]{r}{0.4\textwidth}
\vspace{-1.5em}
\begin{center}
\includegraphics[width=0.8\linewidth]{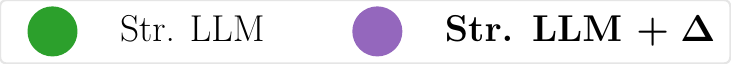}
\includegraphics[width=\linewidth]{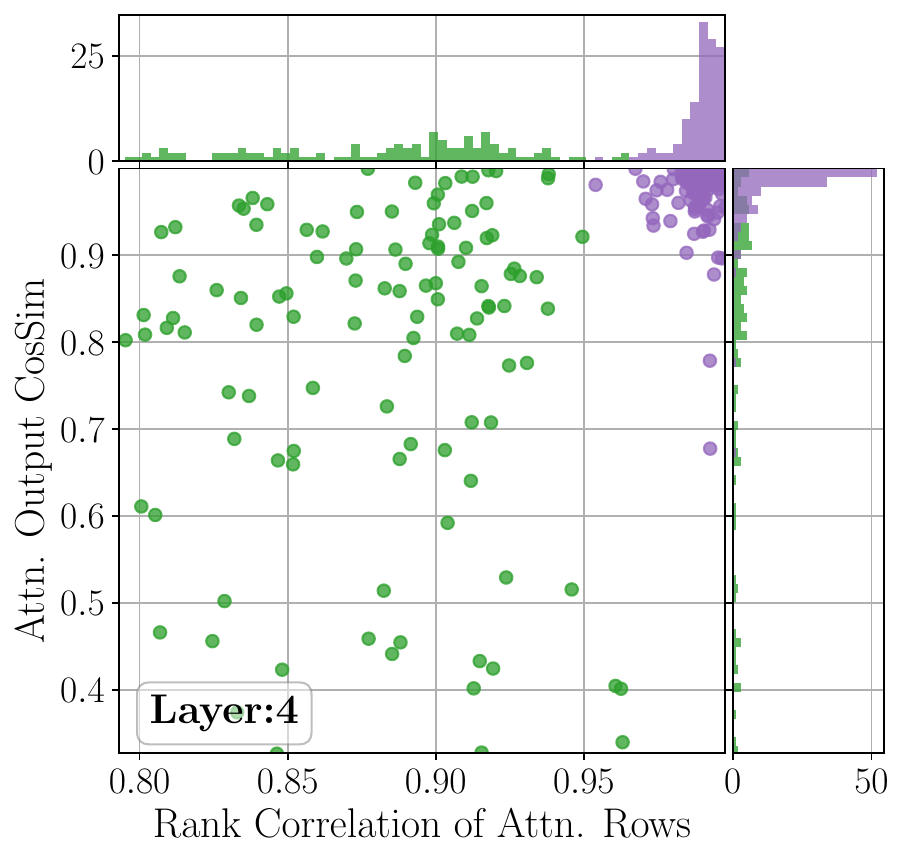}
\end{center}
\vspace{-0.5em}
\caption{
Comparing sparse attention methods to quadratic attention. Our $\mathbf{\Delta}$ correction results in outputs that are more similar to quadratic attention.}
\label{fig:motivation}
\end{wrapfigure}

Other recent works have therefore made efforts to lower the complexity of the prompt as well. Big Bird~\citep{bigbird} studies the effect of randomly choosing keys for every new query in the attention matrix. However, random key selection has been shown to underperform a more targeted selection of keys in HiP Attention~\citep{hip, infhip}, which applies a tree-based pruning mechanism that masks out less important blocks of keys in order to sparsify the computation of the attention matrix. MInference~\citep{minference} studies reliably recurring patterns in the attention matrix of specific attention heads, and builds a set of sparse kernels which apply sparse attention following these patterns. Star Attention~\citep{star-attn} uses a sparse strategy akin to that of Streaming LLM with a sliding window, initial tokens, and a fully dense decode procedure which evaluates the dot product between every past key for new queries during the decoding phase. As we show in our experiments, this scheme does not work for all tasks unless the sliding window represents a large percentage of the total context length (see~\cref{tab:ruler}).

To illustrate how our findings integrate with these prior works, we provide an example in~\cref{fig:motivation}. In this experiment, we use quadratic attention and Streaming LLM to prefill a 131K length input from the RULER benchmark. We then compute the cosine similarity $\text{cos}([\mathbf{A^*V}]_i, [\mathbf{AV}]_i)$ of the sparse and quadratic outputs, and also construct the last part of the full attention matrix using the last 128 queries in order to compare the rank correlation coefficient $\rho(\mathbf{A^*}_i, \mathbf{A}_i)$ in the final rows of the attention matrix. If the sparse attention method does not cause a distributional shift, then the attention outputs should have a high cosine similarity to quadratic attention, and sorting the rows of the attention matrix should lead to the same sort order, which implies that the relative importance (ranking) between queries and keys has been maintained. As seen in~\cref{fig:motivation}, in both dimensions, the sparse attention of Streaming LLM causes a drift in the distribution of tokens, which causes the degradation in task performance seen in \cref{fig:ruler-subset-motivation}. However, we find we can correct this distributional shift with the addition of a $\mathbf{\Delta}$ term which we will describe in the following section. 

\vspace{-0.75em}
\section{Method}
\label{sec:method}
\vspace{-0.5em}

\begin{figure}[t]
    \vspace{-0.25in}
    \centering
    \includegraphics[width=\linewidth,trim={0 0 30mm 20mm},clip]{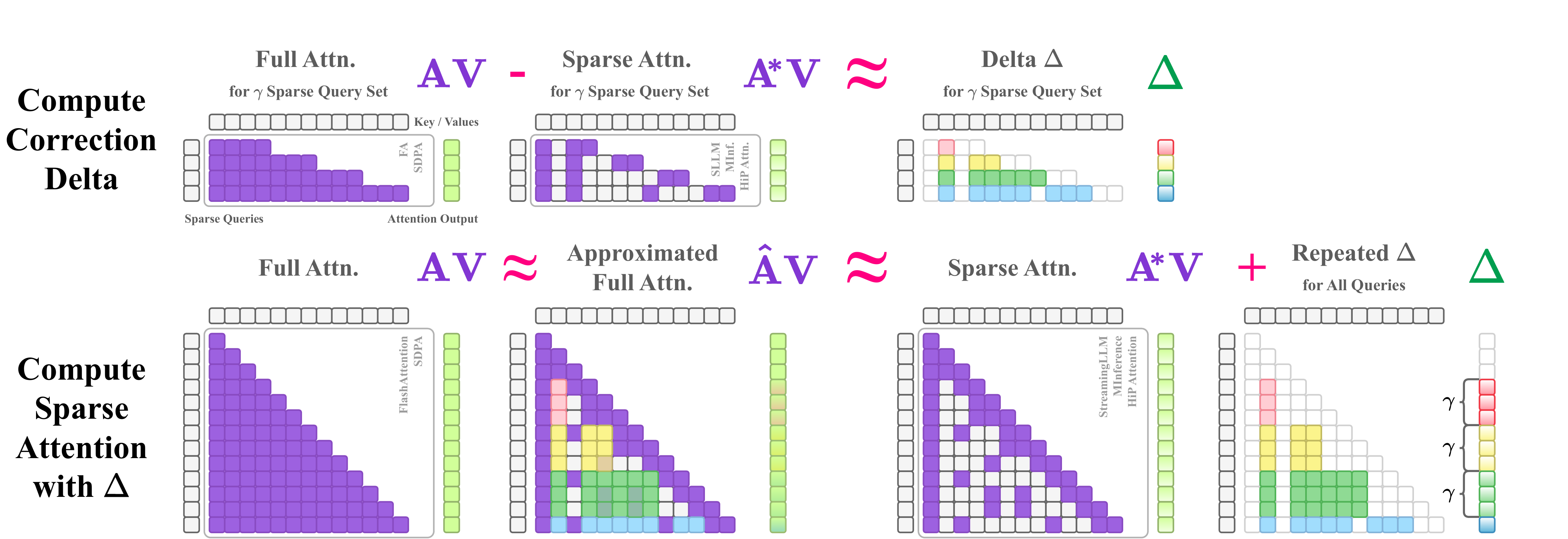}
    \caption{
    \textbf{Overview of $\mathbf{\Delta}$ Attention.} \textbf{(Top)} Given an arbitrary sparse attention method we calculate the difference between the sparse attention and full attention for a small subset of queries. The subset size is controlled by a hyperparameter $\gamma$. \textbf{(Bottom)} We then repeat the calculated difference for all output tokens and add the result to the full sparse attention output. The result is an approximation to the original quadratic attention.}
    \vspace{-1.0em}
    \label{fig:concept}
\end{figure}

\begin{wrapfigure}[8]{r}{0.45\textwidth}
  \vspace{-2.3em}
  \begin{center}
    \includegraphics[width=\linewidth,trim={0 0 170mm 0},clip]{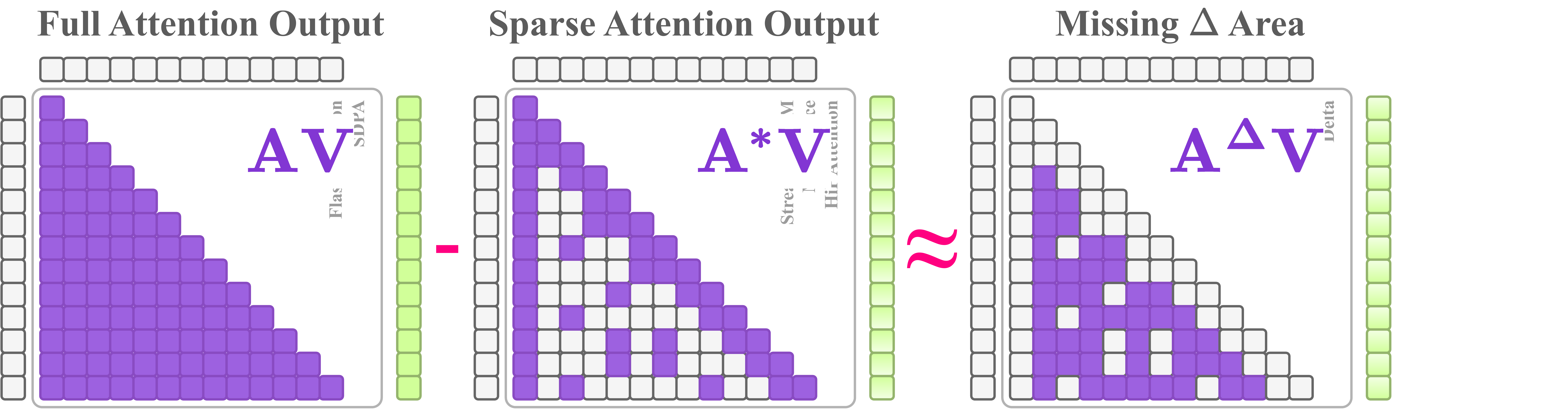}
  \end{center}
  \vspace{-0.8em}
  \caption{
  \textbf{Intuition for $\mathbf{\Delta}$ Attention.} The difference of attention outputs approximates the missing attention contribution.}
  \label{fig:intuition}
\end{wrapfigure}

Given the distributional shift shown in~\cref{fig:motivation}, our method answers the following question: \textit{How may we shift the distribution of attention outputs such that they are closer to the representation which is expected during quadratic attention?} Specifically, we wish to add a term to the sparse attention output $\mathbf{A}^*\mathbf{V}$ such that we recover the attention contribution $\mathbf{A^\Delta V}$ from the places where sparse attention has given zero weight. This region is usually located somewhere inside the lower triangle of the attention matrix and resembles a delta shape. We propose to approximate this $\mathbf{\Delta}$ region by a simple difference of attention outputs, as geometrically depicted in~\cref{fig:intuition}. Specifically, 

\vspace{-1.25em}
\begin{equation}
 \mathbf{A^\Delta V} \approx \mathbf{AV} - \mathbf{A}^*\mathbf{V}
\end{equation}
\vspace{-1.75em}

Note that the softmax normalization of sparse attention methods generally only computes the normalization constant over the nonzero values. Thus, $\mathbf{A}$ and $\mathbf{A}^*$ have different normalization constants, which makes the relation an approximation. We consider $\mathbf{A}$ and $\mathbf{A^\Delta}$ to share the same softmax normalization constant. Let the full attention softmax normalization constant be $T + H$, and the sparse attention normalization constant be $T$.

\begin{lemma} 
\label{lem:delta}
w.l.o.g. Consider an arbitrary row in the attention matrix $\mathbf{a}$ and arbitrary column of the values $\mathbf{v}$, with both $\mathbf{a}$ and $\mathbf{v}$ being sorted according to rank of $\mathbf{a}$ such that $\mathbf{a} = (a_{r(1)} \leq a_{r(2)} \leq \dots \leq a_{r(N)})$. For a top-$k$ sparse attention matrix which only computes the top-$k$ attention scores, one only needs to compute ${\mathbf{a}^*}^\top \mathbf{v}  = \sum_{N-k+1}^N \mathbf{a^*}_i\mathbf{v}_i$. With $\mathbf{\Delta} = \mathbf{a}^\top \mathbf{v}-\mathbf{a^*}^\top \mathbf{v}$, we may bound the error of our attention approximation as,

\vspace{-1em}
\[
\left| \mathbf{\Delta} - \sum_{i=1}^{N-k} \mathbf{a}_i\,\mathbf{v}_i \right| \leq \frac{H}{H + T} \max_{i > N - k} |v_i|
\]

\end{lemma}
\vspace{-1em}
\begin{proof} See~\cref{sec:delta-lemma-restate}.\end{proof}

We ultimately seek a shift in the attention outputs such that $\mathbf{A^*V} + \mathbf{\Delta} \approx \mathbf{AV}$. Trivially, if we choose $\mathbf{\Delta} = \mathbf{AV} - \mathbf{A^*}\mathbf{V}$, we have exact equality; however, calculating $\mathbf{A}$ requires the full quadratic attention procedure that we wish to avoid. As $\mathbf{AV - A^*V} \approx \mathbf{A^\mathbf{\Delta}V}$, if we further assume that $(\mathbf{A^\mathbf{\Delta}V})_i \approx (\mathbf{A^\mathbf{\Delta}V})_{i+\nu}$ for $\nu \in \{1, \dots, \gamma\}$ and $\gamma \in \mathbb{N}$, then we may approximate $(\mathbf{AV})_{i+\nu} \approx (\mathbf{A^\mathbf{\Delta}V)}_{i} + (\mathbf{A^*V})_{i + \nu}$. Under this approximation, one only needs to compute every $\gamma^{\text{th}}$ row of the attention matrix, which maintains a sparse computation by only computing a subset of rows of $\mathbf{A}$. To do this, we select a fixed fraction of row indices from $\mathbf{Q}$, such that,

\vspace{-0.75em}
\begin{equation}
    \label{eq:q-index-selection}
    \widetilde{\mathbf{Q}}_{\lfloor \frac{i}{\gamma} \rfloor} = \mathbf{Q}_i \implies i \bmod \gamma = 0; \quad \forall \quad  i \in \{1\;..\;N\}
\end{equation}
\vspace{-0.75em}

and therefore $\mathbf{\widetilde{A}}\mathbf{V} = \sigma(\widetilde{\mathbf{Q}} \mathbf{K}^\top)\mathbf{V}$ which is sparse in the query dimension, but dense in the key dimension. One possible approach would be to substitute this representation into the appropriate rows of the sparse output $\mathbf{A}^*\mathbf{V}$ such that the final representation $\widehat{\mathbf{A}}$, is the following,

\vspace{-1em}
\begin{equation}
    \label{eq:recompute}
    \left(\widehat{\mathbf{A}} \mathbf{V}\right)_i = \left(\mathbf{A}^*\mathbf{V}\right)_i + \overbrace{\mathbbm{1}\{i \bmod \gamma = 0\} \left[ \mathbf{\widetilde{A}V}_{\lfloor \frac{i}{\gamma} \rfloor} - \left(\mathbf{A}^*\mathbf{V}\right)_{\lfloor \frac{i}{\gamma} \rfloor \gamma} \right]}^{\text{make a dense output row if $i \bmod \gamma = 0$}}; \quad \forall \quad  i \in \{1\;..\;N\}
\end{equation}

We dub this approach as `recompute', as we are essentially using the sparse representation with some densely computed output tokens interwoven at regular intervals. However, we find that this approach still does not shift the distribution of attention outputs far enough towards the expected representation under quadratic attention~(see \cref{fig:spearman-cos}). Therefore, in order to apply a shift to all tokens in the output of $\mathbf{A}^*\mathbf{V}$ while maintaining a sparse computation, we instead apply the following correction to the sparse attention output,

\vspace{-1.25em}
\begin{align}
    \label{eq:delta-final}
    \left(\widehat{\mathbf{A}} \mathbf{V}\right)_i 
    &= \left(\mathbf{A}^*\mathbf{V}\right)_i + (\mathbf{A^\Delta V})_{\lfloor \frac{i}{\gamma} \rfloor \gamma} \\
    &= \left(\mathbf{A}^*\mathbf{V}\right)_i + \underbrace{\left[ \mathbf{\tilde{A}V}_{\lfloor \frac{i}{\gamma} \rfloor} - \left(\mathbf{A}^*\mathbf{V}\right)_{\lfloor \frac{i}{\gamma} \rfloor \gamma} \right]}_{\mathbf{\Delta} \;\; \text{correction term}}
\end{align}
\vspace{-1em}

\begin{wrapfigure}[16]{r}{0.44\linewidth}
\vspace{-1.2em} 
\rule{\linewidth}{1.5pt}
\vspace{-1.2em}
\captionof{algorithm}{
$\mathbf{\Delta}$ Attention Algorithm
}\label{alg:delta}
\vspace{-0.2em}
\rule{\linewidth}{1pt}

\begin{algorithmic}
\Require $f()$, $f^*()$ $\mathbf{Q,K,V}$, $\gamma$

\State \text{\small \textcolor{gray}{// sparse attention for all of $\mathbf{Q}$}}
\State $\mathbf{A}^*\mathbf{V} \gets f^*(\mathbf{Q,K,V})$

\State $\mathbf{\widetilde{Q}} \gets $ Equation 4
\State \text{\small \textcolor{gray}{// dense attention every $\gamma^{\text{th}}$ query}}
\State $\mathbf{\tilde{A}}\mathbf{V} \gets f(\mathbf{\widetilde{Q},K,V})$ 

\State \text{\small \textcolor{gray}{// collect proper indices for $\mathbf{\Delta}$ construction}}
\State $\delta \gets \{i \mid i \bmod \gamma = 0\}$
\State $\mathbf{\Delta} \gets \mathbf{\tilde{A}}\mathbf{V} - (\mathbf{A^*}\mathbf{V})_{i \in \delta}$ 

\State \text{\small \textcolor{gray}{// repeat $\mathbf{\Delta}$ and apply correction}}
\State $\widehat{\mathbf{A}}\mathbf{V} = \mathbf{A}^*\mathbf{V} + \text{repeat}(\mathbf{\Delta}, \gamma)$

\State return $\widehat{\mathbf{A}} \mathbf{V}$
\vspace{-0.5em}
\end{algorithmic}
\rule{\linewidth}{1pt}
\end{wrapfigure}

Which is equivalent to swapping in a dense row of the attention matrix at every $\gamma^{\text{th}}$ row, and applying the difference between the dense and sparse attention for the previous $\gamma^{\text{th}}$ row otherwise. A visual depiction of this process can be seen in~\cref{fig:concept}, and pseudocode in~\cref{alg:delta}. Since our method is applied directly on the attention outputs, we may utilize existing sparse attention kernels to compute $\mathbf{A^*V}$ and make use of a minimally modified flash attention kernel to compute our query-sparse attention $\mathbf{\tilde{A}V}$.

Assuming that a row index $j$ of the attention matrix is not evenly divisible by $\gamma$, this means that an attention differential from a previous row is being applied to the current row $j$.  The intuition from this operation comes from prior works which have studied attention locality~\citep{hip}, finding that the difference between attention scores for neighboring tokens is generally small. Likewise, our conjecture is that the low attention score regions from neighboring rows of the attention matrix also have a negligible difference, allowing for the less important part of the row of the attention matrix to be reused multiple times. Specifically, as stated above~\cref{eq:q-index-selection}, we assume that $(\mathbf{A^\mathbf{\Delta}V})_i \approx (\mathbf{A^\mathbf{\Delta}V})_{i+\nu}$ for $\nu \in \{1, \dots, \gamma\}$ and $\gamma \in \mathbb{N}$. To validate this assumption, we examine the average cosine similarity of $(\mathbf{A^\mathbf{\Delta}V})_i$ within a $\gamma$ window on an input from the RULER 131K task set for various values of $\gamma$ in~\cref{fig:delta-diff-cos}. We find a high average cosine similarity within the window, implying that $(\mathbf{A^\mathbf{\Delta}V})_i$ may be reused for multiple rows of the attention output. 

\vspace{-0.5em}
\section{Experiments}
\label{sec:experiments}
\vspace{-0.5em}

We evaluate our method in terms of perplexity (PPL) and long context perplexity using the LongPPL~\citep{long-ppl} metric on a QA version of the PG19~\citep{pg19} test set, which was recently proposed as a long context understanding dataset~\citep{pg19-long-qa}. We also provide evaluations of our method on the RULER~\citep{ruler} benchmark, which tests models' performance under a number of long context retrieval tasks. Additionally, we evaluate our $\mathbf{\Delta}$ Attention on Infinite-Bench~\citep{inf-bench}, and also provide analysis that evaluates the effect of our $\mathbf{\Delta}$ correction on the distribution of attention outputs and scores, and overall attention latency. Our work considers that the decoding process shown in~\cref{eq:generation} is dense along the key dimension and should be able to successfully learn from previously encoded information during the sparse prefill.

We apply our method in conjunction with the sparse attention methods Streaming LLM~\citep{sink}, HiP~\citep{hip,infhip}, and MInference~\citep{minference}, on models from the Llama~\citep{llama} (3.1 and 4), and Mistral~\citep{mistral} model families. Unless otherwise noted, our standard setting uses $\gamma = 64$ which means we calculate every $64^{\text{th}}$ query row (approximately $98.5$\% sparsity) in the attention computation required by $\mathbf{\Delta}$ Attention.

\textbf{RULER.} For baselines on needle-in-a-haystack type tasks, we compare our method in addition to Streaming LLM, HiP, and MInference for both Llama and Mistral models. In all cases, $\mathbf{\Delta}$ Attention shows a large improvement upon the given sparse methods, and especially at the longer context lengths in~\cref{tab:ruler}. In particular, we note an improvement of nearly 37\%pt over Streaming LLM with the same 2K window size for 131K with Llama 3.1. For Streaming LLM, if we adjust for the extra computation needed by our method, we find that the approximate window size of our method is $3072$ (see~\cref{sec-approx-window-size} for calculation). This is due to the fact that we also use a sliding window of 2048 and compute every 64$^{\text{th}}$ row of the lower triangle in the attention matrix. Therefore, even when Streaming LLM is allowed a higher computational budget of a 4K window, $\mathbf{\Delta}$ Attention still results in an increase of 34\%pt, more than doubling the accuracy of Streaming LLM (+112\%, relative). Even when Streaming LLM is allowed a 32K window, Streaming LLM + $\mathbf{\Delta}$ with a 2K window still delivers higher accuracy.

\begin{table}[t]
\centering
\vspace{-0.25in}
\caption{ 
\textbf{RULER (Llama 3.1 8B Instruct and Mistral NeMo 12B)} for sparse attention methods. Adding $\mathbf{\Delta}$ Attention results in better overall accuracy, with the largest improvement occurring at the longest context length and on the most naive sparse method (Streaming LLM). Colors are relative to each attention method group + Flash Attention 2.}
\label{tab:ruler}
\vspace{0.5em}
\setlength{\tabcolsep}{3.5pt}
\newsavebox{\mytablebox}
\newlength{\mytableheight}
\savebox{\mytablebox}{%
\begin{tabular}{lccccccccccccccc}
\toprule
Model 
    & \multicolumn{10}{c}{Llama 3.1 8B Instruct} 
    & \multicolumn{5}{c}{Mistral NeMo 12B} \\
\midrule
\raisebox{1.0em}{\makecell{Attn.\\Method}}
    & \rotatebox[origin=l, x=3em]{-90}{\footnotesize Flash Attn.} 
    & \rotatebox[origin=l, x=3em]{-90}{\footnotesize Str. \tiny LLM} 
    & \rotatebox[origin=l, x=3em]{-90}{\footnotesize Str. \tiny LLM} 
    & \rotatebox[origin=l, x=3em]{-90}{\footnotesize Str. \tiny LLM} 
    & \rotatebox[origin=l, x=3em]{-90}{\footnotesize Str. \tiny LLM} 
    & \rotatebox[origin=l, x=3em]{-90}{\footnotesize \textbf{Str. \tiny LLM+}$\mathbf{\Delta}$}
    & \rotatebox[origin=l, x=3em]{-90}{\footnotesize MInf.} 
    & \rotatebox[origin=l, x=3em]{-90}{\footnotesize \textbf{MInf.+}$\mathbf{\Delta}$} 
    & \rotatebox[origin=l, x=3em]{-90}{\footnotesize HiP} 
    & \rotatebox[origin=l, x=3em]{-90}{\footnotesize \textbf{HiP+}$\mathbf{\Delta}$}
    & \rotatebox[origin=l, x=3em]{-90}{\footnotesize Flash Attn.} 
    & \rotatebox[origin=l, x=3em]{-90}{\footnotesize Str. \tiny LLM}
    & \rotatebox[origin=l, x=3em]{-90}{\footnotesize \textbf{Str. \tiny LLM+}$\mathbf{\Delta}$} 
    & \rotatebox[origin=l, x=3em]{-90}{\footnotesize HiP} 
    & \rotatebox[origin=l, x=3em]{-90}{\footnotesize \textbf{HiP+}$\mathbf{\Delta}$} \\
\midrule
Wind. & - & 2K & 4K & 16K & 32K & 2K & 3K & 3K & 3K & 3K & - & 2K & 2k & 3K & 3K\\
\midrule
4K & \cellcolor[HTML]{D8FFD8}96\tiny.74 &\cellcolor[HTML]{FFEDD8}90\tiny.52 &\cellcolor[HTML]{D9FFD8}96\tiny.71 &\cellcolor[HTML]{D9FFD8}96\tiny.71 &\cellcolor[HTML]{D9FFD8}96\tiny.71 &\cellcolor[HTML]{DAFFD8}96\tiny.54 &\cellcolor[HTML]{D8FFD8}96\tiny.74 &\cellcolor[HTML]{D9FFD8}96\tiny.71 &\cellcolor[HTML]{D8FFD8}96\tiny.80 &\cellcolor[HTML]{E4FFD8}96\tiny.31 &\cellcolor[HTML]{D8FFD8}90\tiny.60 &\cellcolor[HTML]{FFDFD8}71\tiny.01 &\cellcolor[HTML]{D9FFD8}90\tiny.42 &\cellcolor[HTML]{E0FFD8}90\tiny.36 &\cellcolor[HTML]{DAFFD8}90\tiny.55 \\
8K & \cellcolor[HTML]{D8FFD8}93\tiny.25 &\cellcolor[HTML]{FFDCD8}60\tiny.53 &\cellcolor[HTML]{D8FFD8}93\tiny.76 &\cellcolor[HTML]{D8FFD8}93\tiny.76 &\cellcolor[HTML]{D8FFD8}93\tiny.76 &\cellcolor[HTML]{DAFFD8}92\tiny.25 &\cellcolor[HTML]{D8FFD8}93\tiny.65 &\cellcolor[HTML]{D8FFD8}93\tiny.69 &\cellcolor[HTML]{D8FFD8}94\tiny.56 &\cellcolor[HTML]{D8FFD8}94\tiny.43 &\cellcolor[HTML]{D8FFD8}87\tiny.67 &\cellcolor[HTML]{FFDAD8}44\tiny.89 &\cellcolor[HTML]{DCFFD8}85\tiny.38 &\cellcolor[HTML]{D8FFD8}88\tiny.36 &\cellcolor[HTML]{D8FFD8}87\tiny.69 \\
16K & \cellcolor[HTML]{D8FFD8}90\tiny.99 &\cellcolor[HTML]{FFDAD8}38\tiny.13 &\cellcolor[HTML]{F9FFD8}68\tiny.07 &\cellcolor[HTML]{D8FFD8}91\tiny.15 &\cellcolor[HTML]{D8FFD8}91\tiny.15 &\cellcolor[HTML]{DCFFD8}88\tiny.66 &\cellcolor[HTML]{D8FFD8}92\tiny.32 &\cellcolor[HTML]{D8FFD8}91\tiny.34 &\cellcolor[HTML]{D8FFD8}94\tiny.10 &\cellcolor[HTML]{D8FFD8}93\tiny.86 &\cellcolor[HTML]{D8FFD8}81\tiny.82 &\cellcolor[HTML]{FFDAD8}33\tiny.28 &\cellcolor[HTML]{DEFFD8}78\tiny.07 &\cellcolor[HTML]{FFF2D8}78\tiny.07 &\cellcolor[HTML]{E2FFD8}81\tiny.08 \\
32K & \cellcolor[HTML]{D8FFD8}85\tiny.84 &\cellcolor[HTML]{FFD9D8}30\tiny.25 &\cellcolor[HTML]{FFEBD8}43\tiny.38 &\cellcolor[HTML]{FFFDD8}56\tiny.32 &\cellcolor[HTML]{D8FFD8}85\tiny.83 &\cellcolor[HTML]{DEFFD8}81\tiny.27 &\cellcolor[HTML]{D8FFD8}86\tiny.75 &\cellcolor[HTML]{D8FFD8}85\tiny.96 &\cellcolor[HTML]{D8FFD8}89\tiny.92 &\cellcolor[HTML]{D8FFD8}89\tiny.39 &\cellcolor[HTML]{D8FFD8}62\tiny.54 &\cellcolor[HTML]{FFD9D8}12\tiny.27 &\cellcolor[HTML]{FFFBD8}34\tiny.76 &\cellcolor[HTML]{FFEED8}58\tiny.76 &\cellcolor[HTML]{F8FFD8}60\tiny.38 \\
65K & \cellcolor[HTML]{D8FFD8}85\tiny.25 &\cellcolor[HTML]{FFD9D8}18\tiny.59 &\cellcolor[HTML]{FFEAD8}34\tiny.08 &\cellcolor[HTML]{FFF3D8}41\tiny.28 &\cellcolor[HTML]{F7FFD8}58\tiny.35 &\cellcolor[HTML]{E4FFD8}75\tiny.22 &\cellcolor[HTML]{E9FFD8}84\tiny.43 &\cellcolor[HTML]{F9FFD8}83\tiny.67 &\cellcolor[HTML]{FFF9D8}82\tiny.51 &\cellcolor[HTML]{DEFFD8}84\tiny.89 &\cellcolor[HTML]{D8FFD8}46\tiny.89 &\cellcolor[HTML]{FFD8D8}03\tiny.28 &\cellcolor[HTML]{FFEFD8}16\tiny.22 &\cellcolor[HTML]{FFDED8}35\tiny.87 &\cellcolor[HTML]{FAFFD8}41\tiny.56 \\
131K & \cellcolor[HTML]{D8FFD8}73\tiny.16 &\cellcolor[HTML]{FFD9D8}27\tiny.45 &\cellcolor[HTML]{FFDED8}30\tiny.32 &\cellcolor[HTML]{FFEFD8}40\tiny.51 &\cellcolor[HTML]{FFFDD8}49\tiny.17 &\cellcolor[HTML]{E7FFD8}64\tiny.40 &\cellcolor[HTML]{FFE6D8}65\tiny.73 &\cellcolor[HTML]{D8FFD8}73\tiny.31 &\cellcolor[HTML]{FFEED8}68\tiny.74 &\cellcolor[HTML]{D8FFD8}73\tiny.71 &\cellcolor[HTML]{D8FFD8}18\tiny.09 &\cellcolor[HTML]{FFDCD8}02\tiny.25 &\cellcolor[HTML]{FFD8D8}01\tiny.44 &\cellcolor[HTML]{FFDBD8}10\tiny.10 &\cellcolor[HTML]{FFE2D8}10\tiny.93 \\
\midrule
Avg. & \cellcolor[HTML]{D8FFD8}87\tiny.54 &\cellcolor[HTML]{FFDAD8}44\tiny.25 &\cellcolor[HTML]{FFF7D8}61\tiny.05 &\cellcolor[HTML]{F7FFD8}69\tiny.96 &\cellcolor[HTML]{E7FFD8}79\tiny.16 &\cellcolor[HTML]{E0FFD8}\underline{\textbf{83\tiny.06}} &\cellcolor[HTML]{EFFFD8}86\tiny.60 &\cellcolor[HTML]{DBFFD8}\underline{\textbf{87\tiny.44}} &\cellcolor[HTML]{D8FFD8}87\tiny.77 &\cellcolor[HTML]{D8FFD8}\underline{\textbf{88\tiny.76}} &\cellcolor[HTML]{D8FFD8}64\tiny.60 &\cellcolor[HTML]{FFDAD8}27\tiny.83 &\cellcolor[HTML]{F4FFD8}\underline{\textbf{51\tiny.05}} &\cellcolor[HTML]{FFECD8}60\tiny.25 &\cellcolor[HTML]{FAFFD8}\underline{\textbf{62\tiny.03}} \\
\bottomrule
\end{tabular}
}
\setlength{\mytableheight}{5.2cm}

\resizebox{\linewidth}{!}{%
  \begin{tikzpicture}
    \node[inner sep=0pt] (table) {\usebox{\mytablebox}};
    \draw[thick, dash pattern=on 4pt off 3pt] ($(table.south west)+(1.35cm,0.1cm)$) -- ++(0, \mytableheight);
    \draw[thick, dash pattern=on 4pt off 3pt] ($(table.south west)+(2.2cm,0.1cm)$) -- ++(0, \mytableheight);
    \draw[thick, dash pattern=on 4pt off 3pt] ($(table.south west)+(6.5cm,0.1cm)$) -- ++(0, \mytableheight);
    \draw[thick, dash pattern=on 4pt off 3pt] ($(table.south west)+(9.95cm,0.1cm)$) -- ++(0, 5.75cm);
    \draw[thick, dash pattern=on 4pt off 3pt] ($(table.south west)+(8.25,0.1cm)$) -- ++(0, \mytableheight);
    \draw[thick, dash pattern=on 4pt off 3pt] ($(table.south west)+(10.82cm,0.1cm)$) -- ++(0, \mytableheight);
    \draw[thick, dash pattern=on 4pt off 3pt] ($(table.south west)+(12.55cm,0.1cm)$) -- ++(0, \mytableheight);
  \end{tikzpicture}%
}
\vspace{-1.5em}
\end{table}

\begin{wraptable}[11]{r}{0.5\linewidth}
\vspace{-1.35em}
\setlength{\tabcolsep}{5.5pt}
\caption{ 
\textbf{Perplexity on PG19 Long QA~\citep{pg19-long-qa}.} Our simple $\mathbf{\Delta}$ correction results in a significant drop in both PPL and Long PPL.
}
\label{tab:pg19-longqa}
\vspace{-0.15em}
\resizebox{\linewidth}{!}{
    \begin{tabular}{lll}
        \toprule
        Method & Long PPL $\downarrow$ & PPL $\downarrow$ \\
        \midrule
        Flash Attention 2 & 5.11 (-) & 3.33 (-) \\
        \midrule
        Streaming LLM & 7.02 (+1.91) & 3.54 (+0.21) \\
        \textbf{Streaming LLM + $\mathbf{\Delta}$} & \textbf{5.96 (+0.85)} & \textbf{3.41 (+0.08)} \\
        \midrule
        HiP Attention & 6.29 (+1.18) & 3.48 (+0.15) \\
        \textbf{HiP Attention + $\mathbf{\Delta}$} & \textbf{5.45 (+0.34)} & \textbf{3.37 (+0.04)} \\
        \bottomrule
    \end{tabular}
}
\end{wraptable}

\textbf{Perplexity (PPL) and Long Perplexity (LongPPL).} We generated a QA dataset based on the PG19 test set according to the procedure outlined by~\cite{pg19-long-qa}. This results in a long context task where an entire book is used as context, along with a series of LLM-generated questions and answer pairs with total context lengths of approximately 100K. In order to excel at this task, a model must be able to retain all information and facts from the text, which may be asked in the follow-up QA session. We evaluate both PPL and LongPPL, where the latter metric selects a subset of tokens that are found to rely heavily on long context for the final loss calculation. LongPPL has been shown to have a stronger correlation with long context performance over PPL~\citep{long-ppl}. We use Llama 3.1 8B instruction-tuned models for this experiment. Results can be seen in~\cref{tab:pg19-longqa,fig:pg19-longqa-query-stride}. When our $\mathbf{\Delta}$ Attention is applied on top of both HiP and Streaming LLM, we achieve between a 50-75\% reduction in the PPL performance gap between quadratic attention. This trend holds true for both PPL and LongPPL. \cref{fig:pg19-longqa-query-stride} shows the effect of varying the $\gamma$ parameter form 8-256. As $\gamma$ also controls the sparsity, we find that as the sparsity increases, both perplexity metrics tend to rise.

\begin{table}[t]
\vspace{-0.25in}
\newcolumntype{C}[1]{>{\centering\let\newline\\\arraybackslash\hspace{0pt}}m{#1}}
\centering
\caption{ 
\textbf{$\infty$-bench results.} Colors are made relative to the best and worst metrics within each model group, with Flash Attention being part of every group. Our $\mathbf{\Delta}$ correction improves overall performance in every case. En.QAR displays recall for the En.QA subset.
}
\label{tab:infbench}
\vspace{0.5em}
\setlength{\tabcolsep}{2.5pt}
\newsavebox{\mytableboxruler}
\newlength{\mytableheightruler}
\savebox{\mytableboxruler}{%
\begin{tabular}{clc C{1.0cm} C{1.0cm} C{1.0cm} C{1.0cm} C{1.0cm} C{1.0cm} C{1.0cm} C{1.0cm} C{1.0cm}}
\toprule
Model 
    & Method 
    & Ctx Len. 
    & \scalebox{0.8}[1.0]{En.MC} 
    & \scalebox{0.8}[1.0]{En.QA} 
    & \scalebox{0.7}[1.0]{En.QAR} 
    & \scalebox{0.8}[1.0]{En.Sum} 
    & \scalebox{0.8}[1.0]{Passkey} 
    & \scalebox{0.7}[1.0]{Number} 
    & \scalebox{1.0}[1.0]{KV} 
    & \scalebox{0.8}[1.0]{Math.F} 
    & \scalebox{1.0}[1.0]{Avg.} \\
\midrule
\multirow{5.7}{*}{\makecell{Llama 3.1\\8B Instruct}} &Flash Attention & 126K & \cellcolor[HTML]{D8FFD8}64.19 &\cellcolor[HTML]{D8FFD8}35.89 &\cellcolor[HTML]{D8FFD8}44.69 &\cellcolor[HTML]{D8FFD8}31.59 &\cellcolor[HTML]{D8FFD8}99.13 &\cellcolor[HTML]{D8FFD8}99.83 &\cellcolor[HTML]{D8FFD8}92.40 &\cellcolor[HTML]{D8FFD8}24.86 &\cellcolor[HTML]{D8FFD8}61.57 \\
\cmidrule{2-12}
& HiP & 126K & \cellcolor[HTML]{FFE1D8}54.15 &\cellcolor[HTML]{FFE4D8}31.49 &\cellcolor[HTML]{FFE2D8}38.12 &\cellcolor[HTML]{F7FFD8}31.06 &\cellcolor[HTML]{FFDED8}75.08 &\cellcolor[HTML]{FFF6D8}96.10 &\cellcolor[HTML]{FFD9D8}30.60 &\cellcolor[HTML]{FFDED8}18.86 &\cellcolor[HTML]{FFDED8}46.93 \\
& \textbf{HiP +} $\mathbf{\Delta}$ & 126K & \cellcolor[HTML]{EDFFD8}61.14 &\cellcolor[HTML]{F9FFD8}33.70 &\cellcolor[HTML]{E4FFD8}43.54 &\cellcolor[HTML]{E9FFD8}31.30 &\cellcolor[HTML]{D8FFD8}100.0 &\cellcolor[HTML]{EFFFD8}97.97 &\cellcolor[HTML]{F4FFD8}69.60 &\cellcolor[HTML]{D8FFD8}25.71 &\cellcolor[HTML]{EAFFD8}\textbf{57.87} \\
\cmidrule{2-12}
& Str. LLM & 126K & \cellcolor[HTML]{FFDAD8}27.95 &\cellcolor[HTML]{FFD9D8}07.25 &\cellcolor[HTML]{FFD9D8}14.67 &\cellcolor[HTML]{FFDCD8}20.57 &\cellcolor[HTML]{FFD8D8}02.71 &\cellcolor[HTML]{FFD8D8}01.36 &\cellcolor[HTML]{FFD9D8}01.20 &\cellcolor[HTML]{D8FFD8}25.14 &\cellcolor[HTML]{FFD9D8}12.51 \\
& \textbf{Str. LLM +} $\mathbf{\Delta}$ & 126K & \cellcolor[HTML]{E9FFD8}56.33 &\cellcolor[HTML]{F5FFD8}24.93 &\cellcolor[HTML]{F5FFD8}33.35 &\cellcolor[HTML]{F7FFD8}26.95 &\cellcolor[HTML]{DBFFD8}96.27 &\cellcolor[HTML]{F0FFD8}68.81 &\cellcolor[HTML]{FFD8D8}00.40 &\cellcolor[HTML]{D8FFD8}25.43 &\cellcolor[HTML]{F7FFD8}\textbf{41.66} \\
\midrule
\midrule
\multirow{5.7}{*}{\makecell{Llama 4\\Scout 109B}} &Flash Attention & 384K & \cellcolor[HTML]{D8FFD8}82.10 &\cellcolor[HTML]{D8FFD8}44.34 &\cellcolor[HTML]{D8FFD8}48.82 &\cellcolor[HTML]{D8FFD8}35.30 &\cellcolor[HTML]{D8FFD8}100.0 &\cellcolor[HTML]{D8FFD8}100.0 &\cellcolor[HTML]{D8FFD8}99.20 &\cellcolor[HTML]{D8FFD8}43.14 &\cellcolor[HTML]{D8FFD8}69.11 \\
\cmidrule{2-12}
& HiP & 384K & \cellcolor[HTML]{D8FFD8}74.67 &\cellcolor[HTML]{D8FFD8}43.19 &\cellcolor[HTML]{D8FFD8}48.29 &\cellcolor[HTML]{D8FFD8}34.28 &\cellcolor[HTML]{D8FFD8}100.0 &\cellcolor[HTML]{D8FFD8}99.83 &\cellcolor[HTML]{D8FFD8}99.40 &\cellcolor[HTML]{D8FFD8}41.14 &\cellcolor[HTML]{D8FFD8}67.60 \\
& \textbf{HiP +} $\mathbf{\Delta}$ & 384K & \cellcolor[HTML]{D8FFD8}78.60 &\cellcolor[HTML]{D8FFD8}42.84 &\cellcolor[HTML]{D8FFD8}48.14 &\cellcolor[HTML]{D8FFD8}34.06 &\cellcolor[HTML]{D8FFD8}100.0 &\cellcolor[HTML]{DDFFD8}99.66 &\cellcolor[HTML]{D8FFD8}97.20 &\cellcolor[HTML]{D8FFD8}44.29 &\cellcolor[HTML]{D8FFD8}\textbf{68.10} \\
\cmidrule{2-12}
& Str. LLM & 384K & \cellcolor[HTML]{FFDED8}49.78 &\cellcolor[HTML]{FFDAD8}15.23 &\cellcolor[HTML]{FFDBD8}26.11 &\cellcolor[HTML]{DCFFD8}31.50 &\cellcolor[HTML]{FFDAD8}52.88 &\cellcolor[HTML]{FFD8D8}08.31 &\cellcolor[HTML]{FFD8D8}03.40 &\cellcolor[HTML]{D8FFD8}40.57 &\cellcolor[HTML]{FFDAD8}28.47 \\
& \textbf{Str. LLM +} $\mathbf{\Delta}$ & 384K & \cellcolor[HTML]{D8FFD8}73.80 &\cellcolor[HTML]{D8FFD8}37.82 &\cellcolor[HTML]{DFFFD8}43.03 &\cellcolor[HTML]{FFFAD8}30.62 &\cellcolor[HTML]{DFFFD8}94.75 &\cellcolor[HTML]{DFFFD8}91.36 &\cellcolor[HTML]{FFFDD8}46.60 &\cellcolor[HTML]{D8FFD8}40.86 &\cellcolor[HTML]{E2FFD8}\textbf{57.35} \\
\bottomrule
\end{tabular}
}
\setlength{\mytableheightruler}{5.5cm}

\resizebox{\linewidth}{!}{%
  \begin{tikzpicture}
    \node[inner sep=0pt] (table) {\usebox{\mytableboxruler}};
    \draw[thick, dash pattern=on 4pt off 3pt] ($(table.south west)+(15.0cm,0.1cm)$) -- ++(0, \mytableheightruler);
  \end{tikzpicture}%
}
\end{table}

\begin{figure}[t]
    \includegraphics[width=0.7\linewidth]{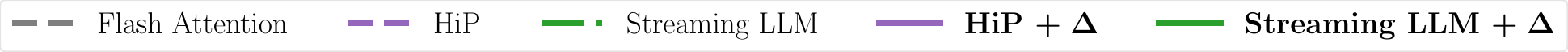}\\
    \begin{subfigure}[t]{0.65\textwidth}
    \centering
        \includegraphics[width=0.49\linewidth]{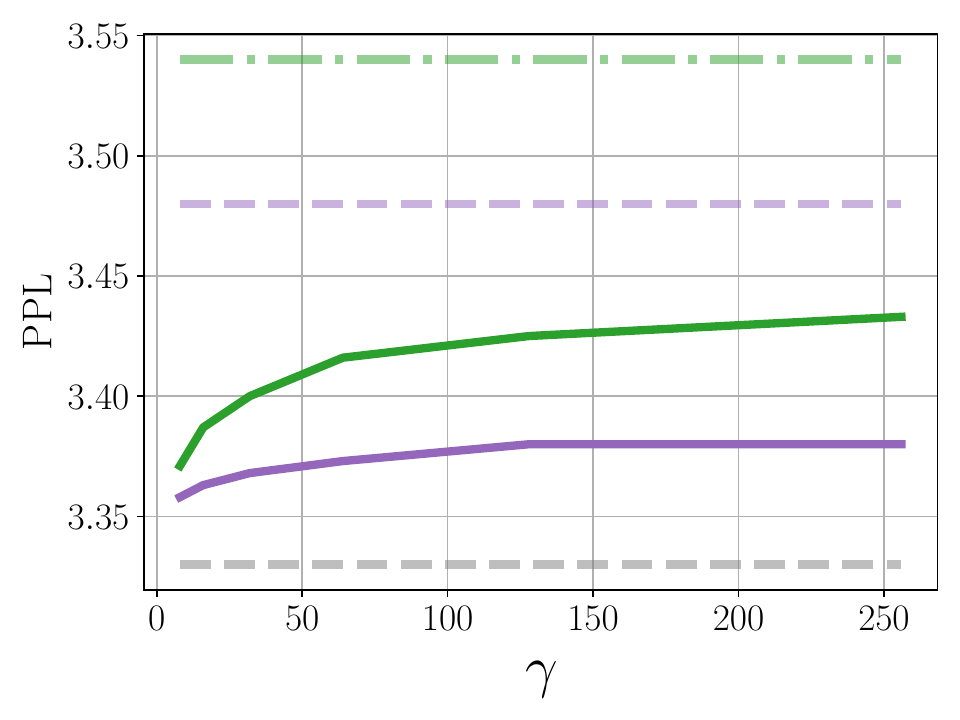}
        \includegraphics[width=0.49\linewidth]{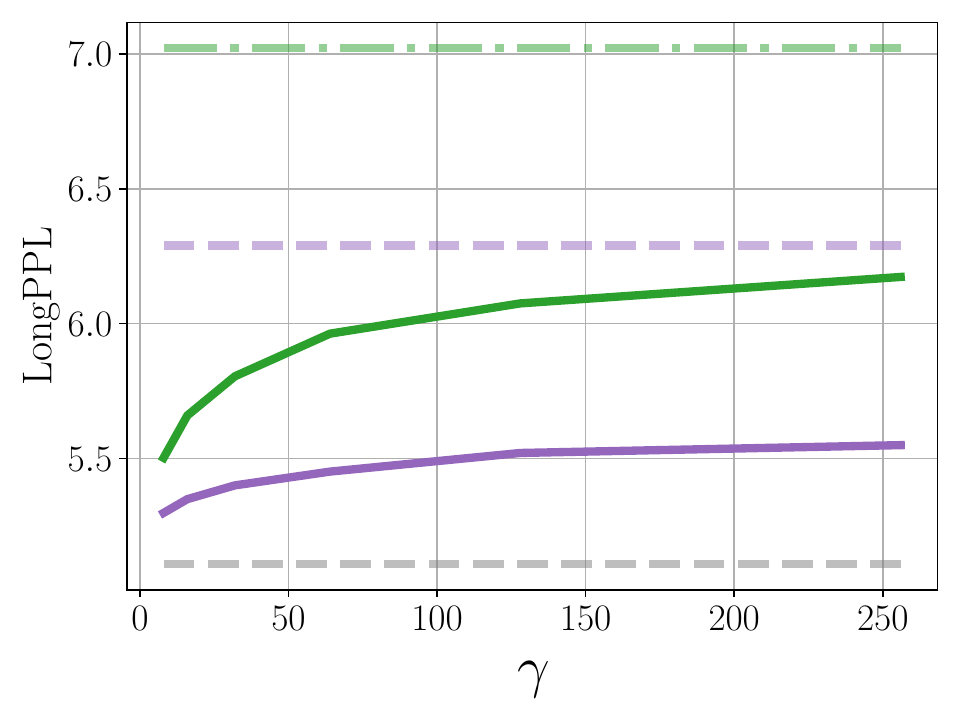}
    \vspace{-0.5em}
    \caption{Perplexity and Long Perplexity}
    \label{fig:ppl}
    \end{subfigure}%
    \begin{subfigure}[t]{0.32\textwidth}
    \centering
        \includegraphics[width=\linewidth]{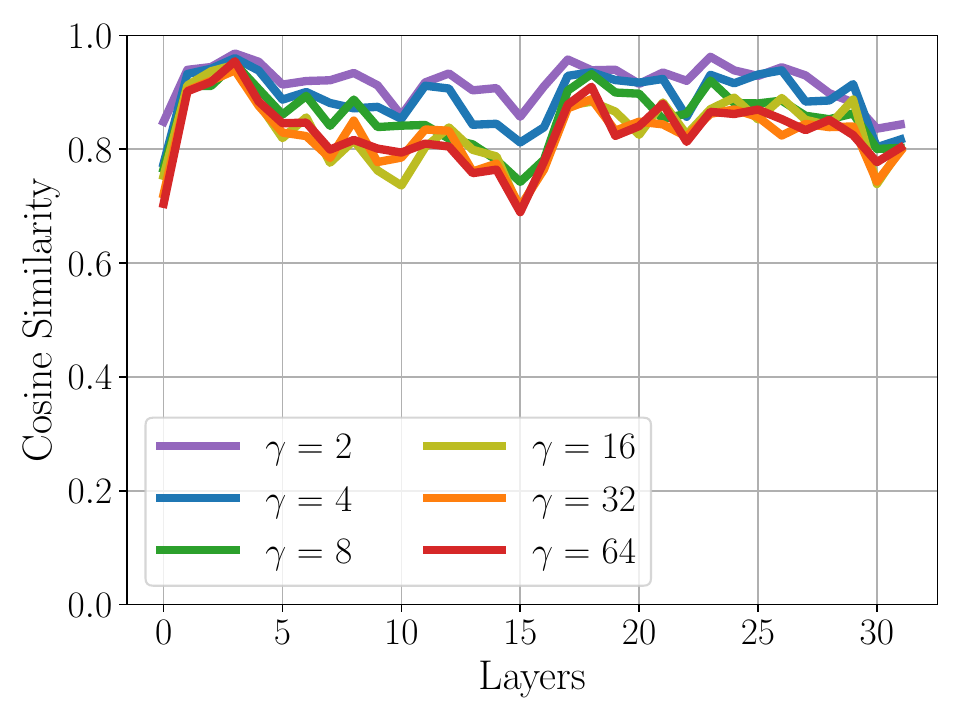}
    \vspace{-1.7em}
    \caption{cos($[\mathbf{A^\mathbf{\Delta}V}]_i, [\mathbf{A^\mathbf{\Delta}V}]_{i+\nu}$)}
    \label{fig:delta-diff-cos}
    \end{subfigure}%
    \vspace{-0.3em}
    \caption{ 
    \textbf{(\subref{fig:ppl})} Perplexity metrics for increasing $\gamma \in \{2^3, \dots, 2^8\}$ . For PPL and LongPPL, increasing the query stride shows a slight trend towards higher PPL with higher sparsity. \textbf{(\subref{fig:delta-diff-cos})} Measures the average cosine similarity between the approximate $(\mathbf{A^\mathbf{\Delta}V})_i$ and $(\mathbf{A^\mathbf{\Delta}V})_{i+\nu}$ for $\nu \in \{1, \dots, \gamma\}$ for Streaming LLM and finds a high similarity within a $\gamma$ neighborhood of attention outputs. High similarity implies $(\mathbf{A^\mathbf{\Delta}V})_i$ can be reused within the $\gamma$ neighborhood.}
    \label{fig:pg19-longqa-query-stride}
\end{figure}

\textbf{Infinite Bench.}~\citep{inf-bench} For both LLama 3.1 8B and Llama 4 Scout 109B, results are displayed in~\cref{tab:infbench}. The display colors are encoded to show the performance difference within each model group, and including flash attention in all groups. For Llama 4 (Streaming LLM), the addition of $\mathbf{\Delta}$ resulted in an increase of 40\%pt, which leads to recapturing 82\% of quadratic attention accuracy (up from 41\%). Similarly, for Llama 3.1, the addition of $\mathbf{\Delta}$ increased overall performance by 29\%pt, which moves from 20\% of full attention accuracy to recovering 67\%. The realized performance gains when applying our method to HiP result in a 10\%pt increase for Llama 3.1 and a 0.5\%pt increase for Llama 4. Note that HiP with Llama 4 only shows a total of 1.5\%pt gap in performance, which means that $\mathbf{\Delta}$ Attention was able to recapture 33\% of the total performance gap.

\subsection{Ablation \& Analysis}
\label{sec:ablation}

\textbf{Latency.} For a single attention layer, our method shows a large reduction in latency when compared to Flash Attention 2 benchmarked at 1M tokens. In~\cref{fig:latency}, HiP + $\mathbf{\Delta}$ runs more than 8 times faster. For Streaming LLM + $\mathbf{\Delta}$ this factor increases to over 32, which means that $\mathbf{\Delta}$ Attention may perform more than 32 attention operations for a single quadratic Flash Attention 2 operation. While our method does require more computation than the standalone sparse methods in~\cref{fig:sparse-latency}, the relative increase is modest in comparison to the latency of quadratic attention. MInference has been excluded from these latency results due to the current public implementation not fully utilizing hardware parallelization in this experiment. For further details, please see~\cref{sec:minference-latency-discussion}. 

\begin{figure}[t]
\centering
\vspace{-0.25in}
\begin{subfigure}[t]{0.32\textwidth}
    \centering
        \includegraphics[width=\linewidth]{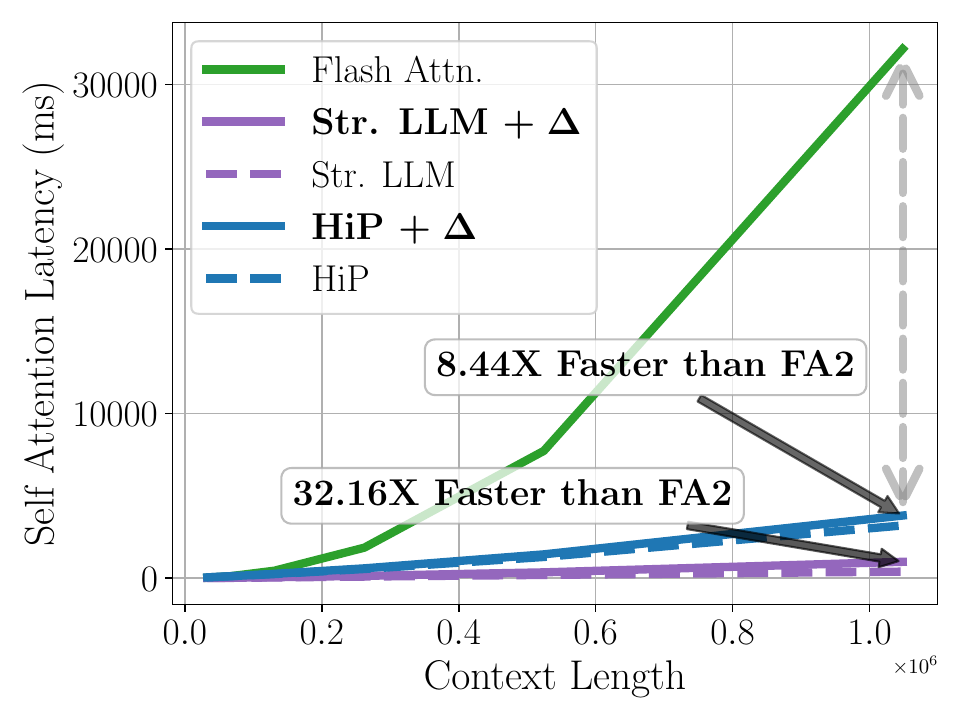}
    \caption{Latency vs. Flash Attention}
    \label{fig:flash-latency}
\end{subfigure}%
\begin{subfigure}[t]{0.325\textwidth}
    \centering
    \includegraphics[width=\linewidth]{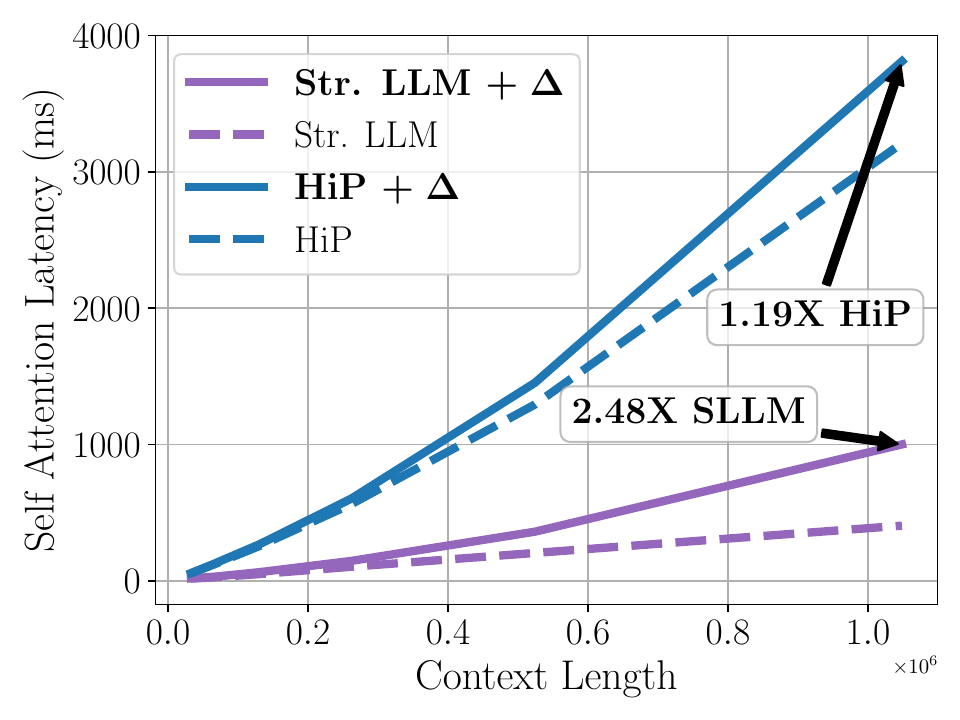}
    \caption{Latency vs. Sparse Methods}
    \label{fig:sparse-latency}
\end{subfigure}%
\begin{subfigure}[t]{0.32\textwidth}
    \centering
        \includegraphics[width=\linewidth]{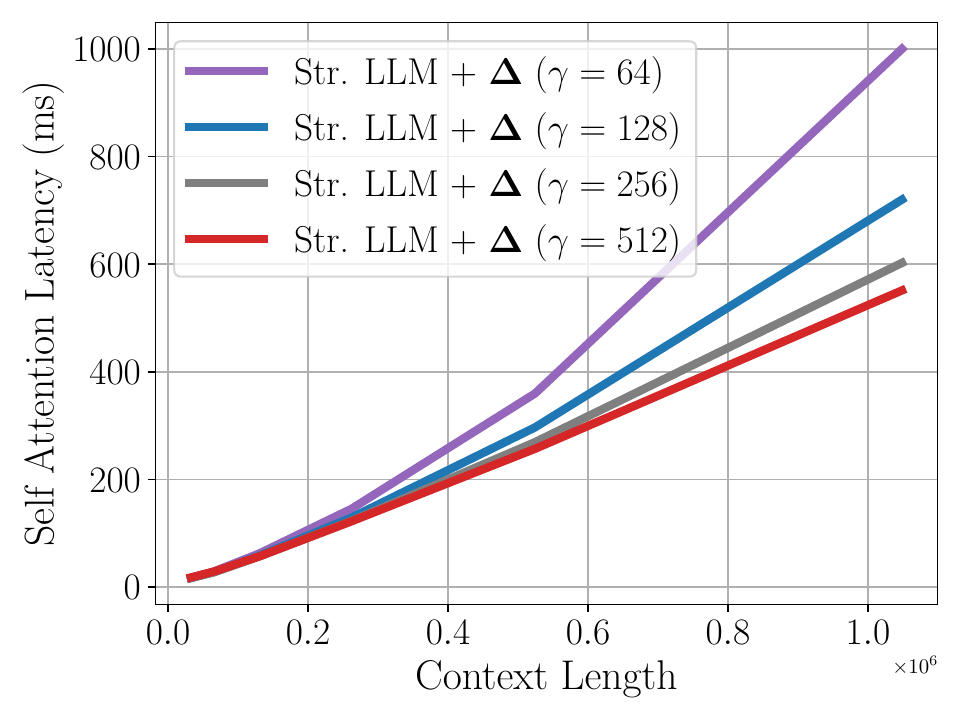}
    \caption{Latency for increasing $\gamma$}
    \label{fig:gamma-ablation}
\end{subfigure}%
\caption{\textbf{(\subref{fig:flash-latency})}  shows latency comparisons against flash attention at 1M tokens. Our method maintains most of the large latency reductions of sparse methods. \textbf{(\subref{fig:sparse-latency})} compares latency against plain sparse methods. Our method introduces a slight overhead due to requiring computation equivalent to 1.5\% of the whole attention matrix. \textbf{(\subref{fig:gamma-ablation})} evaluates the effect of different $\gamma$ parameters on latency. We find that increasing the stride between queries leads to an expected decrease in latency.}
\vspace{-1em}
\label{fig:latency}
\end{figure}

\textbf{How does the $\mathbf{\Delta}$ affect attention outputs and scores?} To study the effect of the $\mathbf{\Delta}$ correction on the attention outputs and scores, we evaluate both attention output cosine similarity and the Spearman rank correlation coefficient~\citep{spearmanr} of the attention rows for the last 128 queries of the prefill. For this, we used a sample from the MultiKey-3 RULER (131K) benchmark with the Llama3.1 8B instruction tuned model. A subset of layers is depicted in~\cref{fig:spearman-cos}, where each point in the plot and histogram is a random sample from one of the $32 \times 128$ (attention heads and queries). Additional plots for all layers in the network can be seen in~\cref{fig:spearman-cos-app,fig:spearman-cos-app-2,fig:spearman-cos-app-3} in the appendix. At the key lower layers where the induction heads are known to be most prevalent, we find that the $\mathbf{\Delta}$ correction results in a large corrective shift in both the rank correlation and cosine similarity, making both metrics much closer to the ground truth distributions of quadratic attention. Notably, only using `Recompute', which densely recomputes some rows of the attention matrix, is not enough to shift the distribution, as it is indistinguishable from the plain Streaming LLM model in \cref{fig:spearman-cos}. 

In ~\cref{sec:intro}, we stated that $\mathbf{\Delta}$ Attention shifts the distribution of attention outputs towards the distribution which would be seen under fully quadratic attention. \cref{fig:spearman-cos} provides three more examples of lower layers which show the same shift as shown in~\cref{fig:motivation}. It is notable, however, that this strong shift towards the distribution of quadratic attention is not present in all layers of the network. \cref{fig:spearman-cos-app,fig:spearman-cos-app-2,fig:spearman-cos-app-3} together show all layers. $\mathbf{\Delta}$ Attention appears to maintain a strong similarity to quadratic attention at the lower layers, which gradually dissipates until layer 10, when the three methods become indistinguishable. However, there is a sudden rise in attention output cosine similarity again towards the last layers of the network

While both the output cosine similarity and the rank correlation are important, the high rank correlation coefficient provides a crucial insight as to how the $\mathbf{\Delta}$ correction aids in improving performance. For sparse methods, the last 128 queries from a 131K context have undergone a distributional shift induced by the sparse method, which means that they no longer correctly align with the appropriate key tokens during dot-product attention. A high rank correlation, however, implies that the ranking (importance order) of dot products across an entire row of the attention matrix remains largely intact and therefore, should result in outputs with higher similarity to quadratic attention outputs. This suggests that dense decoding can now effectively access information buried deep in the prompt, which is something our experiments show sparse attention methods struggle to do. 

\begin{wraptable}[7]{r}{0.45\textwidth}
\setlength{\tabcolsep}{2.5pt}
\vspace{-1.35em}
\caption{
RULER ablation for \cref{eq:recompute} `recompute' and \cref{eq:delta-final} $\mathbf{\Delta}$.
}
\label{tab:ruler-ablation}
\vspace{-0.35em}

\newsavebox{\mytableboxrecompute}
\newlength{\mytableheightrecompute}
\savebox{\mytableboxrecompute}{%
\begin{tabular}{lccccc}
\toprule
Model & 131K & 65K & 32K & \dots & Avg. \\
\midrule
Str. LLM & \cellcolor[HTML]{FFD9D8}27.45 &\cellcolor[HTML]{FFD9D8}18.59 &\cellcolor[HTML]{FFD9D8}30.25 & \dots &\cellcolor[HTML]{FFDAD8}44.25 \\
\midrule
Str. LLM + Recompute & \cellcolor[HTML]{FAFFD8}52.67 &\cellcolor[HTML]{E7FFD8}72.71 &\cellcolor[HTML]{E2FFD8}78.39 & \dots &\cellcolor[HTML]{ECFFD8}79.99 \\
\textbf{Str. LLM + } $\mathbf{\Delta}$ & \cellcolor[HTML]{E7FFD8}64.40 &\cellcolor[HTML]{E4FFD8}75.22 &\cellcolor[HTML]{DEFFD8}81.27 & \dots &\cellcolor[HTML]{E8FFD8}\textbf{83.06} \\
\bottomrule
\end{tabular}
}

\setlength{\mytableheightrecompute}{1.4cm}

\resizebox{1.0\linewidth}{!}{%
  \begin{tikzpicture}
    \node[inner sep=0pt] (table) {\usebox{\mytableboxrecompute}};
    \draw[thick, dash pattern=on 4pt off 3pt] ($(table.south west)+(6.95cm,0.1cm)$) -- ++(0, \mytableheightrecompute);
  \end{tikzpicture}%
}

\end{wraptable}

\textbf{Does \cref{eq:recompute} or \cref{eq:delta-final} Perform Better?} In the previous paragraph we gave qualitative examples of the difference between \cref{eq:recompute} and \cref{eq:delta-final} on the attention output cosine similarity. Now we ask, how does this observed difference affect the performance of the model? \cref{tab:ruler-ablation} shows the effect of `recompute' from \cref{eq:recompute}, which recomputes a selected number of queries with dense attention and does not apply the difference to subsequent tokens in the $\gamma$ neighborhood. Only recomputing tokens results in a 37\%pt increase over all context lengths and is only 3\%pt short of matching $\mathbf{\Delta}$. However, at the longest context length, $\mathbf{\Delta}$ still delivers a more than 11\%pt increase in accuracy. 

\cref{fig:infbench-ablation} shows `recompute' compared to $\mathbf{\Delta}$ Attention for individual subsets of the RULER-131K context length.  We find that the only case where `recompute' outperforms our method is on the variable tracking subset (VT). We are unsure of the cause of this anomaly, although it is important to note that `recompute' even outperformed flash attention by approximately 15\%pt, which implies that there is some structure within this task that happened to benefit from `recompute'. In general, flash attention should represent an upper bound to sparse attention, which is what we observe in general. Note that the CWE subset of RULER is removed from this plot, as all methods (including flash attention) score 0\% on the 131K context length.

\begin{wrapfigure}[17]{r}{0.4\textwidth}
  \vspace{-2.0em}
    \includegraphics[width=\linewidth]{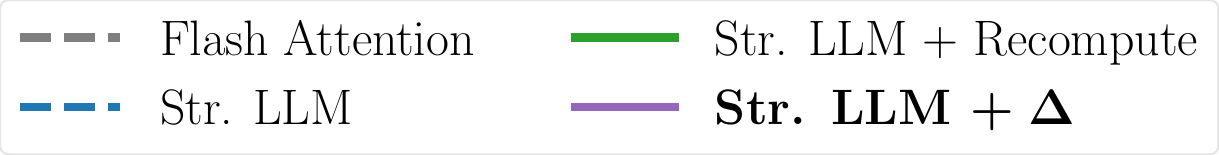}\\[-0em]
    \includegraphics[width=\linewidth]{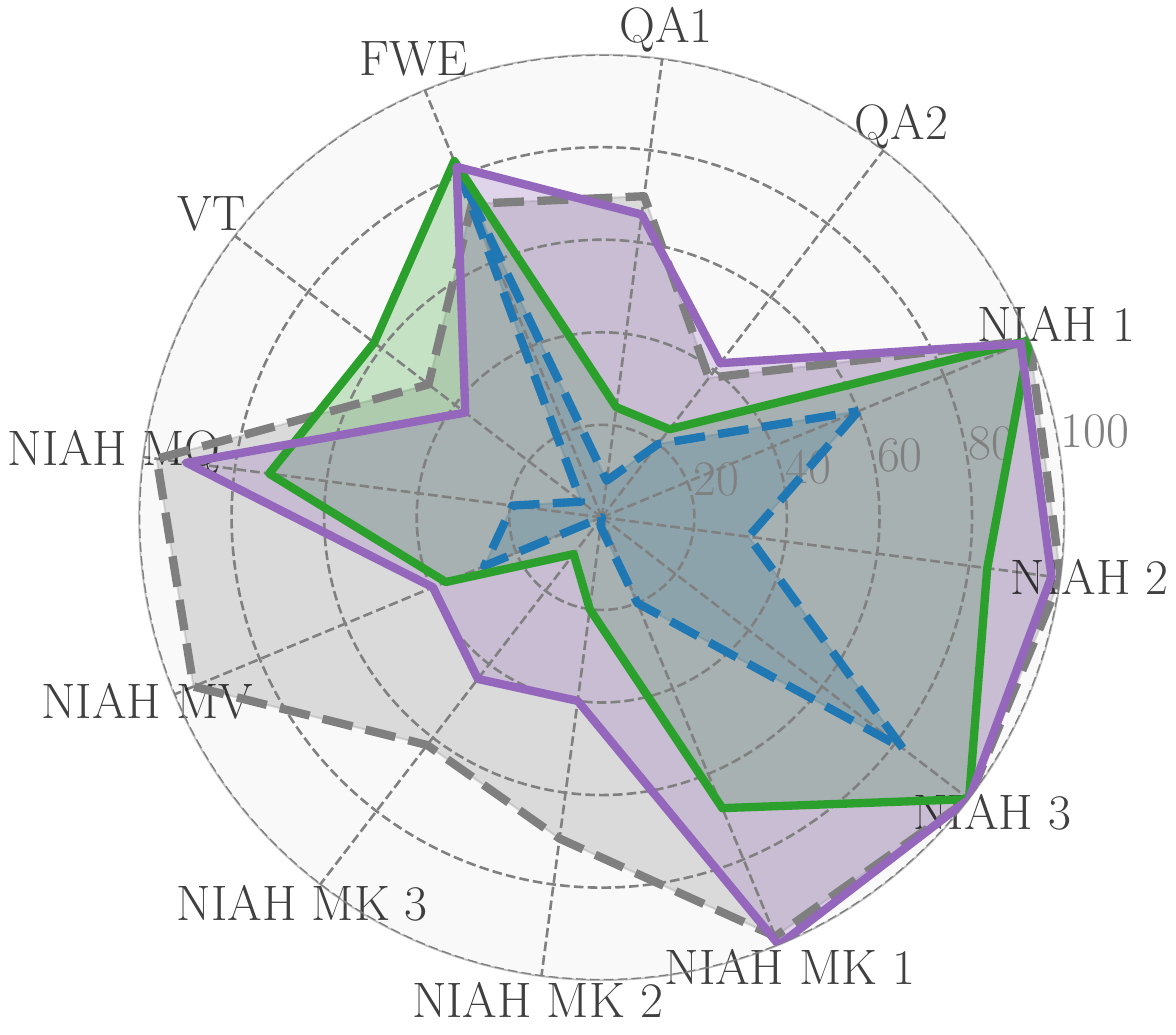}
  \vspace{-1.5em}
  \caption{ 
  Comparing the effects of \cref{eq:recompute} `recompute' and \cref{eq:delta-final} $\mathbf{\Delta}$ on RULER 131K subsets.}
  \label{fig:infbench-ablation}
\end{wrapfigure}

\begin{figure}[t]
    \vspace{-0.25in}
    \centering
    \includegraphics[width=0.5\linewidth]{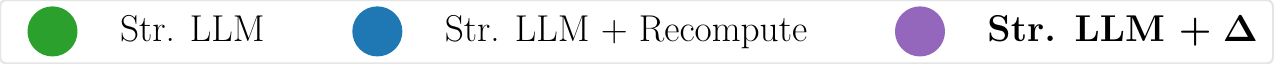} \\
    \includegraphics[width=0.32\linewidth]{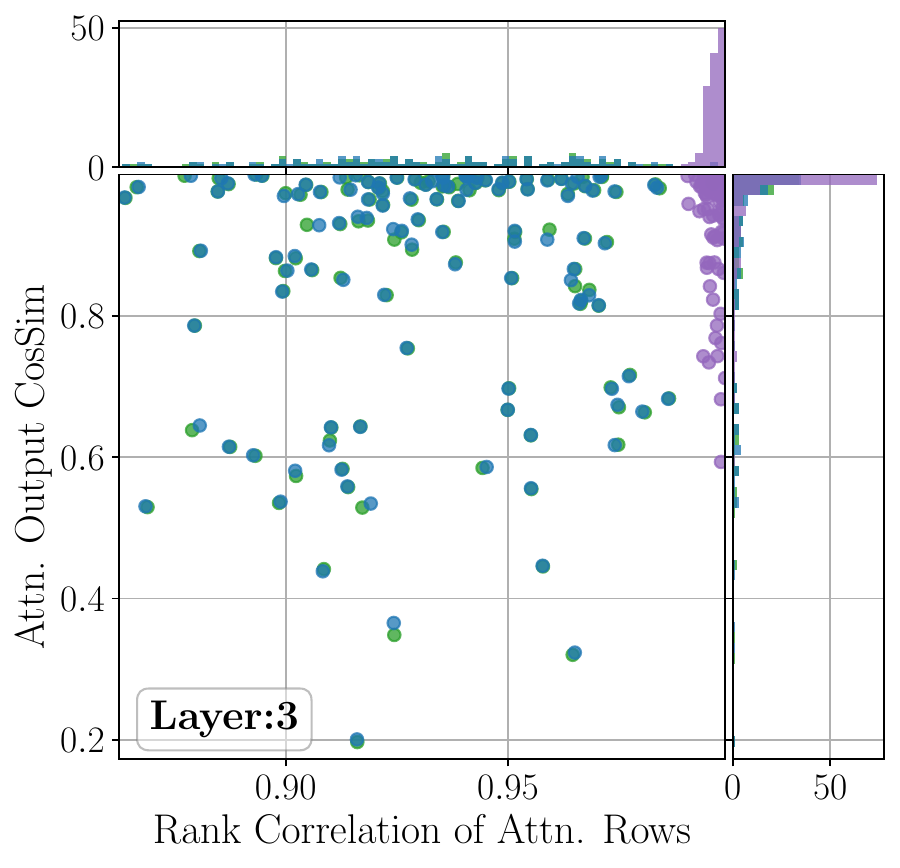}
    \includegraphics[width=0.32\linewidth]{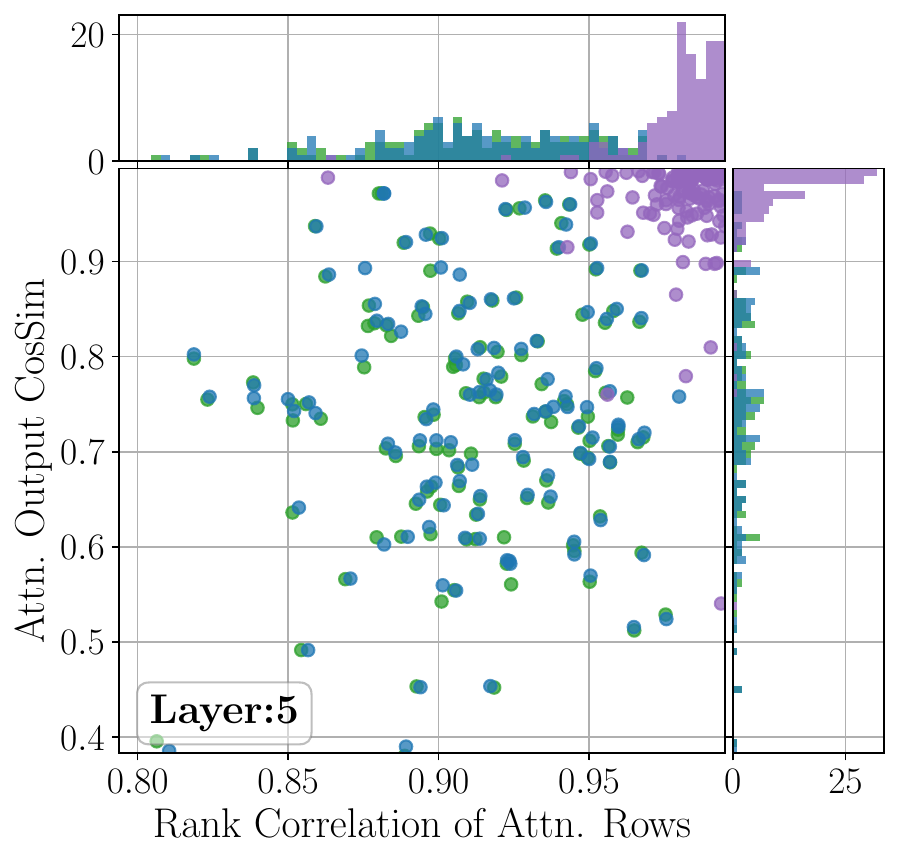}
    \includegraphics[width=0.32\linewidth]{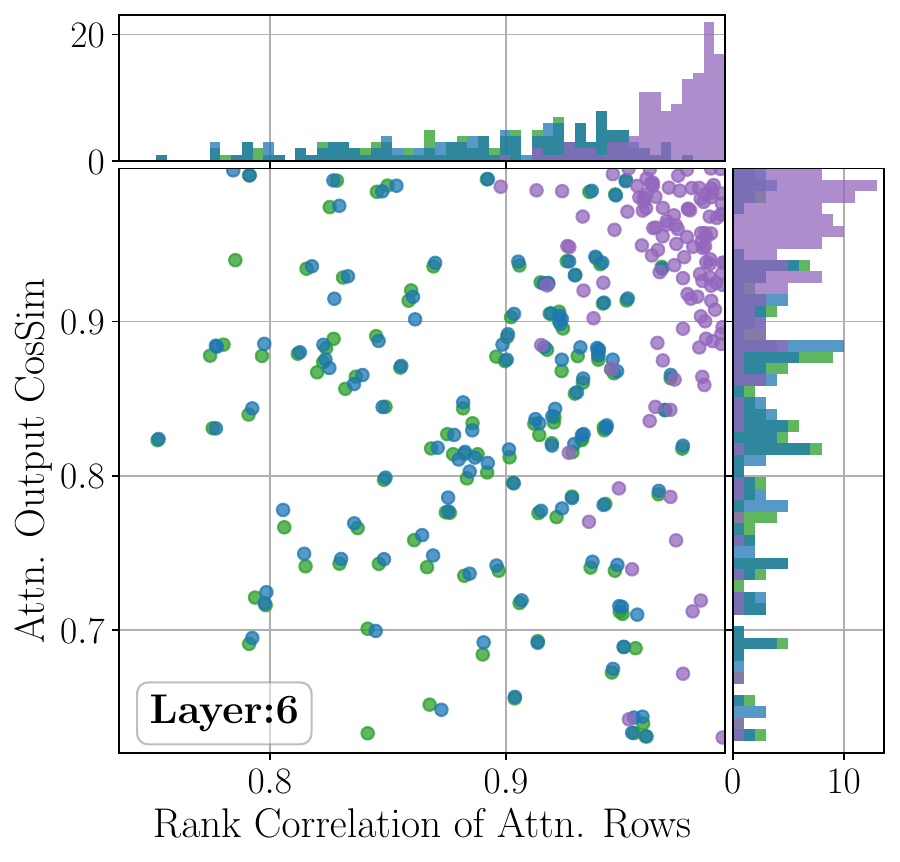}
    \caption{For RULER with a context length of 131K, we look at the final 128 tokens in the attention output and the final 128 queries in the attention matrix. We compare the cosine similarity of the outputs and the rank correlation of the attention rows to quadratic attention. We find that for both measures, $\mathbf{\Delta}$ Attention is more similar to quadratic attention.}
    \label{fig:spearman-cos}
    \vspace{-1.0em}
\end{figure}

\section{Discussion \& Limitations}
\label{sec:discussion}

Our method presented thus far has been a simple extension to existing sparse attention methods, which can be applied with a minimal addition of overhead and a very simple modification to the attention layer. The common way of computing sparse attention in prior work is to compute an attention output that is dense in the queries and sparse in the keys, so that there is at least one output for every input query token. One way to view our $\mathbf{\Delta}$ Attention extension is that we are mixing a key-sparse (and query-dense) attention output with a query-sparse (and key-dense) attention output in order to arrive at a representation which is closer to the quadratic attention output that is dense in both the queries and keys. 

The idea of viewing attention sparsity from both dimensions holds the potential for future works to explore novel ways of combining various combinations of sparse methods in order to approximate the full attention operation. With~\cref{lem:delta}, we were able to show that the difference of attention outputs approximates the missing attention output, however, we only have empirical evidence of the secondary approximation that $(\mathbf{A^\mathbf{\Delta}V})_i \approx (\mathbf{A^\mathbf{\Delta}V})_{i+\nu}$ for $\nu \in \{1, \dots, \gamma\}$. While this is empirically validated in our experiments and by the high cosine similarity in~\cref{fig:delta-diff-cos}, future works may study this approximation further, which could lead to creating a smarter selection criteria for the query sparse attention, as our method uses only a fixed hyperparameter to set the size of the gap between query tokens.

\section{Conclusion}
\label{sec:conclusion}
In this work, we first diagnose a harmful distributional shift induced by sparse attention prefill methods. We then propose a remedy with our lightweight, sparse-kernel agnostic $\mathbf{\Delta}$ Attention procedure. $\mathbf{\Delta}$ Attention corrects sparse outputs to align better with full quadratic attention outputs, requiring only a small post‑processing step that can be integrated seamlessly into existing inference pipelines. Across all benchmarks, and especially at the longest context lengths, our method delivers significant accuracy gains while maintaining high sparsity and low latency.

\newpage

\section{Acknowledgments}
\label{sec:ack}

This work was supported by:

\begin{itemize}
    \item Institute for Information \& communications Technology Planning \& Evaluation(IITP) grant funded by the Korea government(MSIT) (RS-2019-II190075, Artificial Intelligence Graduate School Program(KAIST)) 
    \item National Research Foundation of Korea (NRF) grant funded by the Korea government (MSIT) (No. RS-2023-00256259) 
    \item Artificial intelligence industrial convergence cluster development project funded by the Ministry of Science and ICT(MSIT, Korea) \& Gwangju Metropolitan City 
    \item The Institute of Information \& Communications Technology Planning \& Evaluation (IITP) with a grant funded by the Ministry of Science and ICT (MSIT) of the Republic of Korea in connection with the Global AI Frontier Lab International Collaborative Research. (No. RS-2024-00469482 \& RS-2024-00509279) 
    \item DeepAuto R\&D Program (No. DA-RS-2025-01) 
    \item A gift grant from Google 
\end{itemize}


\bibliographystyle{abbrvnat}
\bibliography{neurips_2025}

@article{less,
  title={Get more with less: Synthesizing recurrence with kv cache compression for efficient llm inference},
  author={Dong, Harry and Yang, Xinyu and Zhang, Zhenyu and Wang, Zhangyang and Chi, Yuejie and Chen, Beidi},
  journal={arXiv preprint arXiv:2402.09398},
  year={2024}
}

@article{cacheblend,
  title={Cacheblend: Fast large language model serving with cached knowledge fusion},
  author={Yao, Jiayi and Li, Hanchen and Liu, Yuhan and Ray, Siddhant and Cheng, Yihua and Zhang, Qizheng and Du, Kuntai and Lu, Shan and Jiang, Junchen},
  journal={arXiv e-prints},
  pages={arXiv--2405},
  year={2024}
}

@article{rectified,
  title={Rectified Sparse Attention},
  author={Sun, Yutao and Ye, Tianzhu and Dong, Li and Xia, Yuqing and Chen, Jian and Gao, Yizhao and Cao, Shijie and Wang, Jianyong and Wei, Furu},
  journal={arXiv preprint arXiv:2506.04108},
  year={2025}
}

@article{ape,
  title={Ape: Faster and longer context-augmented generation via adaptive parallel encoding},
  author={Yang, Xinyu and Chen, Tianqi and Chen, Beidi},
  journal={arXiv preprint arXiv:2502.05431},
  year={2025}
}

@inproceedings{repoqa,
  title={Repo{QA}: Evaluating Long Context Code Understanding},
  author={Jiawei Liu and Jia Le Tian and Vijay Daita and Yuxiang Wei and Yifeng Ding and Yuhan Katherine Wang and Jun Yang and LINGMING ZHANG},
  booktitle={First Workshop on Long-Context Foundation Models @ ICML 2024},
  year={2024},
  url={https://openreview.net/forum?id=hK9YSrFuGf}
}

@article{sink,
  title={Efficient streaming language models with attention sinks},
  author={Xiao, Guangxuan and Tian, Yuandong and Chen, Beidi and Han, Song and Lewis, Mike},
  journal={arXiv preprint arXiv:2309.17453},
  year={2023}
}

@article{minference,
  title={Minference 1.0: Accelerating pre-filling for long-context llms via dynamic sparse attention},
  author={Jiang, Huiqiang and Li, Yucheng and Zhang, Chengruidong and Wu, Qianhui and Luo, Xufang and Ahn, Surin and Han, Zhenhua and Abdi, Amir H and Li, Dongsheng and Lin, Chin-Yew and others},
  journal={arXiv preprint arXiv:2407.02490},
  year={2024}
}

@article{long-ppl,
  title={What is Wrong with Perplexity for Long-context Language Modeling?},
  author={Fang, Lizhe and Wang, Yifei and Liu, Zhaoyang and Zhang, Chenheng and Jegelka, Stefanie and Gao, Jinyang and Ding, Bolin and Wang, Yisen},
  journal={arXiv preprint arXiv:2410.23771},
  year={2024}
}

@article{attention,
  title={Attention is all you need},
  author={Vaswani, Ashish and Shazeer, Noam and Parmar, Niki and Uszkoreit, Jakob and Jones, Llion and Gomez, Aidan N and Kaiser, {\L}ukasz and Polosukhin, Illia},
  journal={Advances in neural information processing systems},
  volume={30},
  year={2017}
}

@article{inf-bench,
  title={Infinity Bench: Extending Long Context Evaluation Beyond 100K Tokens},
  author={Zhang, Xinrong and Chen, Yingfa and Hu, Shengding and Xu, Zihang and Chen, Junhao and Hao, Moo Khai and Han, Xu and Thai, Zhen Leng and Wang, Shuo and Liu, Zhiyuan and others},
  journal={arXiv preprint arXiv:2402.13718},
  year={2024}
}

@article{pg19-long-qa,
  title={Scaling Instruction-Tuned LLMs to Million-Token Contexts via Hierarchical Synthetic Data Generation},
  author={He, Linda and Wang, Jue and Weber, Maurice and Zhu, Shang and Athiwaratkun, Ben and Zhang, Ce},
  journal={arXiv preprint arXiv:2504.12637},
  year={2025}
}

@article{flashattention2,
  title={Flashattention-2: Faster attention with better parallelism and work partitioning},
  author={Dao, Tri},
  journal={arXiv preprint arXiv:2307.08691},
  year={2023}
}

@article{mistral,
  title={Mistral 7B},
  author={Jiang, Albert Q and Sablayrolles, Alexandre and Mensch, Arthur and Bamford, Chris and Chaplot, Devendra Singh and Casas, Diego de las and Bressand, Florian and Lengyel, Gianna and Lample, Guillaume and Saulnier, Lucile and others},
  journal={arXiv preprint arXiv:2310.06825},
  year={2023}
}

@article{rope,
  title={Roformer: Enhanced transformer with rotary position embedding},
  author={Su, Jianlin and Ahmed, Murtadha and Lu, Yu and Pan, Shengfeng and Bo, Wen and Liu, Yunfeng},
  journal={Neurocomputing},
  volume={568},
  pages={127063},
  year={2024},
  publisher={Elsevier}
}

@article{star-attn,
  title={Star attention: Efficient llm inference over long sequences},
  author={Acharya, Shantanu and Jia, Fei and Ginsburg, Boris},
  journal={arXiv preprint arXiv:2411.17116},
  year={2024}
}

@article{retrieval-head,
  title={Retrieval head mechanistically explains long-context factuality},
  author={Wu, Wenhao and Wang, Yizhong and Xiao, Guangxuan and Peng, Hao and Fu, Yao},
  journal={arXiv preprint arXiv:2404.15574},
  year={2024}
}

@article{induction-fv,
  title={Which Attention Heads Matter for In-Context Learning?},
  author={Yin, Kayo and Steinhardt, Jacob},
  journal={arXiv preprint arXiv:2502.14010},
  year={2025}
}

@article{stacked-induction,
  title={Mechanism and Emergence of Stacked Attention Heads in Multi-Layer Transformers},
  author={Musat, Tiberiu},
  journal={arXiv preprint arXiv:2411.12118},
  year={2024}
}

@article{induction-head,
  title={In-context learning and induction heads},
  author={Olsson, Catherine and Elhage, Nelson and Nanda, Neel and Joseph, Nicholas and DasSarma, Nova and Henighan, Tom and Mann, Ben and Askell, Amanda and Bai, Yuntao and Chen, Anna and others},
  journal={arXiv preprint arXiv:2209.11895},
  year={2022}
}

@article{hip,
  title={A Training-free Sub-quadratic Cost Transformer Model Serving Framework With Hierarchically Pruned Attention},
  author={Lee, Heejun and Park, Geon and Lee, Youngwan and Suh, Jaduk and Kim, Jina and Jeong, Wonyoung and Kim, Bumsik and Lee, Hyemin and Jeon, Myeongjae and Hwang, Sung Ju},
  journal={arXiv preprint arXiv:2406.09827},
  year={2024}
}

@article{infhip,
  title={InfiniteHiP: Extending Language Model Context Up to 3 Million Tokens on a Single GPU},
  author={Heejun Lee and Geon Park and Jaduk Suh and Sung Ju Hwang},
  journal={arXiv preprint arXiv:2502.08910},
  year={2024}
}

@article{bigbird,
  title={Big bird: Transformers for longer sequences},
  author={Zaheer, Manzil and Guruganesh, Guru and Dubey, Kumar Avinava and Ainslie, Joshua and Alberti, Chris and Ontanon, Santiago and Pham, Philip and Ravula, Anirudh and Wang, Qifan and Yang, Li and others},
  journal={Advances in neural information processing systems},
  volume={33},
  pages={17283--17297},
  year={2020}
}

@article{ruler,
  title={RULER: What's the Real Context Size of Your Long-Context Language Models?},
  author={Cheng-Ping Hsieh and Simeng Sun and Samuel Kriman and Shantanu Acharya and Dima Rekesh and Fei Jia and Yang Zhang and Boris Ginsburg},
  year={2024},
  journal={arXiv preprint arXiv:2404.06654},
}

@article{pg19,
author = {Rae, Jack W and Potapenko, Anna and Jayakumar, Siddhant M and
          Hillier, Chloe and Lillicrap, Timothy P},
title = {Compressive Transformers for Long-Range Sequence Modelling},
journal = {arXiv preprint},
url = {https://arxiv.org/abs/1911.05507},
year = {2019},
}

@article{spearmanr,
  author = {Spearman, C.},
  journal = {American Journal of Psychology},
  pages = {88--103},
  title = {The Proof and Measurement of Association Between Two Things},
  volume = 15,
  year = 1904
}

@article{h2o,
  title={H2o: Heavy-hitter oracle for efficient generative inference of large language models},
  author={Zhang, Zhenyu and Sheng, Ying and Zhou, Tianyi and Chen, Tianlong and Zheng, Lianmin and Cai, Ruisi and Song, Zhao and Tian, Yuandong and R{\'e}, Christopher and Barrett, Clark and others},
  journal={Advances in Neural Information Processing Systems},
  volume={36},
  year={2024}
}

@article{snapkv,
  title={Snapkv: Llm knows what you are looking for before generation},
  author={Li, Yuhong and Huang, Yingbing and Yang, Bowen and Venkitesh, Bharat and Locatelli, Acyr and Ye, Hanchen and Cai, Tianle and Lewis, Patrick and Chen, Deming},
  journal={arXiv preprint arXiv:2404.14469},
  year={2024}
}

@article{llama,
  title={The llama 3 herd of models},
  author={Dubey, Abhimanyu and Jauhri, Abhinav and Pandey, Abhinav and Kadian, Abhishek and Al-Dahle, Ahmad and Letman, Aiesha and Mathur, Akhil and Schelten, Alan and Yang, Amy and Fan, Angela and others},
  journal={arXiv preprint arXiv:2407.21783},
  year={2024}
}

\newpage
\appendix

\section{Appendix Contents}

\begin{itemize}
\item \cref{sec:broader-impact} Discusses the broader impact of our work.
\item \cref{fig:spearman-cos-app,fig:spearman-cos-app-2,fig:spearman-cos-app-3} shows individual plots comparing cosine similarities and rank correlation coefficients against quadratic attention for all layers, analogous to \cref{fig:spearman-cos}.
\item \cref{fig:attention-latency-stride-ablation-app} shows an additional study on the $\gamma$ parameter and latency for HiP, analogous to~\cref{fig:gamma-ablation}.
\item \cref{sec:implementation-details} discusses details about the implementation of our method.
\item \cref{sec:minference-latency-discussion} discusses details regarding latency for MInference.
\item \cref{fig:ruler-bar-app} shows bar charts for the full set of datasets for the RULER 131K context length.
\item \cref{sec:compute-resources} states the computing resources that were used for the experiments in this work. 
\item \cref{sec:extended-related-work} discusses additional related works and comparisons.
\item \cref{sec:repoqa} shows the performance of our method on code understanding.
\item \cref{sec:interpolation} discusses and shows results for interpolation/inputation between delta terms.
\item \cref{sec:paired-permutation} contains a paired permutation statistical significance test corresponding to the RULER results in \cref{tab:ruler}.
\end{itemize}

\section{Broader Impact}
\label{sec:broader-impact}

We are not aware of any negative potential impacts of our work beyond impacts that are general to all machine learning models. However, lowering the computational cost for inference has the potential to lower costs such as electricity consumption, hardware requirements, and latency for end users. If this can effectively be done with minimal degradation in the performance of the underlying model, it will likely be beneficial to both producers and consumers of AI models.

\section{Implementation Details}
\label{sec:implementation-details}

In addition to the index selection in~\cref{eq:q-index-selection}, in practice, we also select a block of queries for dense recomputation at the end of the prefill sequence, which makes the part of the prefill which requires a delta correction evenly divisible by $\gamma$. We do this for both ease of implementation and also to provide the decoding tokens with the most accurate block of recent context. The block of queries at the end of the sequence allows us to simply reshape a tensor and project the $\mathbf{\Delta}$ correction onto every element in the block, as the tensor that needs a delta correction will have a regular size that is divisible by $\gamma$.  

\section{Compute Resources}
\label{sec:compute-resources}

For LLM inference on benchmark datasets, we use Google Cloud Platform's 8x NVIDIA H100 node. For latency measurements, we use a standalone machine with an NVIDIA RTX 4090 in order to have a controlled environment. Here, we show the detailed specification of the latency benchmarking machine: 

\hspace{0.7in}
\begin{tabular}{p{1in} p{3in}}
\toprule
CPU & AMD Ryzen 7950X, 16 Core, 32 Thread\\
RAM & 128GB, DDR5 5600 Mhz\\
GPU & Nvidia RTX 4090, VRAM 24GB\\
PCIe & Gen 4.0 x8\\
OS & Ubuntu 22.04.4 LTS \\
GPU Driver & 535.171.04\\
\bottomrule
\end{tabular}

\section{Latency of MInference with Delta Attention}
\label{sec:minference-latency-discussion}

We did not report the latency of MInference in the main paper, because MInference shows unusually slower latency than other tested methods, including Flash Attention.
We think this is due to (1) insufficient optimization of the publicly available kernel \footnote{https://github.com/microsoft/MInference} and (2) MInference uses a for-loop across the head dimension that prevents the head dimension from being parallelized within the GPU. This limitation of the publicly available implementation will cause the latency to suffer if the attention calculation for each head does not fully utilize the hardware. This for-loop structure was likely implemented in this way because MInference uses different sparse attention strategies for each head. Therefore, as MInference is algorithmically faster than flash attention, we do not report the latency in~\cref{fig:latency}, as this would be misleading to readers who are not familiar with the low-level details of the implementations.


We capture the kernel latencies and hardware utilization for MInference. 
In our analysis with Nsight Systems, their vertical slash pattern kernel `\texttt{\_triton\_mixed\_sparse\_attn\_fwd\_kernel}', shows around 32 milliseconds latency for a single head, while flash attention shows only 462 milliseconds for 32 heads.
The MInference kernel shows noticeably low utilization of streaming multiprocessor warps, which is around 9\%.

However, for completeness, we put the latency measurements of MInference in~\cref{tab:minference-latency}.
In our measurement using their official codebase without meaningful modification, with pre-compiled model configuration for head-wise sparse method settings, Minference is about 1.377 times slower than Flash Attention. 
We believe this is only due to the lack of a fully parallelized kernel and not the design of the method.

\begin{table}[h]
\centering
\caption{Prefill latency measurements (ms) that include MInference on RTX 4090 up to 256K context length.}
\label{tab:minference-latency}
\vspace{1.0em}
\begin{tabular}{lrrrrr}\toprule
&32K &64K &128K &256K \\\midrule
FA &34.27 &119.77 &462.39 &1858.60 \\
HiP &53.61 &118.53 &255.05 &562.24 \\
HiP + $\Delta$ &55.44 &123.49 &268.02 &602.74 \\
Minference &135.28 &395.92 &1083.66 &2559.47 \\
\bottomrule
\end{tabular}
\end{table}





\begin{figure}
    \centering
    \includegraphics[width=0.5\linewidth]{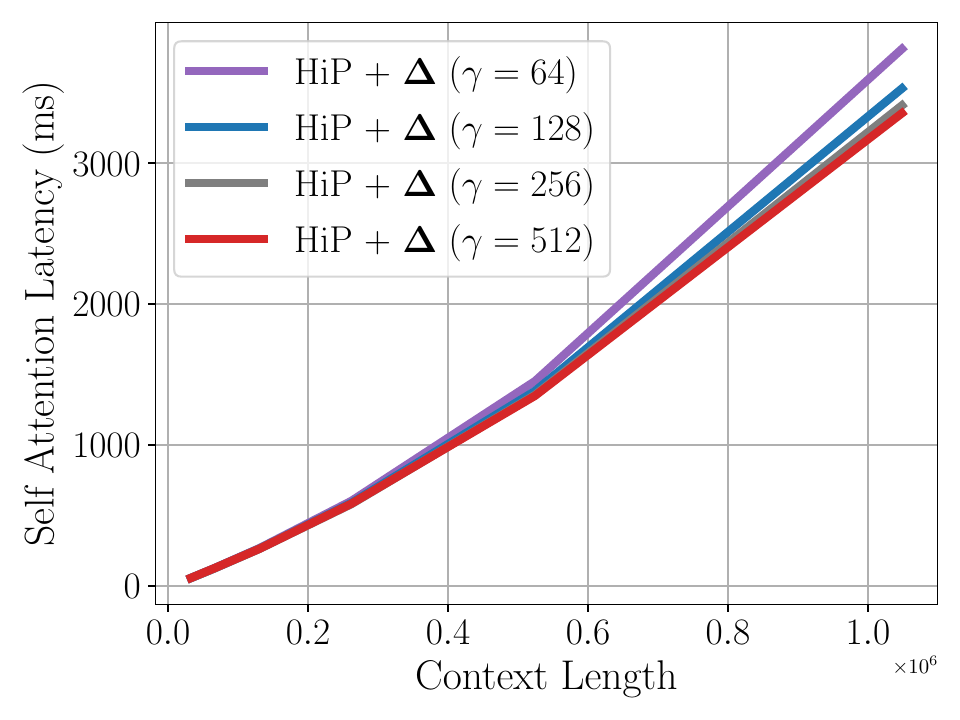}
    \caption{Latency measurements for different settings of $\gamma$ which controls the gap size between queries and also the overall sparsity of the calculation. This figure accompanies the latency ablation for Streaming LLM in the main text, \cref{fig:gamma-ablation}.}
    \label{fig:attention-latency-stride-ablation-app}
\end{figure}

\section{Approx Window Size Calculation}
\label{sec-approx-window-size}

When comparing out method to Streaming LLM, we would like to know how much computation overhead is increased in order to estimate the approximate window size of our method due to the fact that $\mathbf{\Delta}$ Attention computes extra tokens. We can calculate this as follows with $C$ as the context size, and $w$ as the window size in a single row of the attention matrix, our method will compute every $\gamma^{\text{th}}$ row of the attention matrix which would be equivalent to $\frac{C}{2\gamma}$ when amortized into each row calculation. This brings the total calculation per row to $w + \frac{C}{2\gamma}$. In the case of 131K context, a window size of $2048$, and $\gamma = 64$ (our standard setting) this would be evaluated as $2048 + \frac{2^{17}}{2(2^6)} = 2048 + 2^{10} = 2048 + 1024 = 3072$.

\section{Restatement and proof of~\texorpdfstring{\cref{lem:delta}}{}}
\label{sec:delta-lemma-restate}

We want to show that the difference of $\mathbf{A^\Delta V \approx AV - A}^* \mathbf{V}$ is approximately equal to the missing delta-shaped attention output, which is pictured in~\cref{fig:intuition}. w.l.o.g., we will consider a single arbitrary row of the attention matrix $\mathbf{a}$ and a single column vector from the values $\mathbf{v}$. The following is true regardless of the selected entries in $\mathbf{a}$, however, in order to create a tighter error bound, we assume the existence of a sparse attention method which chooses the largest attention values in $\mathbf{a}$ when calculating the sparse dot product $\mathbf{v^\top a^*}$. Specifically,

\begin{lemmarestate}[name=Lemma~\ref{lem:delta}, restated]
\label{lem:delta-restate}
Let \(\mathbf{\bar{a}}=(\bar{a}_1,\dots,\bar{a}_d)\in\mathbb{R}^d\) be the pre-softmax vector which is sorted and satisfies,
\[
\bar{a}_1\le \bar{a}_2\le\cdots\le \bar{a}_N,
\]
then any exact top-$k$ sparse attention method which selects the top-$k$ attention scores should select the last $k$ elements of $\mathbf{a}$. Fix an integer \(1 \le k \leq N\).  Define the head-sum $H$, tail-sum $T$, and normalization constant $Z$ to be the following:
\begin{align}
H &= \sum_{i=1}^{N-k}e^{\bar{a}_i}, \\
T &= \sum_{i=N-k+1}^{N}e^{\bar{a}_i}, \\
Z &= H+T.
\end{align}
Set
\[
\mathbf{a}_i \;=\;\frac{e^{\bar{a}_i}}{Z},
\qquad
\mathbf{a^*}_i \;=\;
\begin{cases}
0, & i\le N-k,\\[6pt]
\dfrac{e^{\bar{a}_i}}{T}, & i>N-k.
\end{cases}
\]
For any \(\mathbf{v}=(v_1,\dots,v_d)\in\mathbb{R}^d\) which is sorted according to the rank of elements in $\mathbf{a}$, define the tail‐max as,
\[
M_{\mathrm{tail}}
\;=\;\max_{\,i>N-k}\bigl|v_i\bigr|.
\]
write
\[
\mathbf{\Delta} \;=\;\mathbf{a}^\top \mathbf{v} \;-\;\mathbf{a^*}^\top \mathbf{v},
\]
we have the exact decomposition
\[
\mathbf{\Delta}
=\sum_{i=1}^{N-k} \mathbf{a}_i\,\mathbf{v}_i
\;+\;R,
\]
where the “remainder” term
\[
R
=\sum_{i=N-k+1}^{N}\bigl[\mathbf{a}_i- \mathbf{a^*}_i\bigr]\,\mathbf{v}_i
\]
is upper bounded by
\[
\bigl\lvert R\bigr\rvert
\;\le\;\frac{H}{H + T}\;M_{\mathrm{tail}}.
\]
Therefore,
\begin{align}
\left|\mathbf{\Delta} - \sum_{i=1}^{N-k}\mathbf{a}_i\,\mathbf{v}_i\right|
&=\lvert R\rvert \\
& \leq \frac{H}{H + T}\;M_{\mathrm{tail}}
\end{align}

\end{lemmarestate}

\begin{proof}
Split
\begin{align}
\mathbf{\Delta}
&=\sum_{i=1}^{N-k}\mathbf{a}_i\,\mathbf{v}_i + \sum_{i=N-k+1}^N\bigl[\mathbf{a}_i-\mathbf{a^*}_i\bigr]\,\mathbf{v}_i \\
&=\sum_{i=1}^{N-k}\mathbf{a}_i\,\mathbf{v}_i+R.
\end{align}
For \(i>N-k\), 
\[
\mathbf{a}_i
=\frac{e^{\bar{a}_i}}{H+T}
=\frac{e^{\bar{a}_i}}{T}\,\frac{T}{H+T}
=\mathbf{a^*}_i\,\frac{T}{H+T},
\]
so
\begin{align}
\mathbf{a}_i- \mathbf{a^*}_i
&=  \mathbf{a^*}_i\,\frac{T}{H+T} - \mathbf{a}^*_i \\
&=  \mathbf{a^*}_i\,\left(\frac{T}{H+T} - 1\right) \\
&= -\,\mathbf{a^*}_i\,\frac{H}{H+T}.
\end{align}
Thus
\[
R
=-\frac{H}{H+T}\sum_{i=N-k+1}^N \mathbf{a^*}_i\,\mathbf{v}_i,
\]
and since \(\sum_{i=N-k+1}^N \mathbf{a^*}_i=1\) and \(\lvert \mathbf{v}_i\rvert\le M_{\mathrm{tail}}\) on the tail,
\begin{align}
\lvert R\rvert
 & =\frac{H}{H+T}\left|\sum_{i=N-k+1}^N\mathbf{a^*}_i\, \mathbf{v}_i\right| \\
 & \le\frac{H}{H+T}\sum_{i=N-k+1}^N\mathbf{a^*}_i\,\lvert \mathbf{v}_i\rvert \\
 &\le\frac{H}{H+T}\,M_{\mathrm{tail}}.
\end{align}
completing the proof.
\end{proof}

If we assume that $T \gg H$ as is the expected outcome with sparse attention, then the bound becomes tighter, as the denominator $H + T \gg H$. This implies that better sparse top-$k$ approximations will result in a lower error bound. We empirically verified this difference in \cref{fig:r-bound}, which analyzes both the error bound and the empirical error on a real input from the RULER-131K subset. \cref{fig:oracle-bound} measures the bound and empirical error of an oracle top-$k$ attention while \cref{fig:sllm-bound} measures the same bound and empirical error for Streaming LLM, which chooses a sliding window and attention sink. We find that the bound is generally tighter for the oracle top-$k$ attention, but in both cases, the overall empirical approximation error remains low.

\begin{figure}[H]
    \centering
    \includegraphics[width=\linewidth]{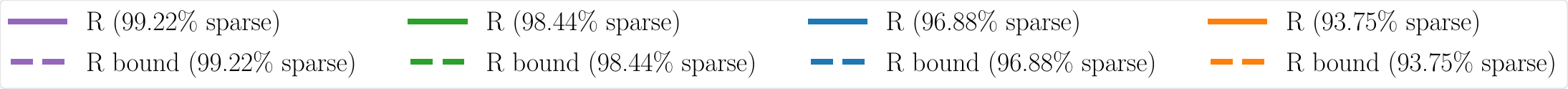}
    \begin{subfigure}[t]{0.5\textwidth}
        \includegraphics[width=\linewidth]{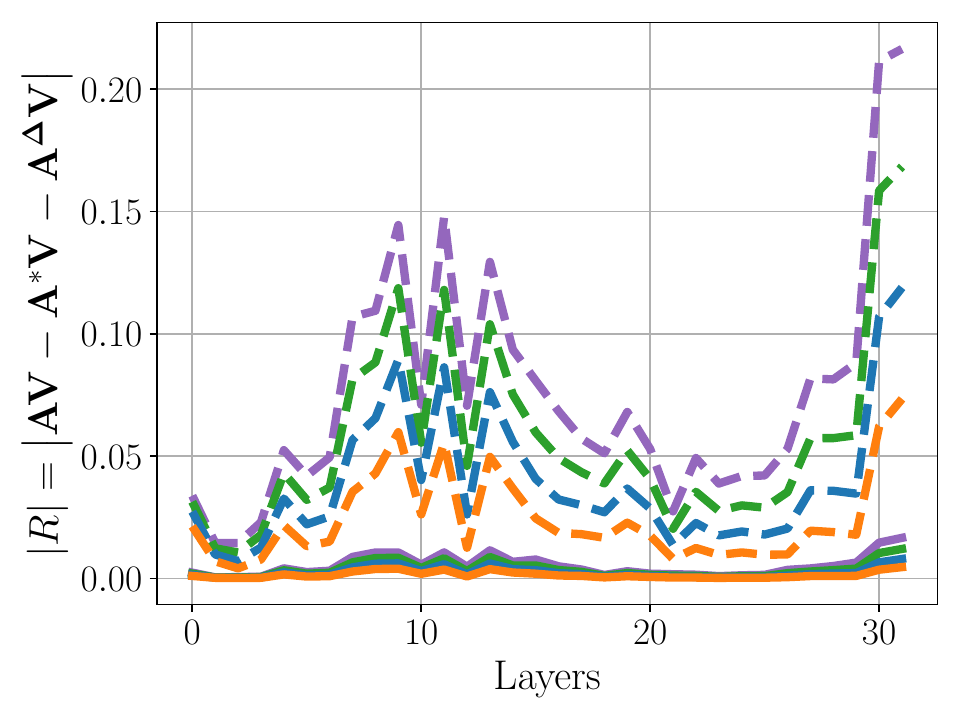} 
        \caption{Oracle top-$k$ sparse attention.}
        \label{fig:oracle-bound}
    \end{subfigure}%
    \begin{subfigure}[t]{0.5\textwidth}
        \includegraphics[width=\linewidth]{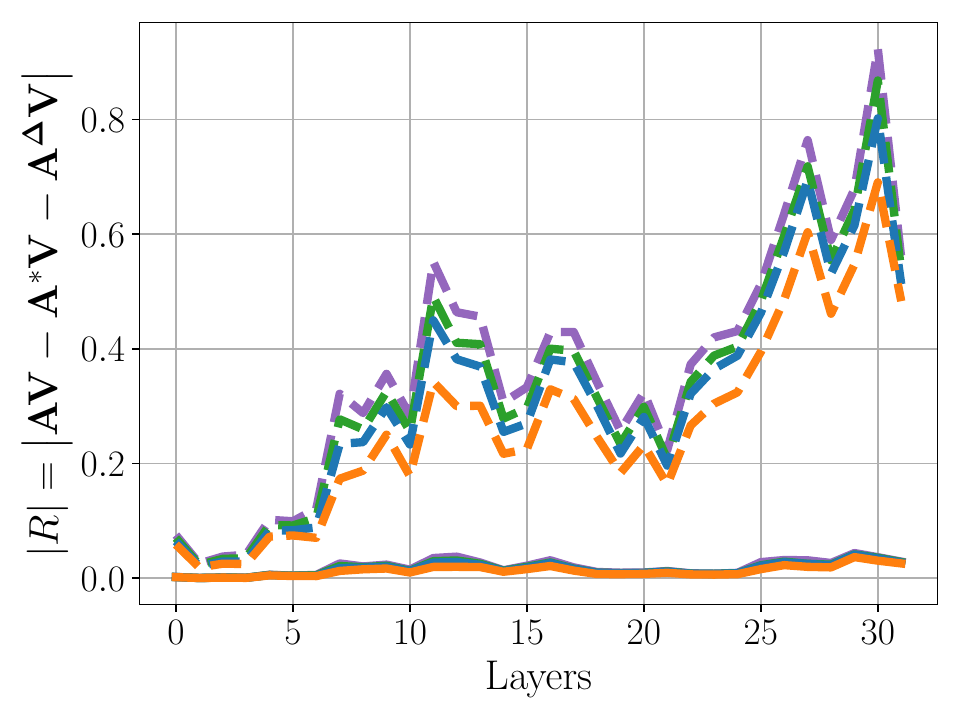}
        \caption{Streaming LLM.}
        \label{fig:sllm-bound}
    \end{subfigure}%
    \caption{Empirically analyzing the approximation and bound from~\cref{lem:delta}. A more precise sparse top-$k$ attention method, such as an oracle~\textbf{(\subref{fig:oracle-bound})} maintains a tighter bound on the approximation error. Streaming LLM~\textbf{(\subref{fig:sllm-bound})} results in a looser bound, however the empirical approximation error (solid lines) remains low in both methods.}
    \label{fig:r-bound}
\end{figure}

\begin{figure}[H]
    \centering
    \includegraphics[width=\linewidth]{figures/recompute-plots/bar-legend.pdf}
    \includegraphics[width=\linewidth]{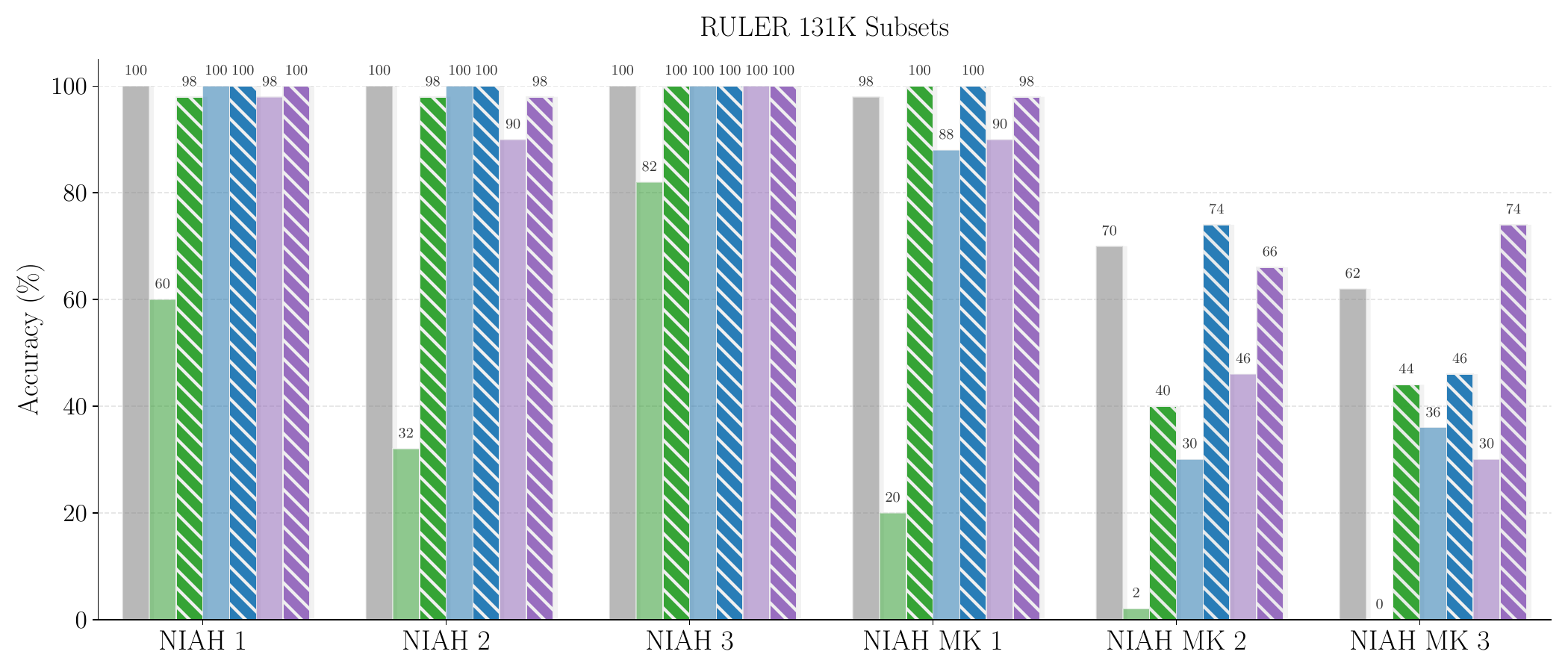}
    \includegraphics[width=\linewidth]{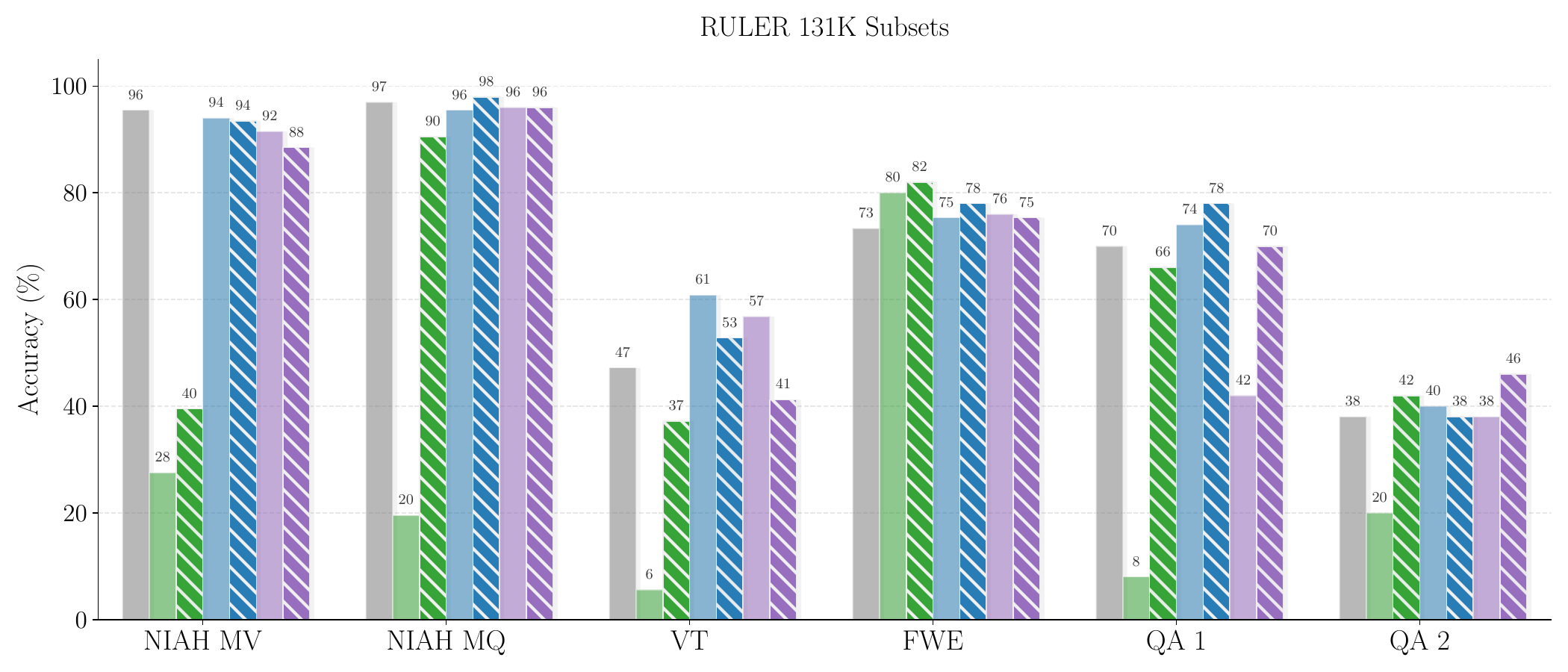}
    \caption{All RULER 131K subsets. This is a companion to~\cref{fig:ruler-subset-motivation}. The CWE subset is excluded, as all models, including quadratic attention, scored 0\%.}
    \label{fig:ruler-bar-app}
\end{figure}

\section{Extended Related Work}
\label{sec:extended-related-work}

In addition to the related work cited in \cref{sec:background}, there are a number of additional works which deal with related topics that we wish to highlight. 

LESS \citep{less} requires training a low rank cache compressor. LESS mentions differences in attention distributions between dense and sparse attention, however, the authors make no mention of the critical insight of our work, namely that a dense decode fails to properly align with the tokens resulting from a sparse prefill due to the distributional shift of the keys that is induced by the sparse prefill. 

Cacheblend \citep{cacheblend} proposes using one dense attention layer to identify important tokens, and then selectively recomputing these tokens in later layers to add missing parts of the sparse attention to cached KV pairs. Cacheblend proposed this as a way to augment and consolidate independently processed chunks of a RAG pipeline. In practice, however, this would effectively be similar to a "smart" sparse prefill method like HiP or MInference which fills in some of the missing tokens in the attention matrix which are outside of the local window context. As out experiments show, this is not always sufficient to fix the distributional shift between sparse and dense prefills. 

APE \citep{ape} proposes temperature scaling and rescaling the attention post-hoc in order to correct any error introduced. However, APE misses the crucial insight of our work, namely that sparse and dense attention result in completely different token distributions which means that there is a problem of query-key matching during decoding. APE only considers query and key geometry as a function of a) position and b) input. They deduce that because they key states of the first few keys (sink tokens) are relatively stable, then the geometry of all other keys are also stable. 

Rectified Sparse Attention \citep{rectified} (a concurrent work) considers dense prefills and a sparse decoding procedure. Their sparse decoding procedure is similar to what we call “recompute” in \cref{tab:ruler-ablation} and \cref{fig:spearman-cos,fig:infbench-ablation,fig:spearman-cos-app,fig:spearman-cos-app-2,fig:spearman-cos-app-3} where we showed that this “recomptue” method is insufficient to mitigate the distributional shift in the outputs.

Comparisons to these extended related works, and to Star Attention~\citep{star-attn} can be seen in \cref{tab:related-work-ruler}.

\begin{table}[h]
\centering
\caption{RULER comaprison to related works on sparse attention and sparse RAG works.}
\label{tab:related-work-ruler}
\begin{tabular}{lccccccc}
    \toprule
    Method & 131K & 65K & 32K & 16K & 8K & 4K &	Avg. \\
    \midrule
    Str.LLM	& 27.45 & 18.59 & 30.25 & 38.13 & 60.53 & 90.52 & 44.25 \\
    Cachblend  & 0.00 &	0.21 & 0.31 & 1.49 & 24.42 & 96.27 & 20.45 \\
    APE & 26.76 & 43.03 & 53.13 & 67.50 & 77.25 & 93.76 & 60.24 \\
    Str.LLM + Delta & \textbf{64.40} & \textbf{75.22} & \textbf{81.27} & \textbf{88.66} & \textbf{92.25} & \textbf{96.54} & \textbf{83.06} \\
    \midrule
    \midrule
    Star Attention Mask & 12.00 & 14.86 & 20.43 & 31.66 & 51.60 & 78.62 & 34.86 \\
    Star Attention Mask + Delta & \textbf{58.84} & \textbf{70.12} & \textbf{74.77} & \textbf{82.69} & \textbf{89.12} & \textbf{93.28} & \textbf{78.13} \\
    \bottomrule
\end{tabular}
\end{table}

\section{RepoQA}
\label{sec:repoqa}

We evaluate $\Delta$ Attention on code understanding by using the RepoQA~\citep{repoqa} dataset that asks the model to retrieve a function from a long block of input text. In this dataset, the long input text contains the code from many functions and the query contains a plain language description of what the function does. The model is then supposed to return back the correct function as output. We compare Streaming LLM with and without our delta correction in~\cref{tab:repoqa}  

\begin{table}[h]
\centering
\caption{RepoQA results for Streaming LLM and Streaming LLM + Delta. Plain FA3 is included for reference.}
\label{tab:repoqa}
\resizebox{\linewidth}{!}{%
\begin{tabular}{lccccccccccccc}
    \toprule
    Threshold &	0.0 & 0.1 & 0.2 & 0.3 & 0.4 & 0.5 & 0.6 & 0.7 & 0.8 & 0.9 & 1.0 & Avg \\
    \midrule
    Vanilla (FA3) & 94.8 & 92.2 & 90.6 & 89.4 & 88.8 & 88.4 & 86.8 & 85.4 & 84.4 & 83.2 & 76 & 87.27 \\
    \midrule
    Str.LLM	& 73.6 & 64.0 & 60.6 & 58.8 & 57.6 & 56.8 & 55.0 & 53.0 & 50.4 & 44.2 & 35.6 & 55.42 \\
    Str.LLM + Delta & \textbf{85.8} & \textbf{78.0} & \textbf{73.8} & \textbf{72.0} & \textbf{70.6} & \textbf{67.0} & \textbf{64.8} & \textbf{61.2} & \textbf{57.2} & \textbf{50.4} & \textbf{42.6} & \textbf{65.76} \\
    
    \bottomrule
\end{tabular}}
\end{table}

\section{Interpolation}
\label{sec:interpolation}

The method presented in~\cref{sec:method} proposes to use a single delta correction at index $i$ to influence the next $i + \gamma - 1$ attention outputs. This causes a discrete jump in the delta correction at every $\gamma^{\text{th}}$ output. It may be the case that a better strategy would be to smooth out the transition or impute the delta corrections within the window by some imputation function. In~\cref{tab:interpolation}, we look at three different possible imputation functions and evaluate the overall effect on RULER. 

\textbf{Linear Interpolation.} For linear interpolation, we first compute all delta corrections, and then produce mixing coefficients $\beta \in [0, 1]$ which linearly increase from $[0, ..., 1]$. Interpolation is then performed between consecutive delta correction terms by the function $\hat{\Delta}_k = (1 - \beta_k)\Delta_i + \beta_k \Delta_{i + 1}$. Each $\Delta$ term will therefore expand into $\vert k \vert  = \gamma$ terms, such that the number of delta corrections now matches the sparse attention output size. These expanded, and smoothed delta corrections will be treated as the new correction term, providing a smoother transition between terms.

\textbf{EMA.} Instead of linear interpolation, which technically violates the causality of the attention mechanism by incorporating information from the future into the past, we may instead expand the delta correction term by repeating each vector $\gamma$ times, and then perform an exponential moving average (EMA) over the full set of vectors using a coefficient $\beta \in [0, 1]$ and computing the EMA as $\Delta_i = (1 - \beta)\Delta_{i-1} + \beta \Delta_i$. The EMA acts as a smoothing mechanism which smooths the transition between delta terms.

\textbf{$\alpha,\beta,\gamma$ Filter.} A third option is to use a Kalman style filter. We chose to use an $\alpha,\beta,\gamma$ filter where $\alpha$ is a position coefficient, $\beta$ is a velocity coefficient, and $\gamma$ is an acceleration coefficient. At each step, position, velocity, and acceleration are updated based on a mixture of the real position and the accumulated statistics for position, velocity, and acceleration. We consider every operation to be an elementwise scalar operation. The algorithm for the $\alpha$, $\beta$, $\gamma$ filter can be seen in~\cref{alg:abg}

Although there are slight improvements using these imputation methods in~\cref{tab:interpolation}, no method shows conclusive improvements over our original method. However, we think delta smoothing or imputation shows a promising direction for future research.

\begin{figure}
\rule{\linewidth}{1.5pt}
\vspace{-1.2em}
\captionof{algorithm}{$\alpha,\beta,\gamma$ Filter}\label{alg:abg}
\vspace{-0.2em}
\rule{\linewidth}{1pt}
\begin{algorithmic}
    \Require $\alpha,\beta,\gamma$ and $\Delta$ vectors

    \State $o \gets$ zero vector like $\Delta$
    \State $p \gets \Delta_0$
    \State $v \gets$ zero vector like $p$
    \State $a \gets$ zero vector like $p$
    \State $o \gets \Delta_0$ 

    \For {$i$ in $[1, ..., \text{len}(\Delta)]$}
        \State $y \gets \Delta_i$
        \State \text{\small \textcolor{gray}{// update approx position and velocity}}
        \State $\hat{p} \gets p + v + 0.5 a$
        \State $\hat{v} \gets v + a$
        \State \text{\small \textcolor{gray}{// calculate difference between real and predicted position}}
        \State $r \gets y - \hat{p}$
        \State \text{\small \textcolor{gray}{// update position, velocity, and acceleration.}}
        \State $p \gets \hat{p} + \alpha r$
        \State $v \gets \hat{v} + \beta r$
        \State $a \gets a + \gamma r$

        \State $o_i \gets p$
    \EndFor
    \State return $o$
    \end{algorithmic}
\rule{\linewidth}{1pt}
\end{figure}

\begin{table}[h]
\centering
\caption{Interpolation Experiments.}
\label{tab:interpolation}
\resizebox{\linewidth}{!}{%
\begin{tabular}{lcccccccc}
    \toprule
    Method & 131K & 65K & 32K & 16K & 8K & 4K &	Avg. \\
    \midrule
    Str.LLM + Delta + Linear Interpolation & \textbf{65.15} & \textbf{75.65} & 81.26 & 88.26 & 92.34 & \textbf{96.66} & \textbf{83.22} \\
    Str. LLM + Delta + EMA ($\beta =0.5$) &	63.21 & 75.22 & \textbf{81.27} & \textbf{88.66} & 92.25 & 96.54 & 82.85 \\
    Str. LLM + Delta + EMA ($\beta =0.75$) & 63.40 & 74.60 & 80.76 & 88.52 & 92.29 & 96.62 & 82.69 \\
    Str. LLM + Delta + EMA ($\beta =0.95$) & 63.16 & 75.87 & 81.03 & 88.27 & 92.26 & 96.59 & 82.86 \\
    Str. LLM + Delta + ($\alpha=0.05,\beta=1.25 \times 10^{-4},\gamma=2.08 \times 10^{-5}$) Filter & 58.35 & 73.48 & 80.54 & 88.15 & 92.03 & 96.48 & 81.50 \\
    Str. LLM + Delta + ($\alpha=0.1,\beta=5 \times 10^{-3},\gamma=1.66\times 10^{-4}$) Filter & 57.99 & 72.70 & 79.64 & 88.58 & \textbf{92.42} & 96.57 & 81.31 \\
    Str. LLM + Delta + ($\alpha=0.2,\beta=5 \times 10^{-2},\gamma=3.5 \times 10^{-3}$) Filter & 61.47 & 74.31 & 80.29 & 88.47 & 92.32 & 96.58 & 82.24 \\
    \midrule
    Str.LLM + Delta	& 64.40 & 75.22 & \textbf{81.27} & \textbf{88.66} & 92.25 & 96.54 & 83.06 \\
    \bottomrule
\end{tabular}
}
\end{table}

\section{Statistical Significance Tests}
\label{sec:paired-permutation}

We assess the statistical significance of the results presented in~\cref{tab:ruler}. For this, we use a one-sided paired permutation test to test the significance of the difference between the versions of Streaming LLM, HiP and MInference with and without our delta correct applied. The results are shown in~\cref{tab:significance}. We split RULER tasks according to QA vs. non-QA retrieval tasks. The statistical significance shows a high correlation with the displayed colors in~\cref{tab:ruler} and verifies that our results are statistically significant. 

\begin{table}[h]
\centering
\caption{Interpolation Experiments. Each entry is a p-value assessing whether or not our delta correction results in a significant improvement (significance level is $p < 0.05$).}
\label{tab:significance}
\resizebox{\linewidth}{!}{%
\begin{tabular}{lccccccc}
    \toprule
    Method & 131K & 65K & 32K & 16K & 8K & 4K \\
    \midrule
    Str.LLM (all non qa tasks)& \textbf{0.0001} & \textbf{0.0001} & \textbf{0.0001} & \textbf{0.0001} & \textbf{0.0001} & \textbf{0.0001} \\
    Str.LLM (all qa tasks) & \textbf{0.0001} & \textbf{0.0001} & \textbf{0.0001} & \textbf{0.0001} & \textbf{0.0001} & 0.4958 \\
    HiP (all non qa tasks) & \textbf{0.0001} & \textbf{0.0018} & 0.7112 & 0.5018 & 0.8331 & 0.5747 \\
    HiP (all qa tasks) & 0.4918 & 0.7252 & 0.9245 & 0.8076 & 0.5009 & 1 \\
    MInference (all non qa tasks) & \textbf{0.0001} & 0.858 & 0.6485 & 0.5116 & 0.8774 & 1 \\
    MInference (all qa tasks) & \textbf{0.0004} &  0.8813 & 0.9848 & 0.8777 & 0.499 & 1 \\
    \bottomrule
\end{tabular}
}
\end{table}

\begin{figure}
    \centering
    \includegraphics[width=0.5\linewidth]{figures/recompute-plots/spearman-cos-legend.pdf} \\
    \includegraphics[width=0.32\linewidth]{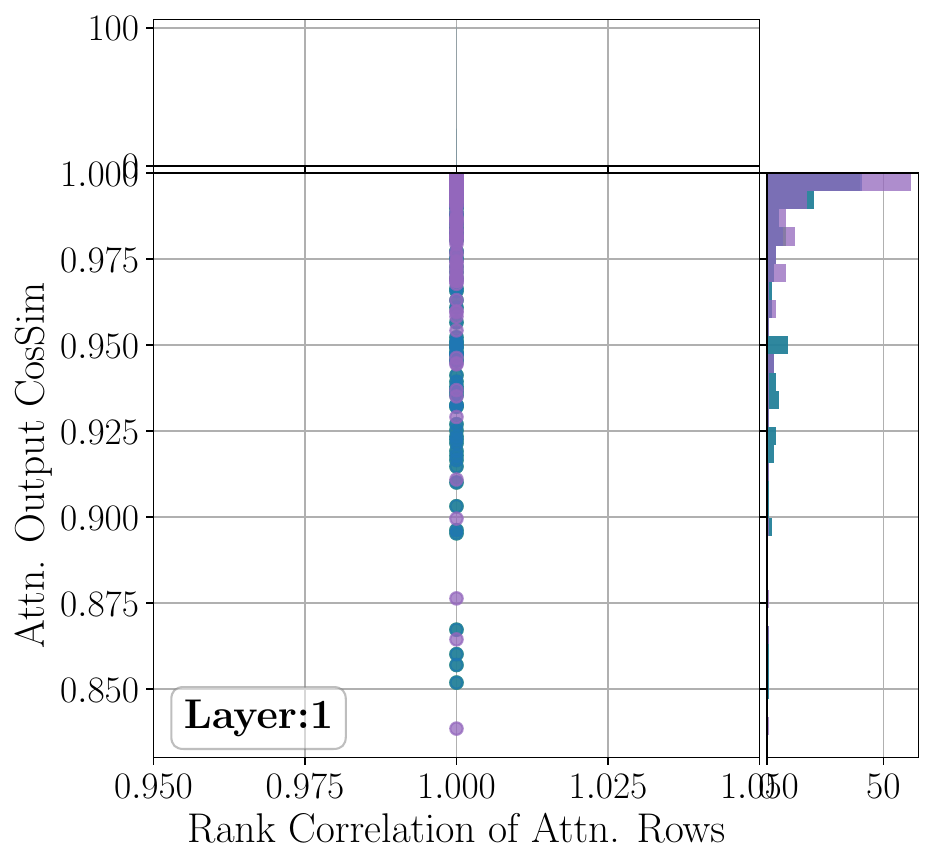}
    \includegraphics[width=0.32\linewidth]{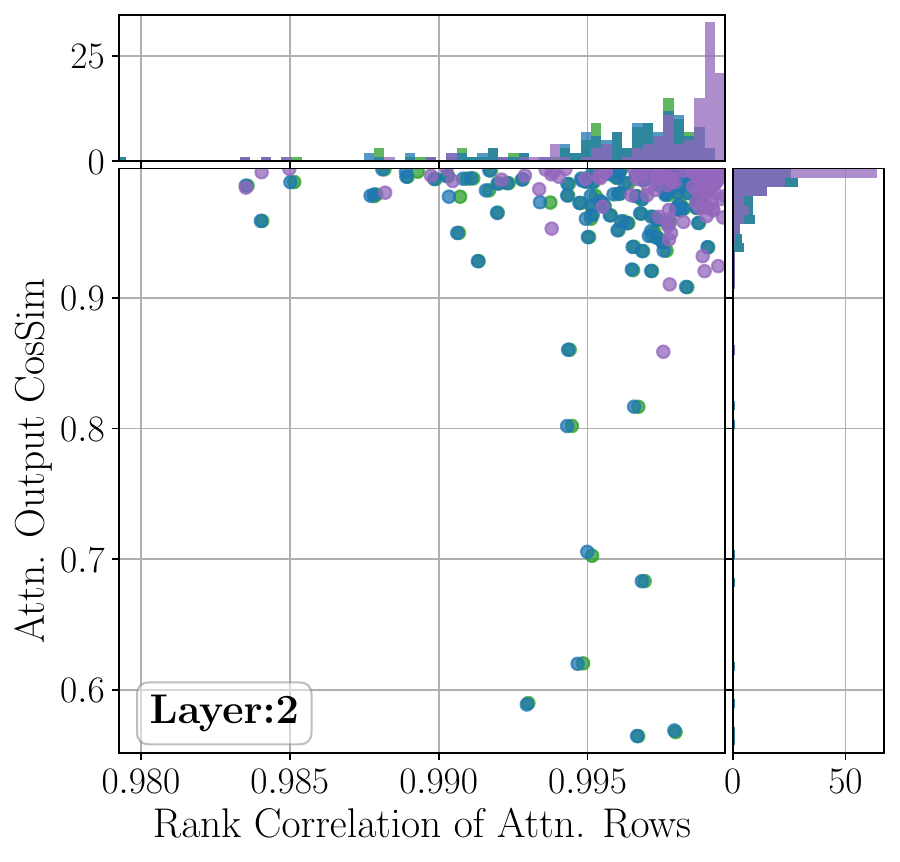}
    \includegraphics[width=0.32\linewidth]{figures/recompute-plots/spearman-vs-output-layer-3.pdf}
    \includegraphics[width=0.5\linewidth]{figures/recompute-plots/spearman-cos-legend.pdf} \\
    \includegraphics[width=0.32\linewidth]{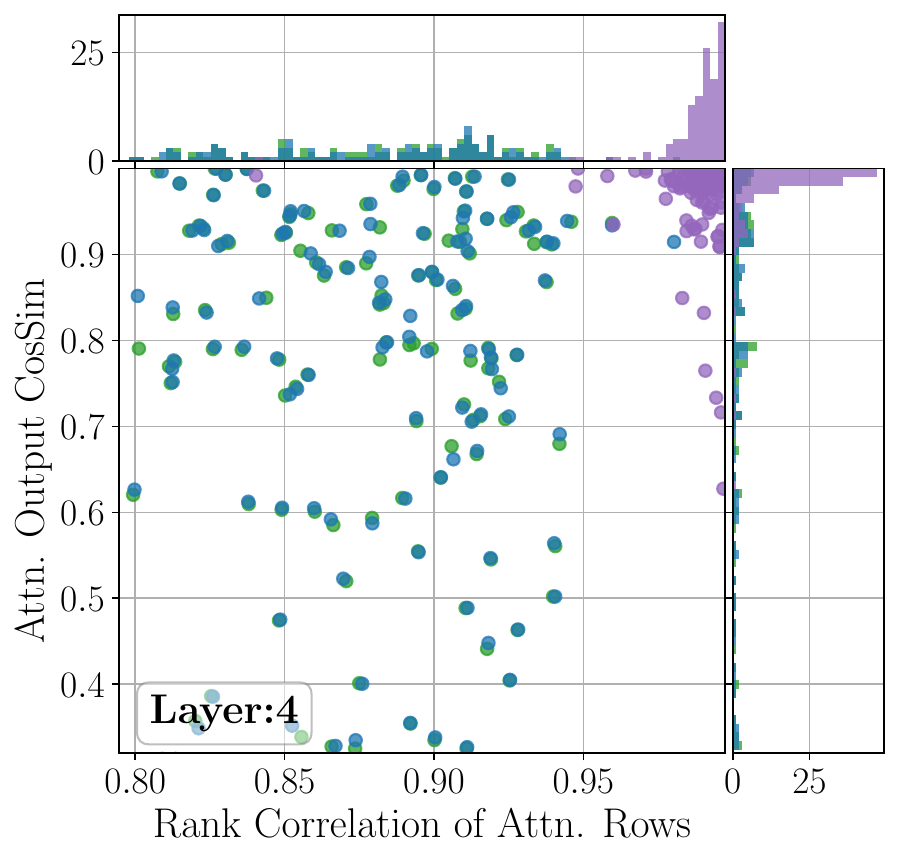}
    \includegraphics[width=0.32\linewidth]{figures/recompute-plots/spearman-vs-output-layer-5.pdf}
    \includegraphics[width=0.32\linewidth]{figures/recompute-plots/spearman-vs-output-layer-6.pdf}
    \includegraphics[width=0.5\linewidth]{figures/recompute-plots/spearman-cos-legend.pdf} \\
    \includegraphics[width=0.32\linewidth]{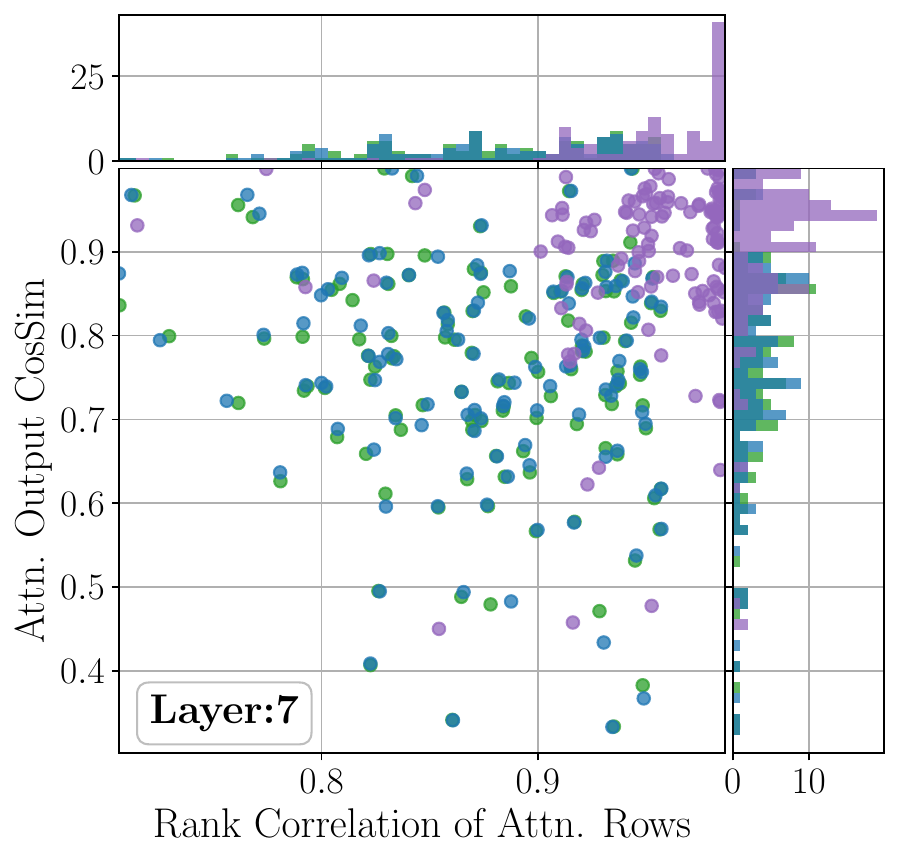}
    \includegraphics[width=0.32\linewidth]{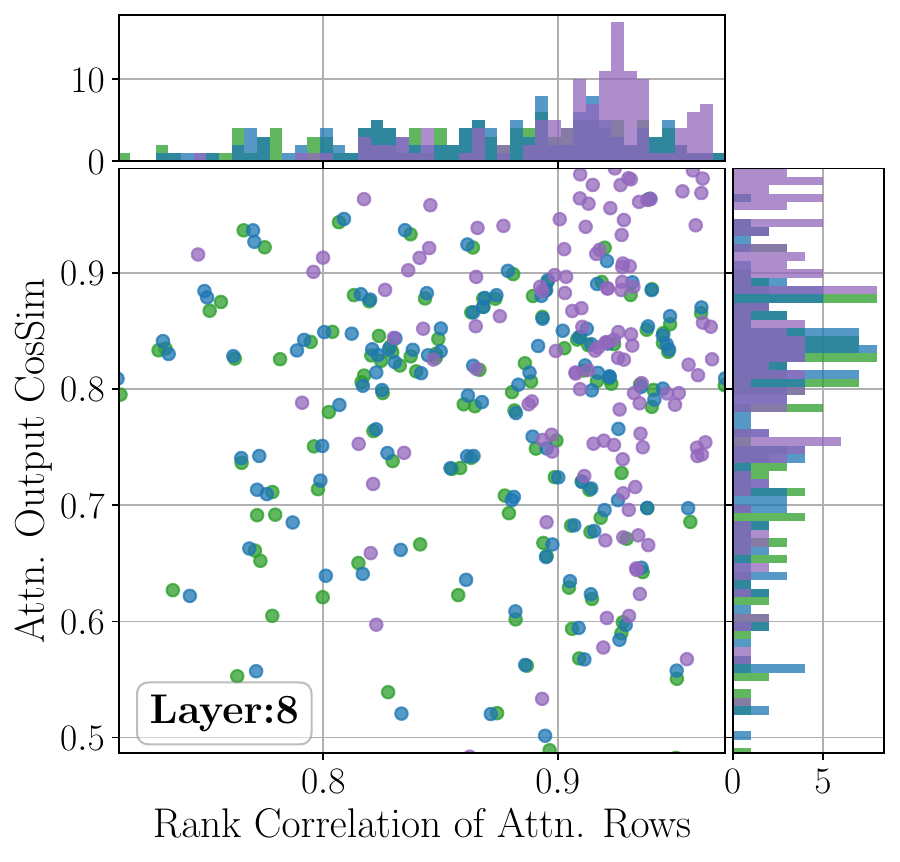}
    \includegraphics[width=0.32\linewidth]{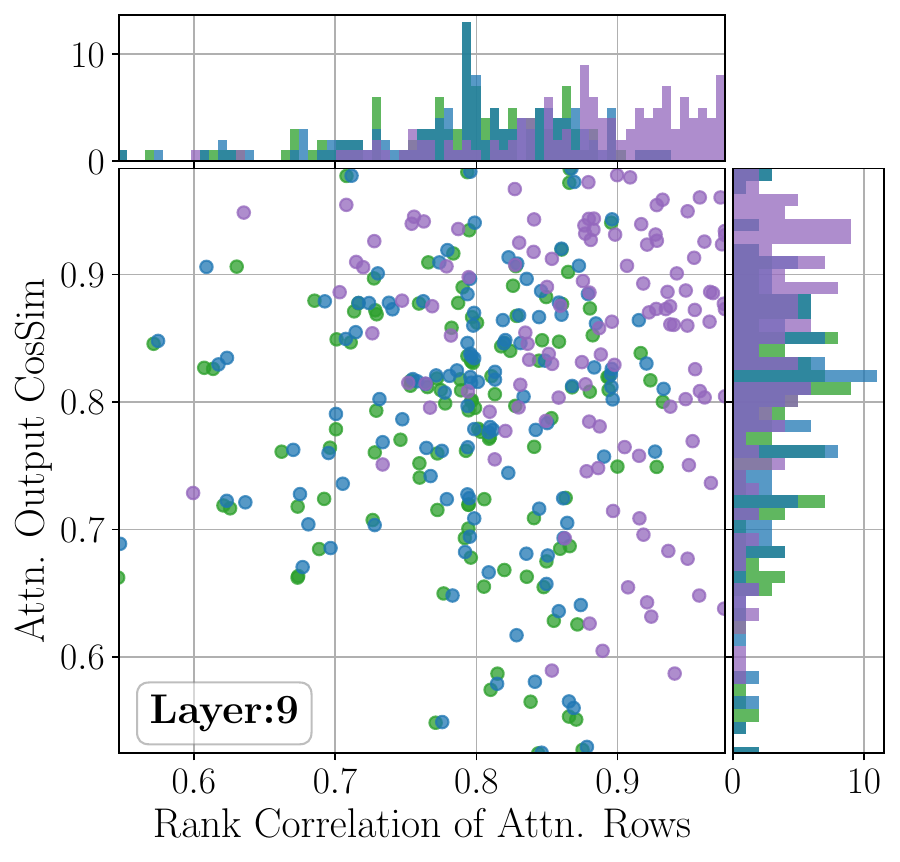}
    \includegraphics[width=0.5\linewidth]{figures/recompute-plots/spearman-cos-legend.pdf} \\
    \includegraphics[width=0.32\linewidth]{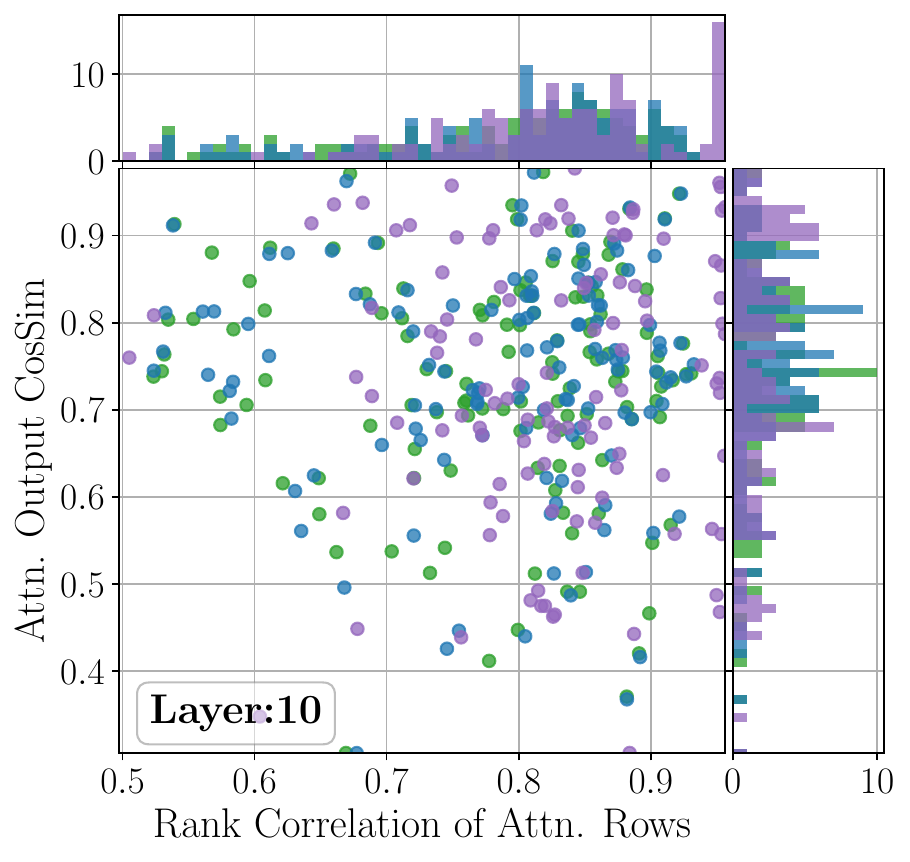}
    \includegraphics[width=0.32\linewidth]{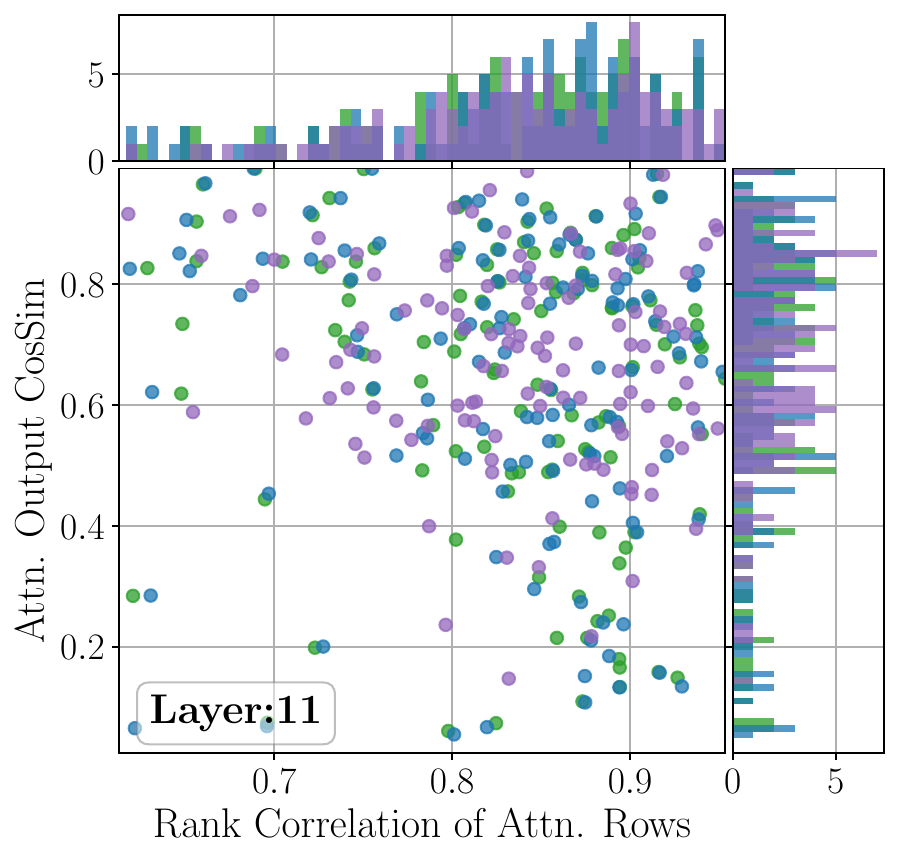}
    \includegraphics[width=0.32\linewidth]{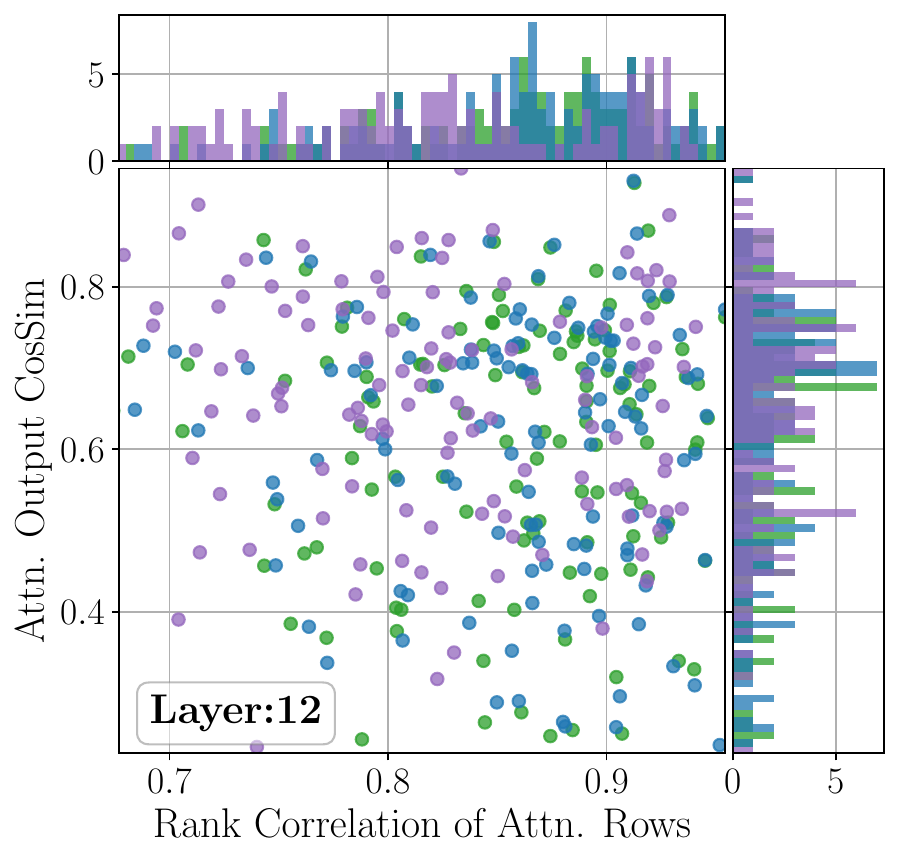}
    \caption{Attention output cosine similarity (compared to full attention) for Streaming LLM with our method. \cref{fig:spearman-cos-app,fig:spearman-cos-app-2,fig:spearman-cos-app-3} show the results from every layer, and are a counterpart to~\cref{fig:spearman-cos} in the main text. For the lower layers where induction heads are most prevalent, our method shows higher cosine similarity and attention row rank correlation as compared to quadratic attention.}
    \label{fig:spearman-cos-app}
\end{figure}

\begin{figure}
    \centering
    \includegraphics[width=0.5\linewidth]{figures/recompute-plots/spearman-cos-legend.pdf} \\
    \includegraphics[width=0.32\linewidth]{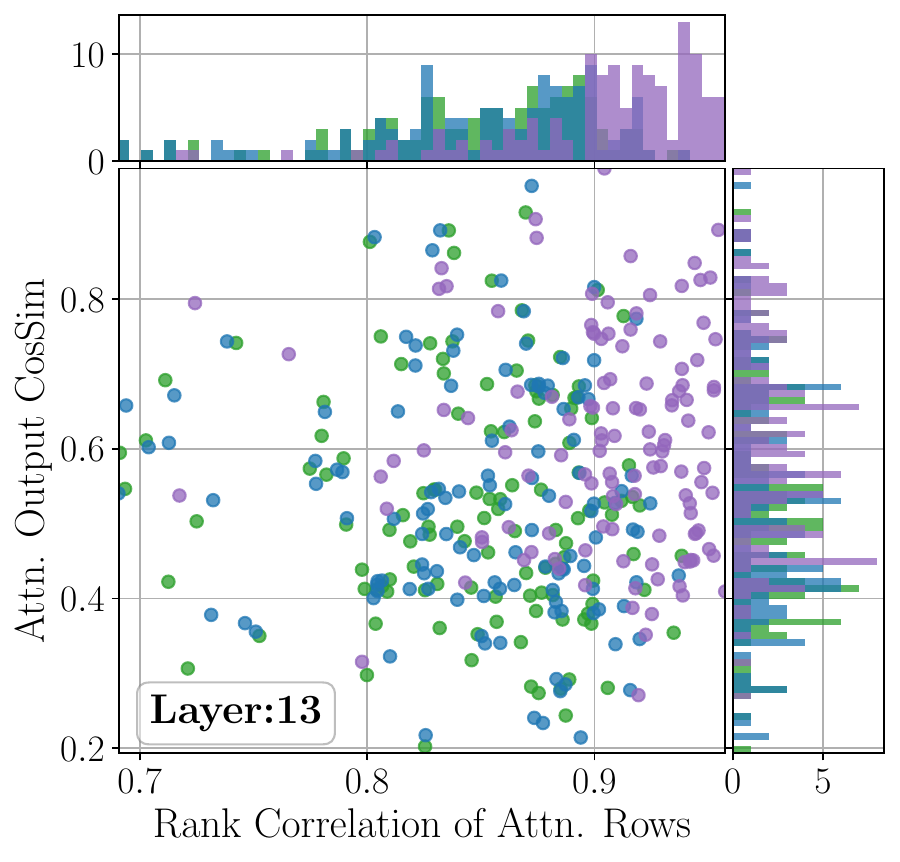}
    \includegraphics[width=0.32\linewidth]{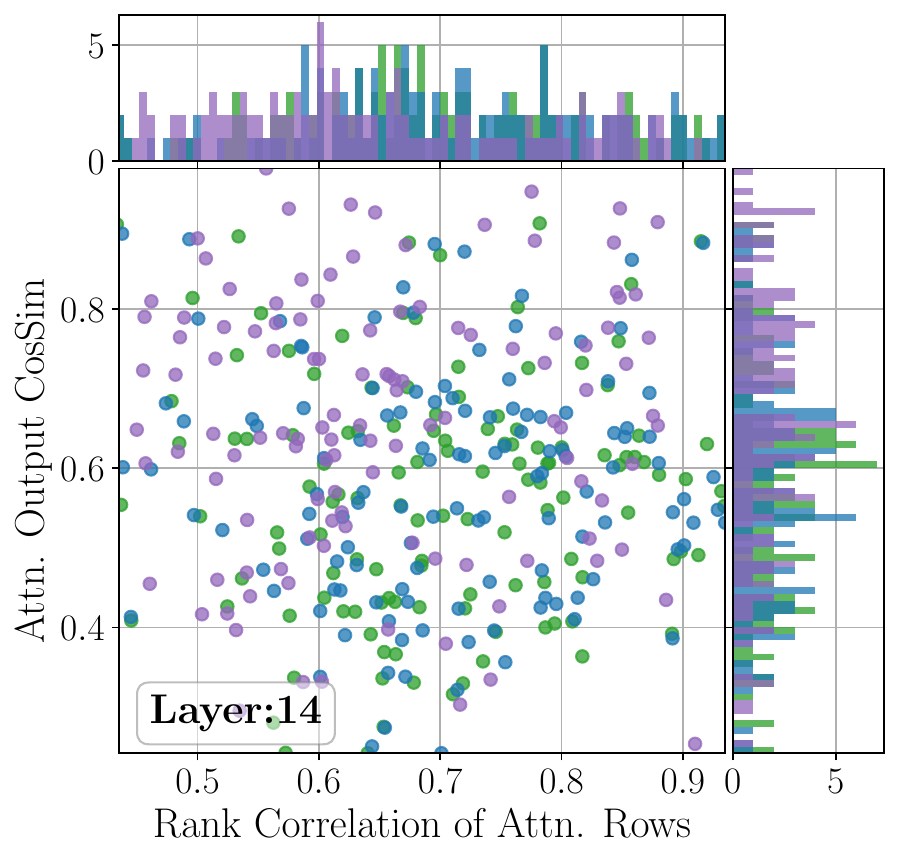}
    \includegraphics[width=0.32\linewidth]{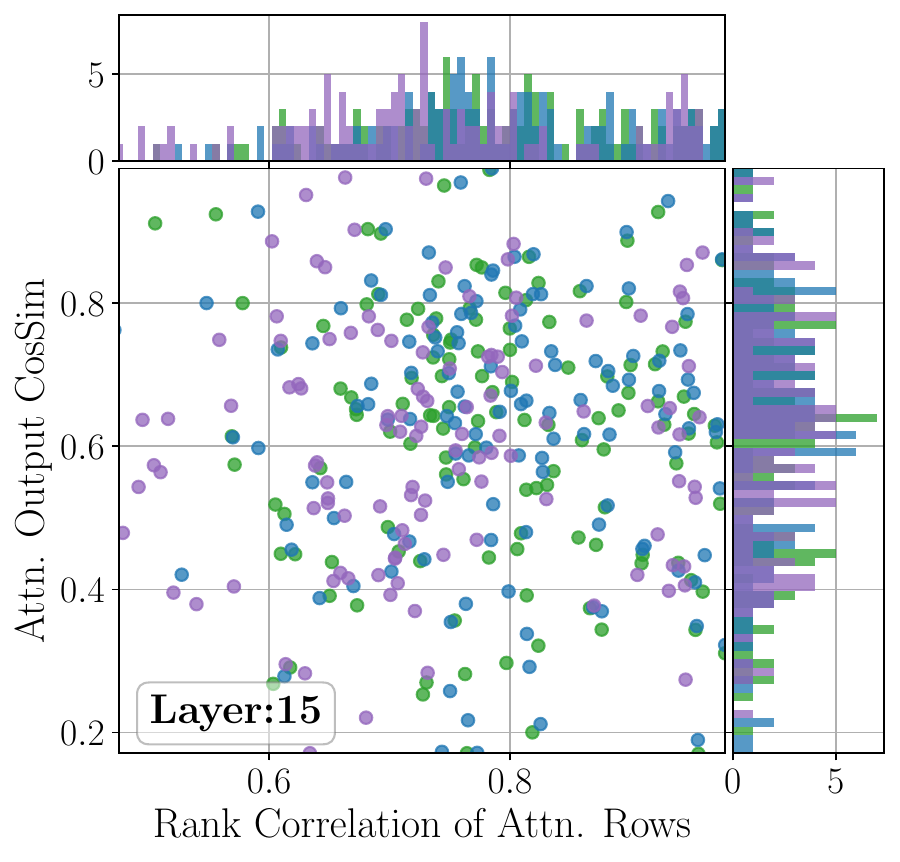}
    \includegraphics[width=0.5\linewidth]{figures/recompute-plots/spearman-cos-legend.pdf} \\
    \includegraphics[width=0.32\linewidth]{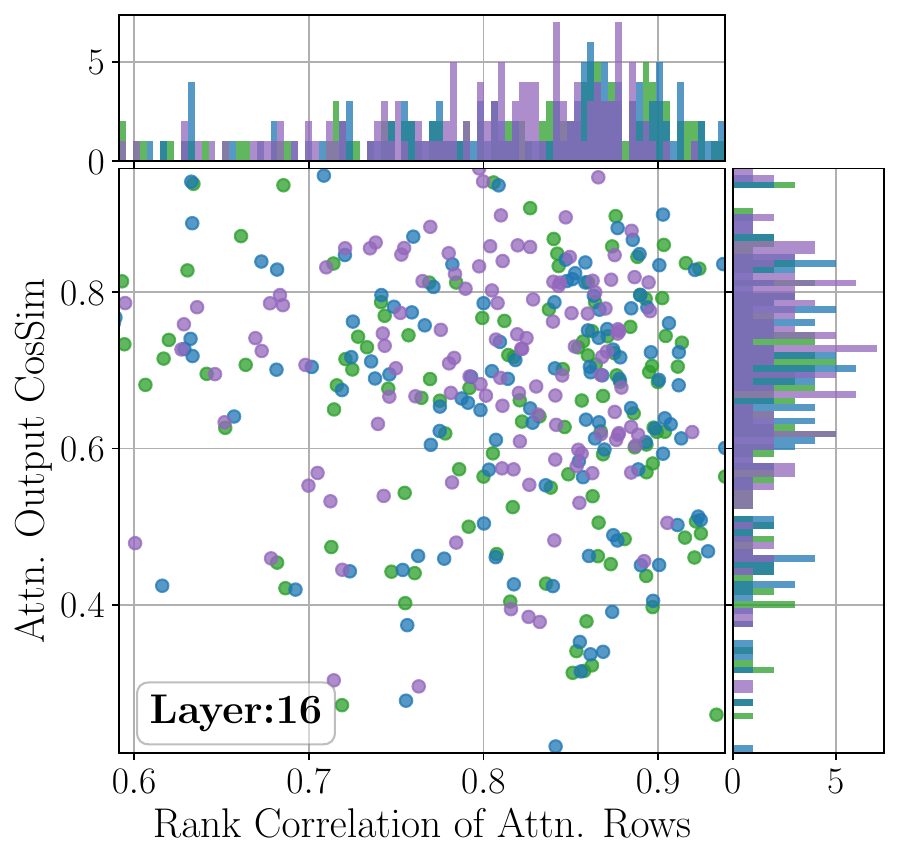}
    \includegraphics[width=0.32\linewidth]{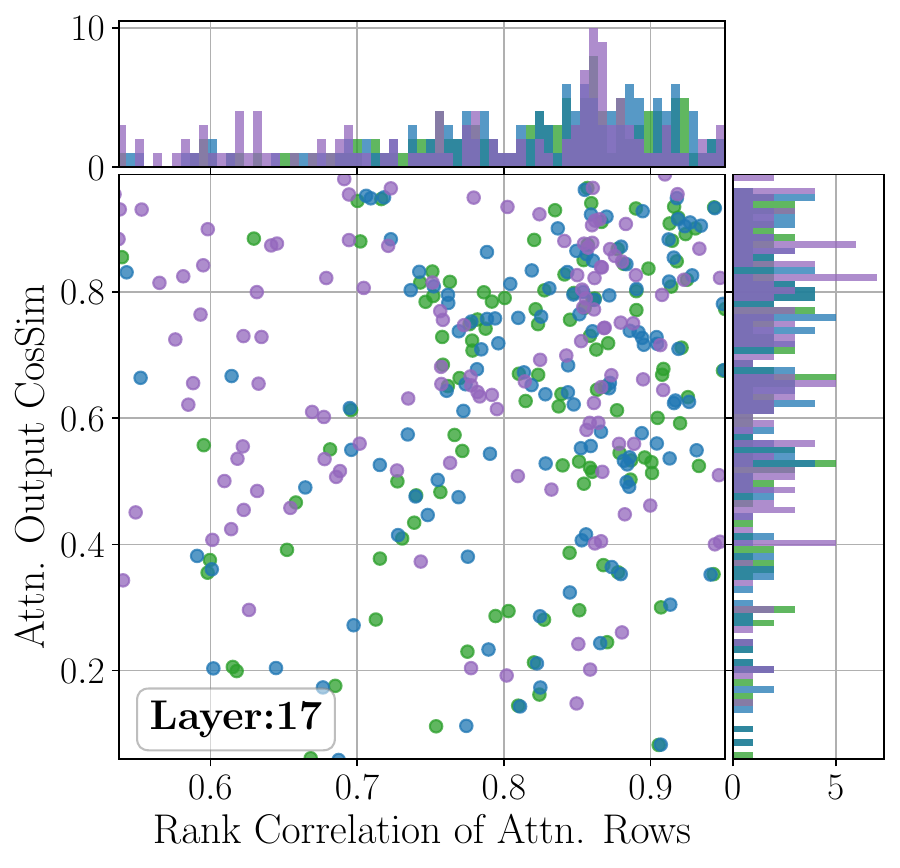}
    \includegraphics[width=0.32\linewidth]{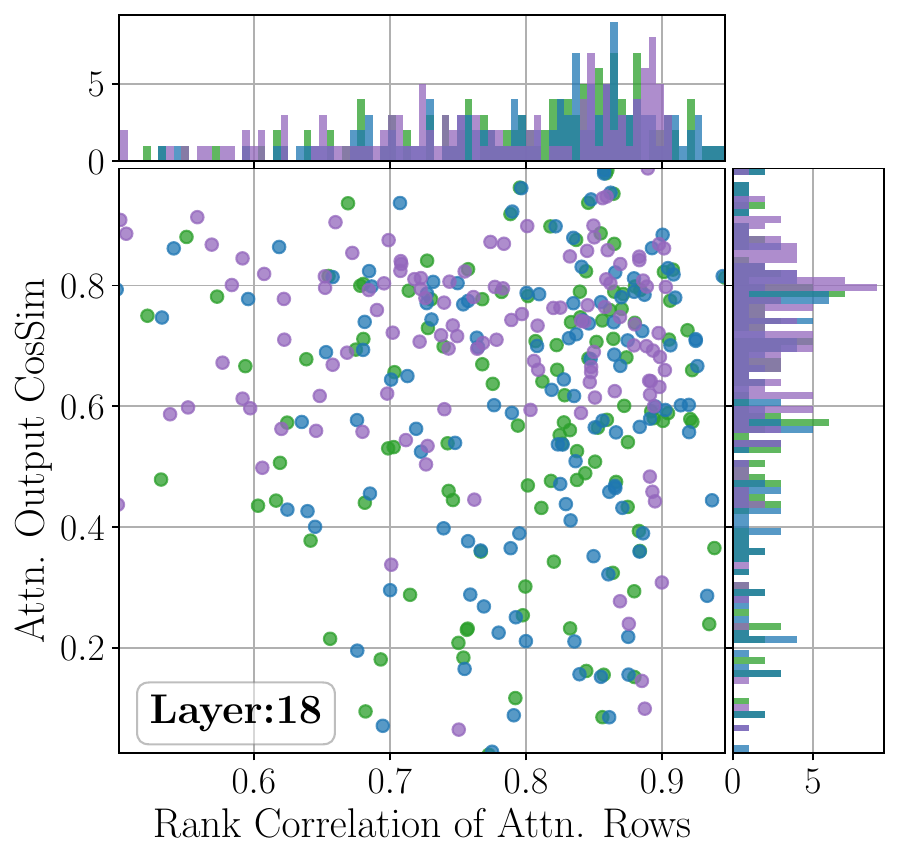}
    \includegraphics[width=0.5\linewidth]{figures/recompute-plots/spearman-cos-legend.pdf} \\
    \includegraphics[width=0.32\linewidth]{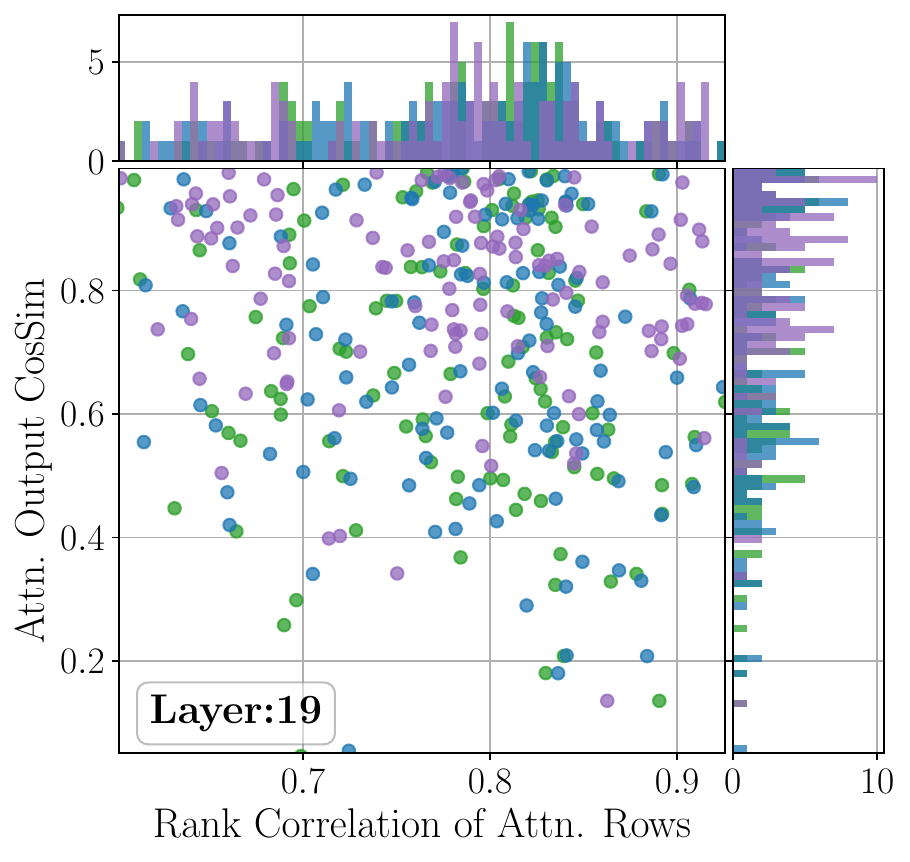}
    \includegraphics[width=0.32\linewidth]{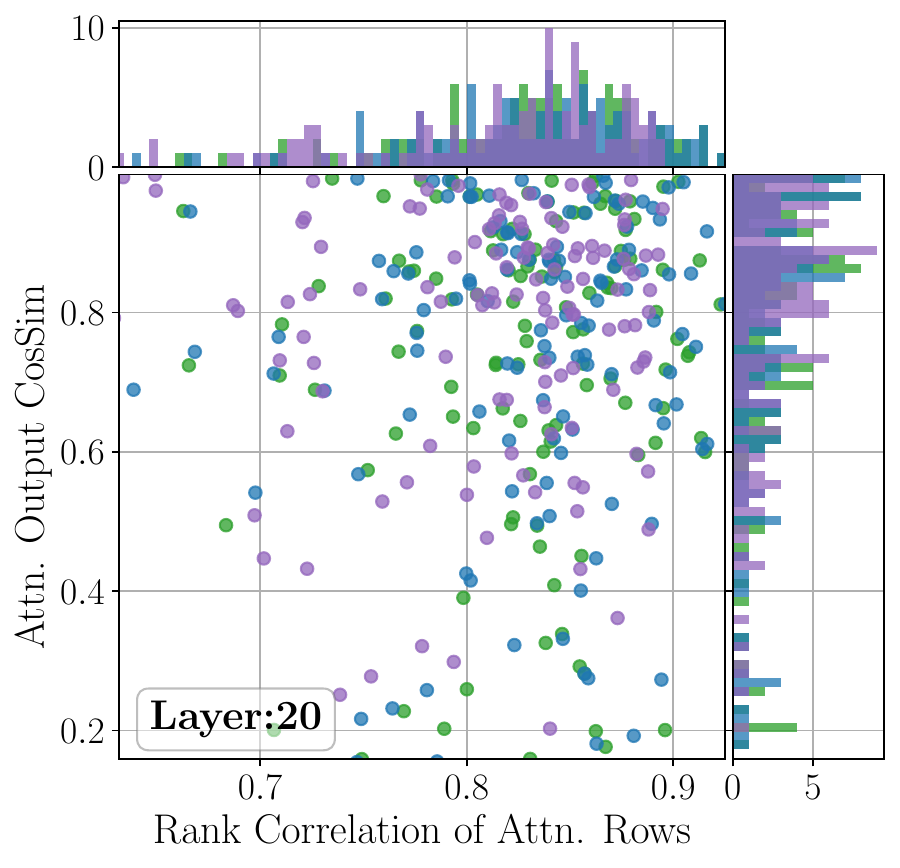}
    \includegraphics[width=0.32\linewidth]{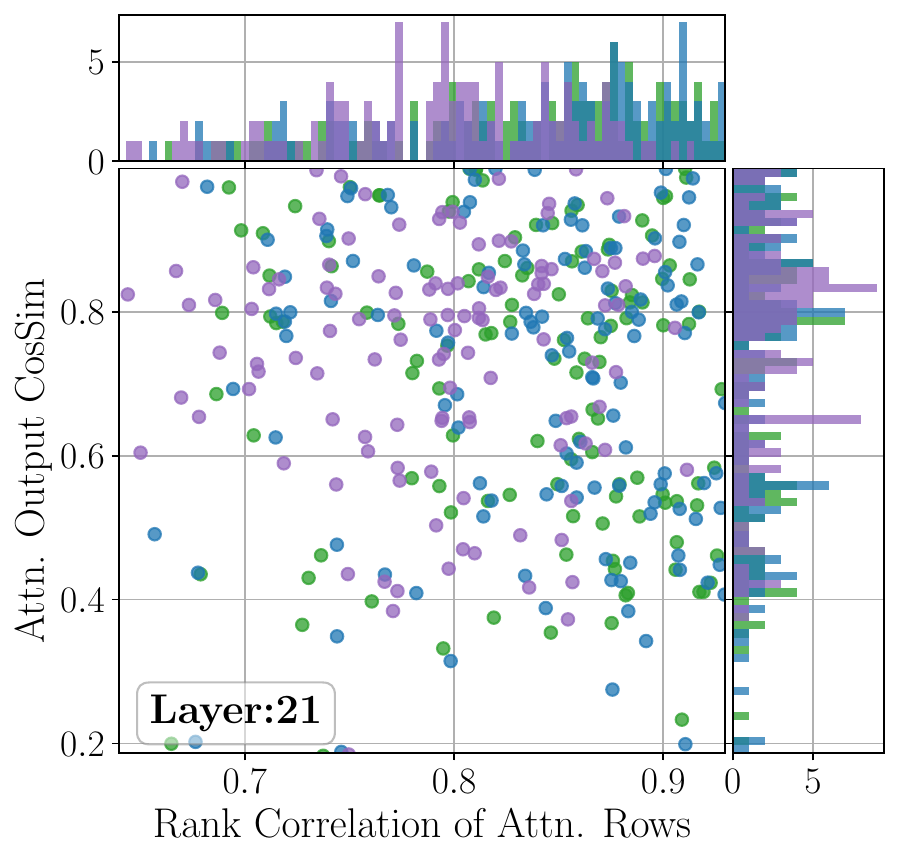}
    \includegraphics[width=0.5\linewidth]{figures/recompute-plots/spearman-cos-legend.pdf} \\
    \includegraphics[width=0.32\linewidth]{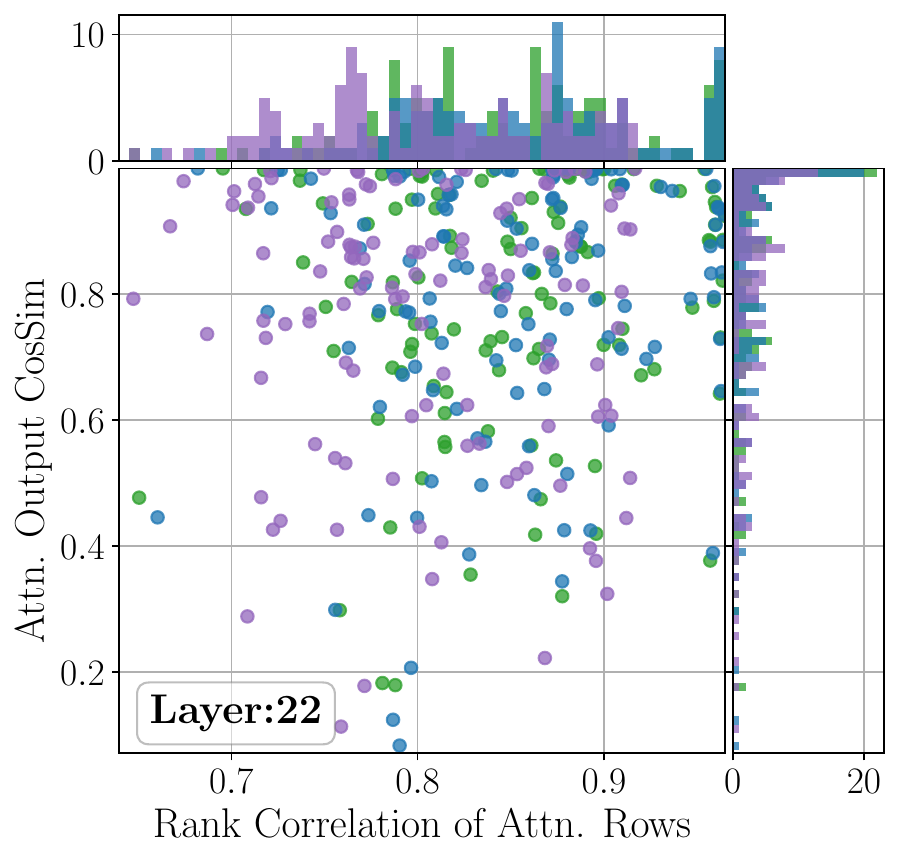}
    \includegraphics[width=0.32\linewidth]{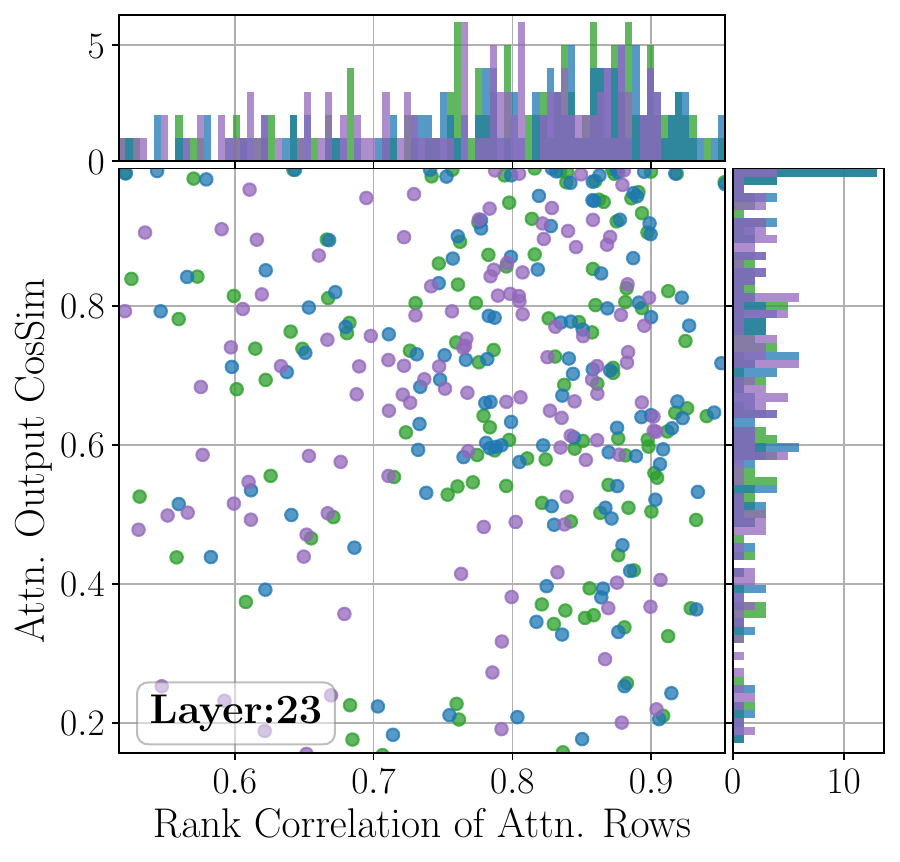}
    \includegraphics[width=0.32\linewidth]{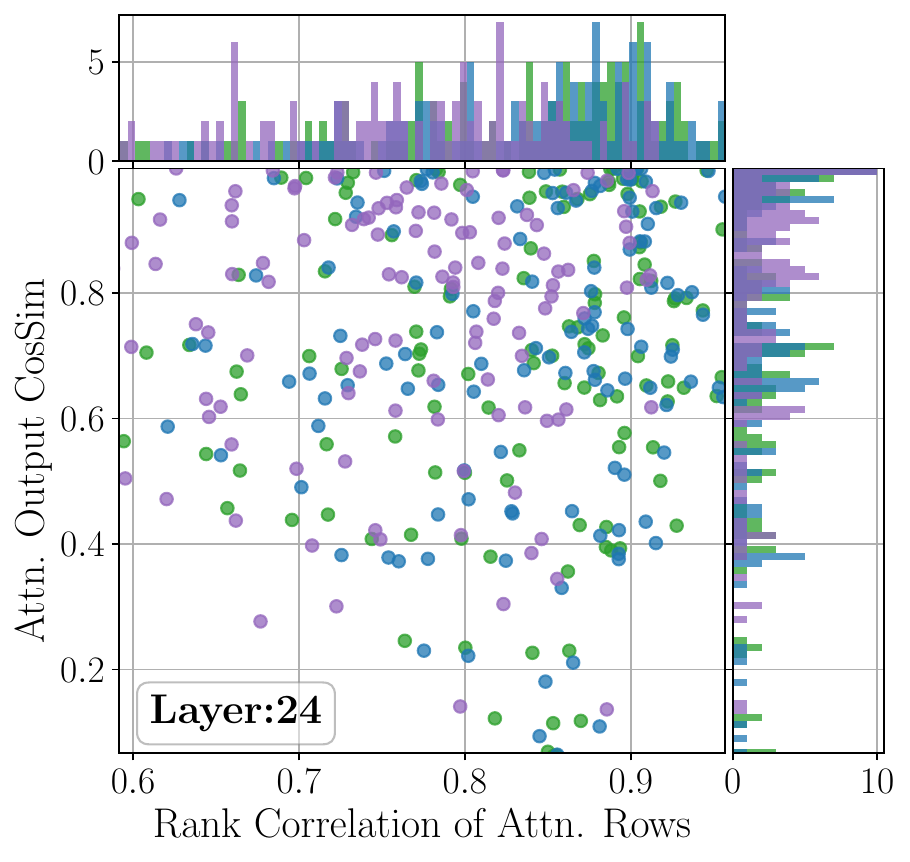}
    \caption{Attention output cosine similarity (compared to full attention) for Streaming LLM with our method. \cref{fig:spearman-cos-app,fig:spearman-cos-app-2,fig:spearman-cos-app-3} show the results from every layer, and are a counterpart to~\cref{fig:spearman-cos} in the main text. For the lower layers where induction heads are most prevalent, our method shows higher cosine similarity and attention row rank correlation as compared to quadratic attention.}
    \label{fig:spearman-cos-app-2}
\end{figure}

\begin{figure}
    \centering
    \includegraphics[width=0.5\linewidth]{figures/recompute-plots/spearman-cos-legend.pdf} \\
    \includegraphics[width=0.32\linewidth]{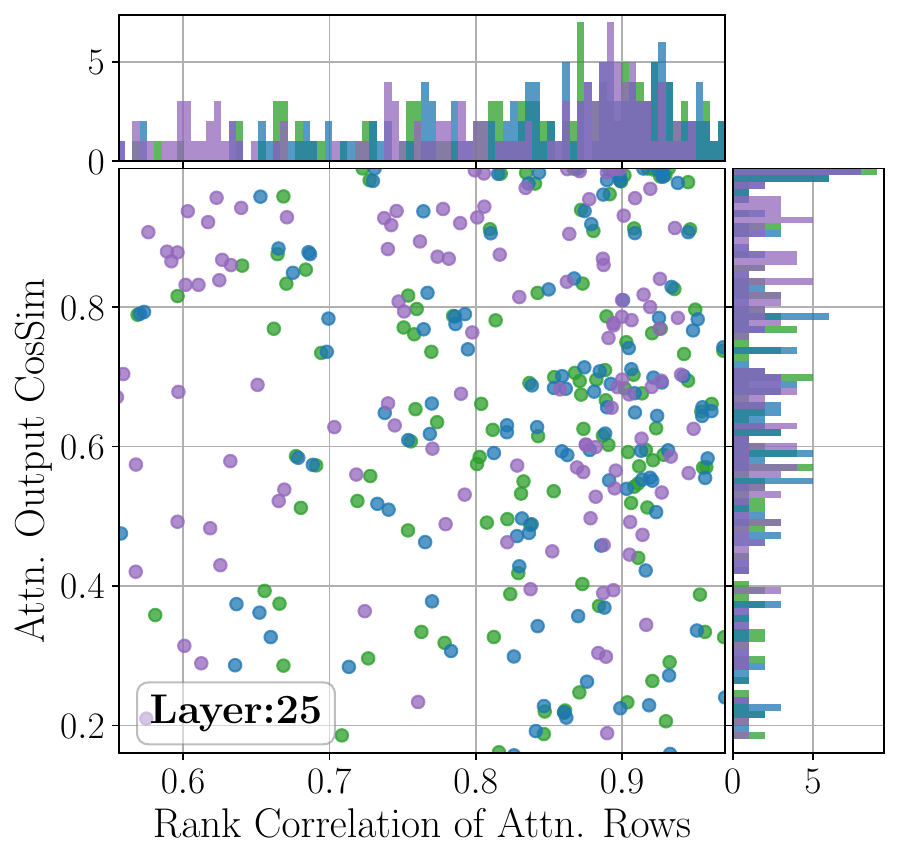}
    \includegraphics[width=0.32\linewidth]{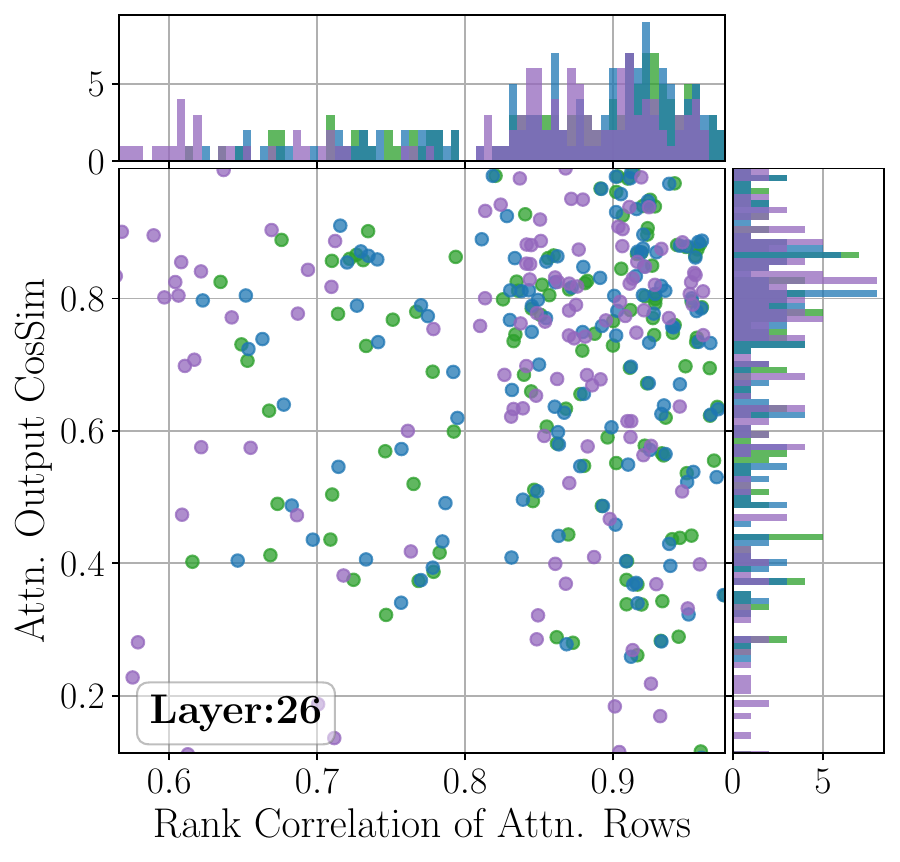}
    \includegraphics[width=0.32\linewidth]{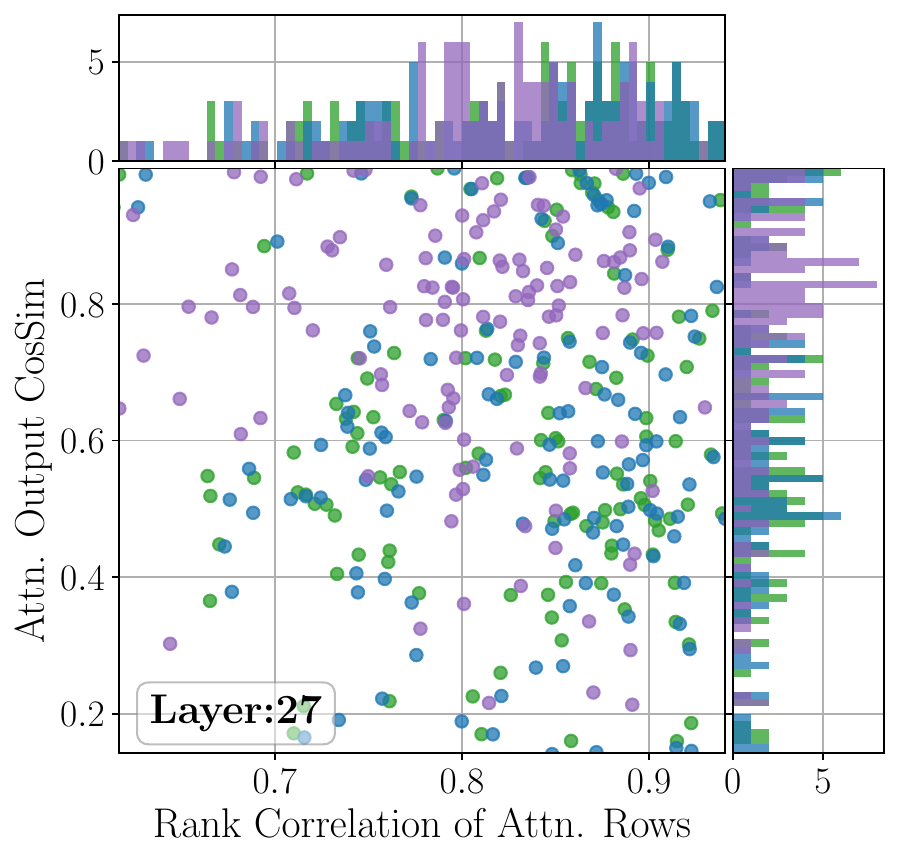}
    \includegraphics[width=0.5\linewidth]{figures/recompute-plots/spearman-cos-legend.pdf} \\
    \includegraphics[width=0.32\linewidth]{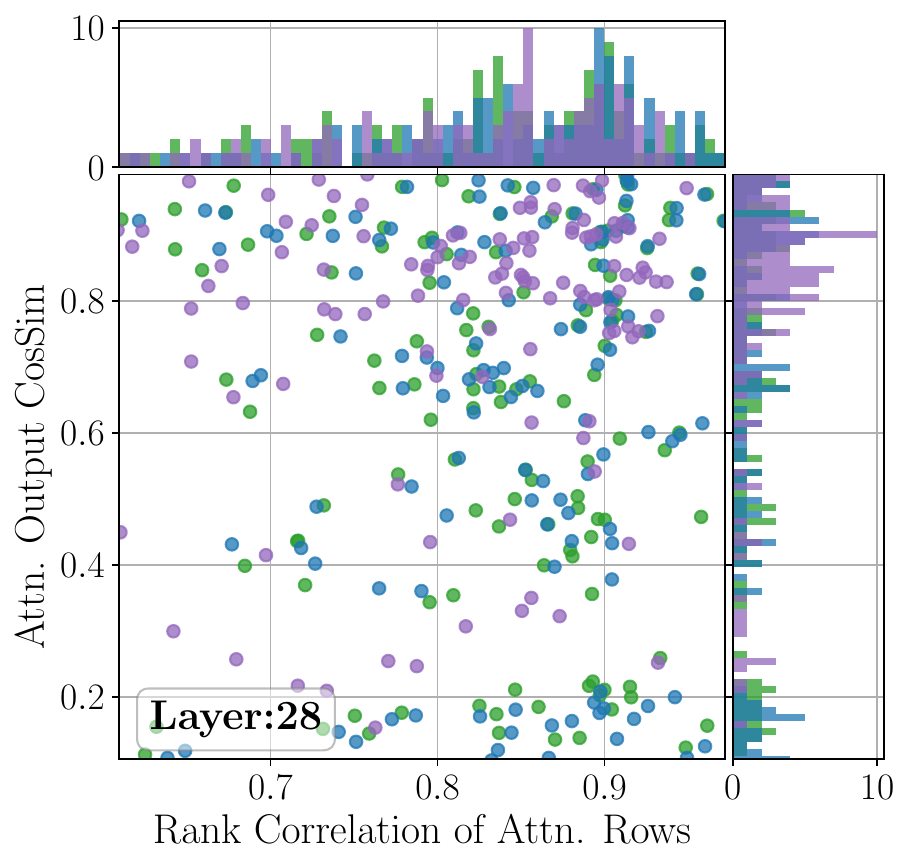}
    \includegraphics[width=0.32\linewidth]{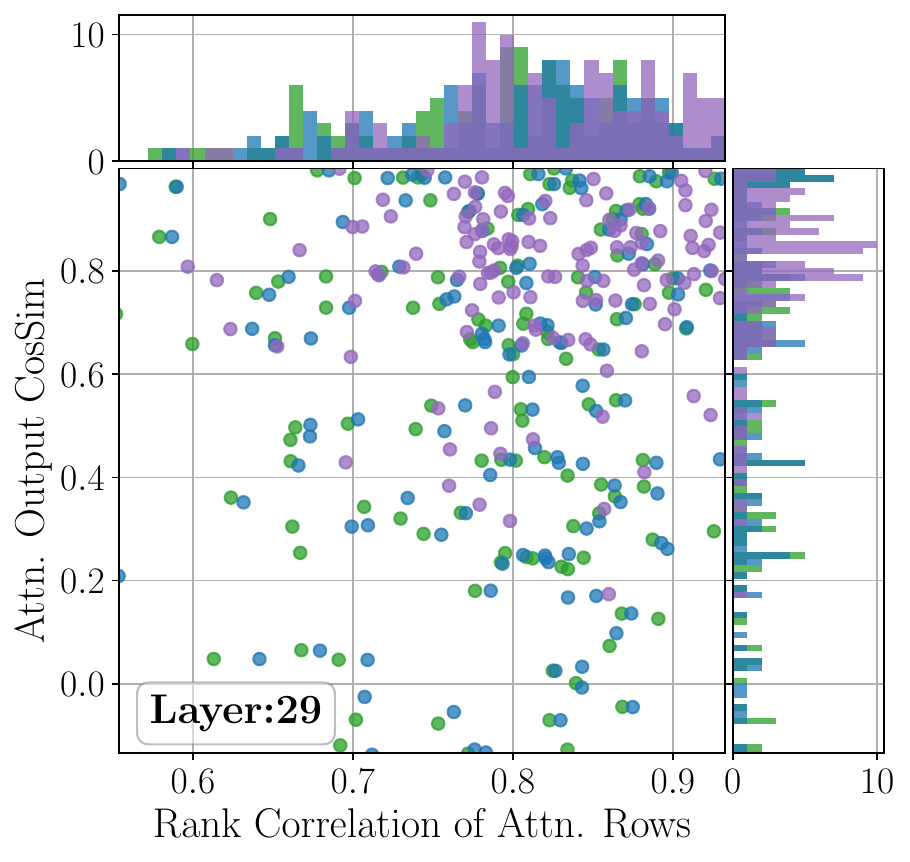}
    \includegraphics[width=0.32\linewidth]{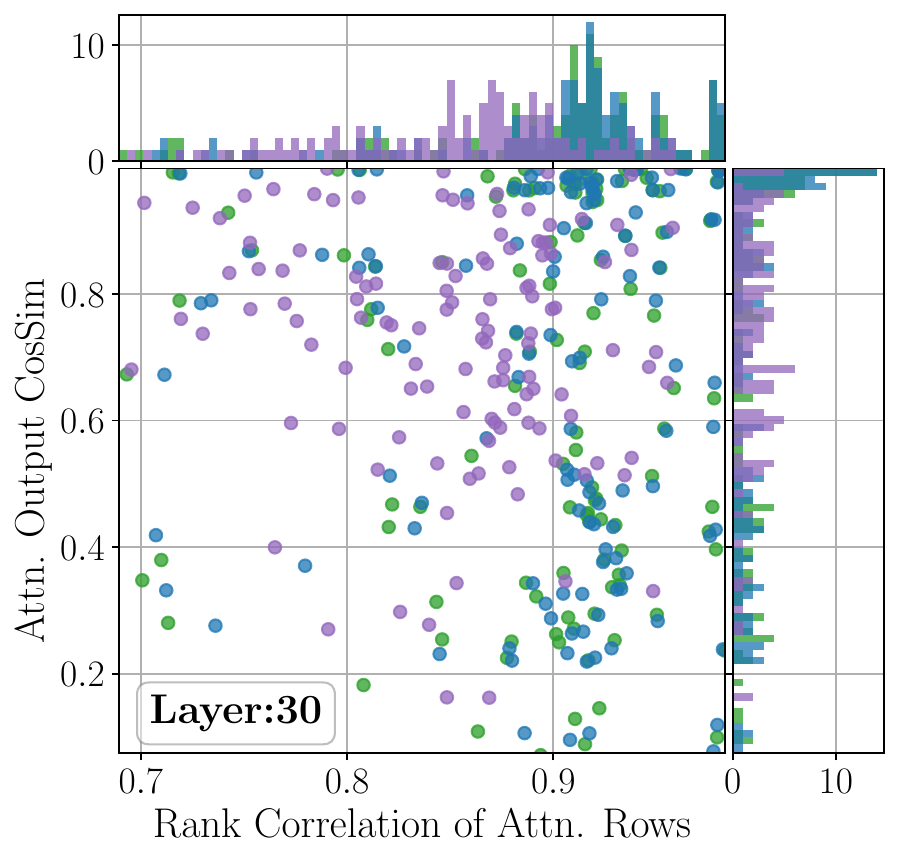}
    \includegraphics[width=0.5\linewidth]{figures/recompute-plots/spearman-cos-legend.pdf} \\
    \includegraphics[width=0.32\linewidth]{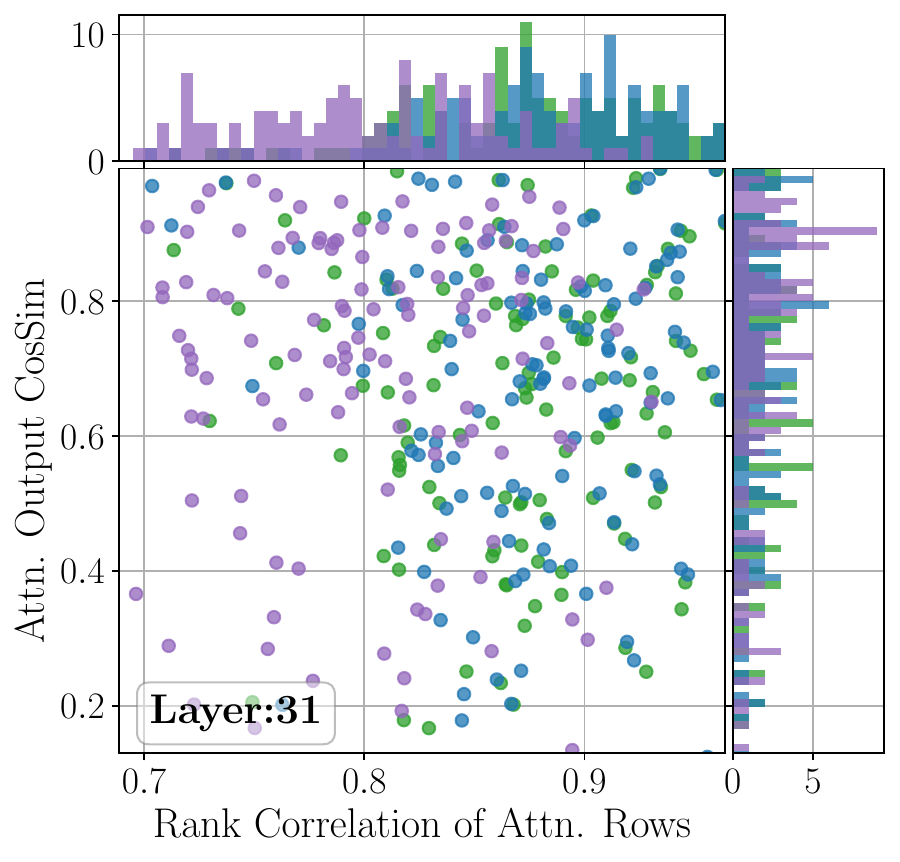}
    \includegraphics[width=0.32\linewidth]{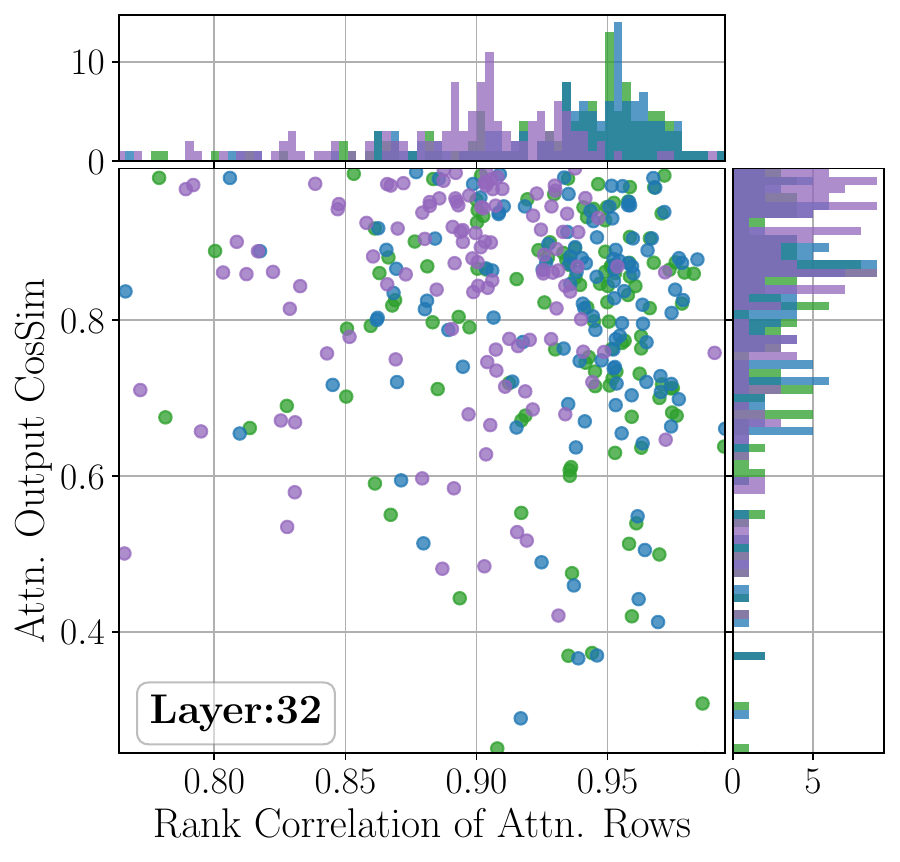}
    \includegraphics[width=0.32\linewidth]{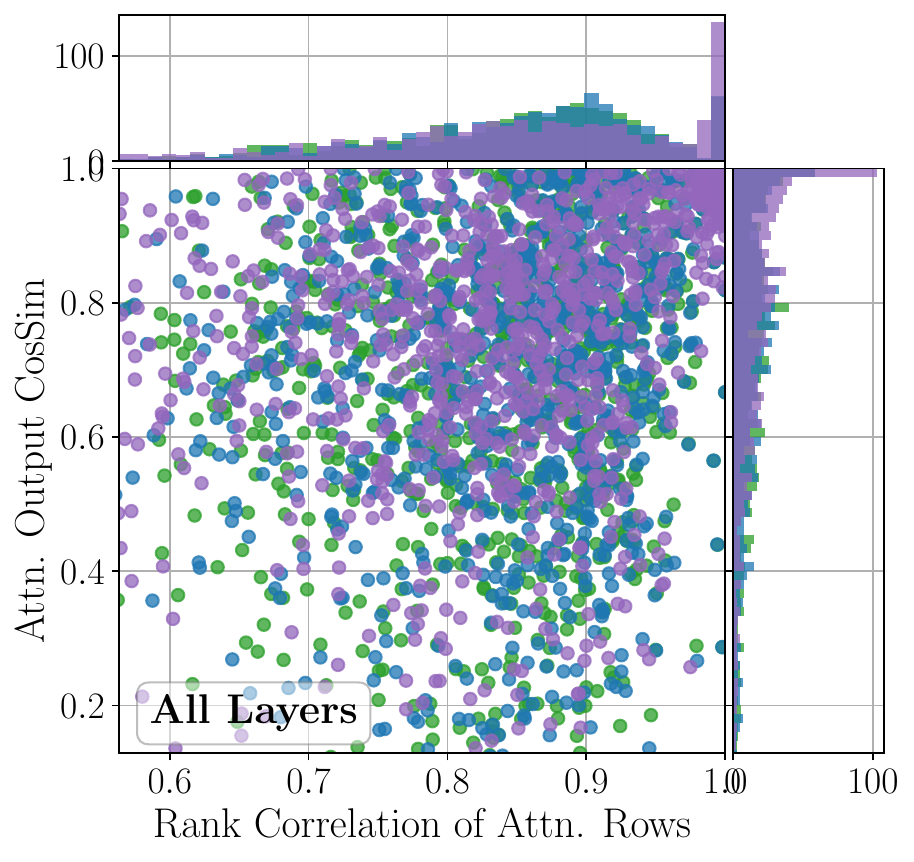}
    \caption{Attention output cosine similarity (compared to full attention) for Streaming LLM with our method. \cref{fig:spearman-cos-app,fig:spearman-cos-app-2,fig:spearman-cos-app-3} show the results from every layer, and are a counterpart to~\cref{fig:spearman-cos} in the main text. For the lower layers where induction heads are most prevalent, our method shows higher cosine similarity and attention row rank correlation as compared to quadratic attention.}
    \label{fig:spearman-cos-app-3}
\end{figure}

\clearpage
\section*{NeurIPS Paper Checklist}

\begin{enumerate}

\item {\bf Claims}
    \item[] Question: Do the main claims made in the abstract and introduction accurately reflect the paper's contributions and scope?
    \item[] Answer: \answerYes{} 
    \item[] Justification: The claims made in the abstract and introduction are verified in our experiments conducted in \cref{sec:experiments}.

\item {\bf Limitations}
    \item[] Question: Does the paper discuss the limitations of the work performed by the authors?
    \item[] Answer: \answerYes{} 
    \item[] Justification: We have discussed limitations of our method in \cref{sec:discussion}.

\item {\bf Theory assumptions and proofs}
    \item[] Question: For each theoretical result, does the paper provide the full set of assumptions and a complete (and correct) proof?
    \item[] Answer: \answerYes{} 
    \item[] Justification: We have one theoretical result in \cref{lem:delta}, which was stated briefly in the main text. We have included a more detailed derivation and statement in \cref{sec:delta-lemma-restate}. This section was also referenced under the lemma in the main text. 

    \item {\bf Experimental result reproducibility}
    \item[] Question: Does the paper fully disclose all the information needed to reproduce the main experimental results of the paper to the extent that it affects the main claims and/or conclusions of the paper (regardless of whether the code and data are provided or not)?
    \item[] Answer: \answerYes{} 
    \item[] Justification: We have provided all necessary information to reproduce our results. Our method only relies on publicly available pretrained models. We have included experimental code as well. 

\item {\bf Open access to data and code}
    \item[] Question: Does the paper provide open access to the data and code, with sufficient instructions to faithfully reproduce the main experimental results, as described in supplemental material?
    \item[] Answer: \answerYes{} 
    \item[] Justification: The datasets we use are publicly available and cited. We generated one dataset according to a previous paper (PG19 Long QA), which has been included in our supplementary materials. The code for our experiments is included in the supplementary material.

\item {\bf Experimental setting/details}
    \item[] Question: Does the paper specify all the training and test details (e.g., data splits, hyperparameters, how they were chosen, type of optimizer, etc.) necessary to understand the results?
    \item[] Answer: \answerYes{} 
    \item[] Justification: We have one hyperparameter which is specified in \cref{sec:experiments}. We have also provided \cref{alg:delta}.

\item {\bf Experiment statistical significance}
    \item[] Question: Does the paper report error bars suitably and correctly defined or other appropriate information about the statistical significance of the experiments?
    \item[] Answer: \answerNo{} 
    \item[] Justification: Our method is deterministic and works on pretrained models. Therefore, there is no stochasticity present in order to report error bars. Instead we conduct a range of experiments on different datasets in \cref{sec:experiments} in order to verify that the results do not randomly favor our method for a particular experiment. However, we do provide a paired permutation test for the RULER experiments in~\cref{sec:paired-permutation}.

\item {\bf Experiments compute resources}
    \item[] Question: For each experiment, does the paper provide sufficient information on the computer resources (type of compute workers, memory, time of execution) needed to reproduce the experiments?
    \item[] Answer: \answerYes{} 
    \item[] Justification: We have stated the full range of compute resources in~\cref{sec:compute-resources}.
    
\item {\bf Code of ethics}
    \item[] Question: Does the research conducted in the paper conform, in every respect, with the NeurIPS Code of Ethics \url{https://neurips.cc/public/EthicsGuidelines}?
    \item[] Answer: \answerYes{} 
    \item[] Justification: We have read the ethics guidelines, and we believe our paper conforms to them.

\item {\bf Broader impacts}
    \item[] Question: Does the paper discuss both potential positive societal impacts and negative societal impacts of the work performed?
    \item[] Answer: \answerYes{} 
    \item[] Justification: In \cref{sec:intro} we discuss the enormous costs and negative externalities caused by inference compute requirements. We also discuss the broader impacts in~\cref{sec:broader-impact}.
    
\item {\bf Safeguards}
    \item[] Question: Does the paper describe safeguards that have been put in place for responsible release of data or models that have a high risk for misuse (e.g., pretrained language models, image generators, or scraped datasets)?
    \item[] Answer: \answerNA{} 
    \item[] Justification: We create no new data or models to release, as our method proposes a modification to existing pretrained models for inference efficiency. 

\item {\bf Licenses for existing assets}
    \item[] Question: Are the creators or original owners of assets (e.g., code, data, models), used in the paper, properly credited and are the license and terms of use explicitly mentioned and properly respected?
    \item[] Answer: \answerYes{} 
    \item[] Justification: The datasets we use are publicly available and cited or included in the supplementary material. The dataset included in the supplementary material is a derivation of a publicly available dataset, and the method for constructing it has been cited in~\cref{sec:experiments}. 

\item {\bf New assets}
    \item[] Question: Are new assets introduced in the paper well documented and is the documentation provided alongside the assets?
    \item[] Answer: \answerYes{} 
    \item[] Justification: We are releasing a QA test set which was specified by a previous work, but not released by those authors directly. We have generated the dataset according to their code, and are releasing it with our supplementary materials.

\item {\bf Crowdsourcing and research with human subjects}
    \item[] Question: For crowdsourcing experiments and research with human subjects, does the paper include the full text of instructions given to participants and screenshots, if applicable, as well as details about compensation (if any)? 
    \item[] Answer: \answerNA{} 
    \item[] Justification: Not applicable

\item {\bf Institutional review board (IRB) approvals or equivalent for research with human subjects}
    \item[] Question: Does the paper describe potential risks incurred by study participants, whether such risks were disclosed to the subjects, and whether Institutional Review Board (IRB) approvals (or an equivalent approval/review based on the requirements of your country or institution) were obtained?
    \item[] Answer: \answerNA{} 
    \item[] Justification: Not applicable.

\item {\bf Declaration of LLM usage}
    \item[] Question: Does the paper describe the usage of LLMs if it is an important, original, or non-standard component of the core methods in this research? Note that if the LLM is used only for writing, editing, or formatting purposes and does not impact the core methodology, scientific rigorousness, or originality of the research, declaration is not required.
    \item[] Answer: \answerNA{} 
    \item[] Justification: Not applicable

\end{enumerate}

\end{document}